%% file: weighted_neurips.tex
\documentclass[11pt]{article}

\usepackage{amsfonts,amsmath,amsthm,times,fullpage,amssymb,epsfig,graphicx,enumerate,bm}
\usepackage{hyperref}
\usepackage{enumitem}
\hypersetup{colorlinks=true}
\usepackage[dvipsnames]{xcolor} 
\usepackage{color}
\usepackage{dsfont}
\usepackage{wrapfig}
\usepackage{relsize}
\usepackage{url} 
\usepackage[noend]{algpseudocode} 
\usepackage[nothing]{algorithm}  
\usepackage{caption}
\newtheorem{theorem}{Theorem}[section]
\newtheorem{lemma}[theorem]{Lemma}
\newtheorem{corollary}[theorem]{Corollary}

\newtheorem{definition}[theorem]{Definition} 
\newtheorem{remark}{Remark}[section]
\newtheorem{proposition}[theorem]{Proposition}

\newcommand\ex{{\mathbb{E}}}

\DeclareMathOperator*{\argmin}{arg\,min}
\DeclareMathOperator*{\argmax}{arg\,max}

\usepackage{subcaption}
\usepackage{babel,blindtext}

\usepackage[utf8]{inputenc} 
\usepackage[T1]{fontenc}    
\usepackage{hyperref}       
\usepackage{url}            
\usepackage{booktabs}       
\usepackage{amsfonts}       
\usepackage{nicefrac}       
\usepackage{microtype}      
\usepackage{xcolor}         

\newcommand{\blue}[1]{{\color{black}{#1}}}

\newcommand{\ignore}[1]{}

\begin{document}

\title{Weighted Distillation with Unlabeled Examples}

\author{
  Fotis Iliopoulos \\
  Google Research\\
\footnotesize{  \texttt{fotisi@google.com} }\\
 \and
  Vasilis Kontonis \\
  Google Research\\
 \footnotesize{ \texttt{kontonis@google.com} } \\
  \and
  Cenk Baykal \\
  Google Research\\
\footnotesize{  \texttt{baykalc@google.com} } \\
      \and
  Gaurav Menghani \\
  Google Research
  \\
 \footnotesize{ \texttt{gmenghani@google.com}} \\
  \and
  Khoa Trinh \\
  Google Research
  \\
\footnotesize{  \texttt{khoatrinh@google.com} }\\
    \and
  Erik Vee \\
  Google Research
  \\
\footnotesize{  \texttt{erikvee@google.com}} \\
}

\maketitle

\begin{abstract}
Distillation with unlabeled examples is a popular and powerful method for training deep neural networks in settings where the amount of labeled data  is limited: A large ``teacher'' neural network is trained on the labeled data available, and then it is used to generate labels on an unlabeled dataset (typically much larger in size). These labels are then utilized to train the smaller ``student'' model which will
actually be deployed.  Naturally, the success of the approach depends on the quality of the teacher's labels, since the student could be confused if trained on inaccurate data.
This paper proposes a principled approach for addressing this issue based on  \blue{a ``debiasing" reweighting of the student's loss function} tailored to the distillation training paradigm. Our method is hyper-parameter free,  data-agnostic, and simple to implement. We demonstrate significant improvements on popular academic datasets and we accompany our results with a theoretical analysis which rigorously justifies the performance of our method in certain settings.

\end{abstract}

\section{Introduction}
\label{introduction}

In many modern applications of deep neural networks, where the amount of labeled examples for training is limited, \emph{distillation with unlabeled examples}~\cite{chen2020big, distillation} has been enormously successful.
In this two-stage training paradigm 
a larger and more sophisticated  ``teacher model'' — typically non-deployable for the purposes of the application — is trained to learn from the limited amount of available training data in the first stage.
In the second stage, the teacher model is used to generate labels on an unlabeled dataset, which is usually much larger in size than the original dataset the teacher itself was trained on. These labels are then utilized to train the ``student model'', namely the model which will  actually be deployed. 
\blue{Notably, distillation with unlabeled examples is the most commonly used training-paradigm in applications where one finetunes and distills from very large-scale foundational models such as BERT~\cite{devlin2018bert} and  GPT-3~\cite{brown2020language} and, additionally,  it can be used to significantly improve distillation on supervised examples only (see e.g.~\cite{xie2020self}).}


While this
has proven to be a very powerful approach in practice, its success depends on the quality of labels provided by the teacher model. Indeed,  often times the teacher model generates inaccurate labels on a non-negligible portion of the unlabeled dataset, confusing the student. As  deep neural networks are susceptible to overfitting to corrupted labels~\cite{zhang2021understanding}, training the student on the teacher's noisy labels can lead to degradation in its generalization performance. \blue{As an example, Figure~\ref{oracle_and_noise} depicts an instance based on CIFAR-10 where filtering out the teacher's \emph{noisy} labels is quite beneficial for the student's performance.}



\blue{
We address this shortcoming by introducing a natural ``noise model'' for the teacher which allows us to modify the student's loss function in a principled way in order to produce an unbiased estimate of the student's clean objective. From a practical standpoint, this produces a fully ``plug-and-play'' 
method which adds minimal implementation overhead, and which is composable with essentially every other known distillation technique.}

The idea of compressing a teacher model into a smaller student model by matching the predictions of the teacher was initially introduced by Bucilu\v{a}, Caruana and Niculescu-Mizil \cite{bucilua} and, since then,  variations of this method~\cite{distillation,   lee2013pseudo, muller2020subclass,  pham2021meta,  riloff1996automatically, scudder1965probability, yarowsky1995unsupervised}  have been applied in a wide variety of contexts~\cite{radosavovic2018data, yalniz2019billion, zoph2020rethinking}. (Notably, some of these applications go beyond compression — see e.g.~\cite{chen2020big,  xie2020self,yalniz2019billion, zou2019confidence} for reference.) In the simplest form of the method~\cite{lee2013pseudo},  and using classification as a canonical example, the labels produced by the teacher are one-hot vectors that represent the class which has the maximum predicted probability — this method is often referred to as ``hard''-distillation. More generally, Hinton et. al.~\cite{chen2020big, distillation} have shown that it is often beneficial to train the student  so that it minimizes the cross-entropy (or KL-divergence)  with the probability distribution produced by the teacher while also potentially using a temperature higher than $1$ in the softmax of both models (``soft''-distillation). (Temperatures higher than $1$ are used in order to emphasize the difference between the probabilities of the classes with lower likelihood of corresponding to the correct label according to the teacher model.)


The main contribution of this work is a principled method for improving distillation with unlabeled examples by reweighting the loss function of the student. That is, we  assign importance weights to the examples labeled by the teacher so that each weight reflects (i) how likely it is that the teacher has made an inaccurate prediction regarding the label of the example and (ii) how ``distorted''  the unweighted loss function we use to train the student  is (measured with respect to using  the ground-truth label for the example instead of the teacher's label). \blue{More concretely, our  reweighting strategy is based on introducing and analyzing a certain  noise model designed to capture the behavior of the teacher  in distillation.} 
In this setting, we are able to come up with a closed-form solution for weights that ``de-noise'' the objective in the sense that (on expectation) 
they simulate having access to clean labels. 
Crucially, we empirically observe that the key characteristics of \blue{our  noise model for the teacher}   can be effectively estimated through a small validation dataset, \blue{since in practice the teacher's noise is neither random nor adversarial, and it typically correlates well with its ``confidence'' — see e.g. Figure~\ref{oracle_and_noise}}. In particular, we use the validation dataset to learn a map that takes as input the teacher's and student's  confidence for the label of a certain example and outputs estimates for the quantities mentioned in items (i) and (ii) above. Finally, we plug in these estimates to our closed-form solution for the weights  so that, overall, we obtain an automated way of computing the \blue{student's reweighted objective} in practice. A detailed description of our method can be found in Section~\ref{our_method}.

Our main findings and contributions can be summarized as follows:

\begin{itemize}
    \item We propose a \emph{principled and hyperparameter-free}  reweighting method for knowledge distillation with unlabeled examples.
    The method is efficient, data-agnostic and simple to implement.  
    

    \item We present extensive experimental results which show that our method provides significant improvements when evaluated in standard benchmarks.
    
    
    \item 

    Our  reweighting technique comes with provable guarantees:
    (i) it is information-theoretically optimal; and (ii) under standard assumptions SGD optimization of the reweighted objective
     learns a solution with nearly optimal generalization.
\end{itemize}

\begin{figure*}[!ht]

  \begin{minipage}[t]{0.5\textwidth}
  \centering
 \includegraphics[width=0.8\textwidth]{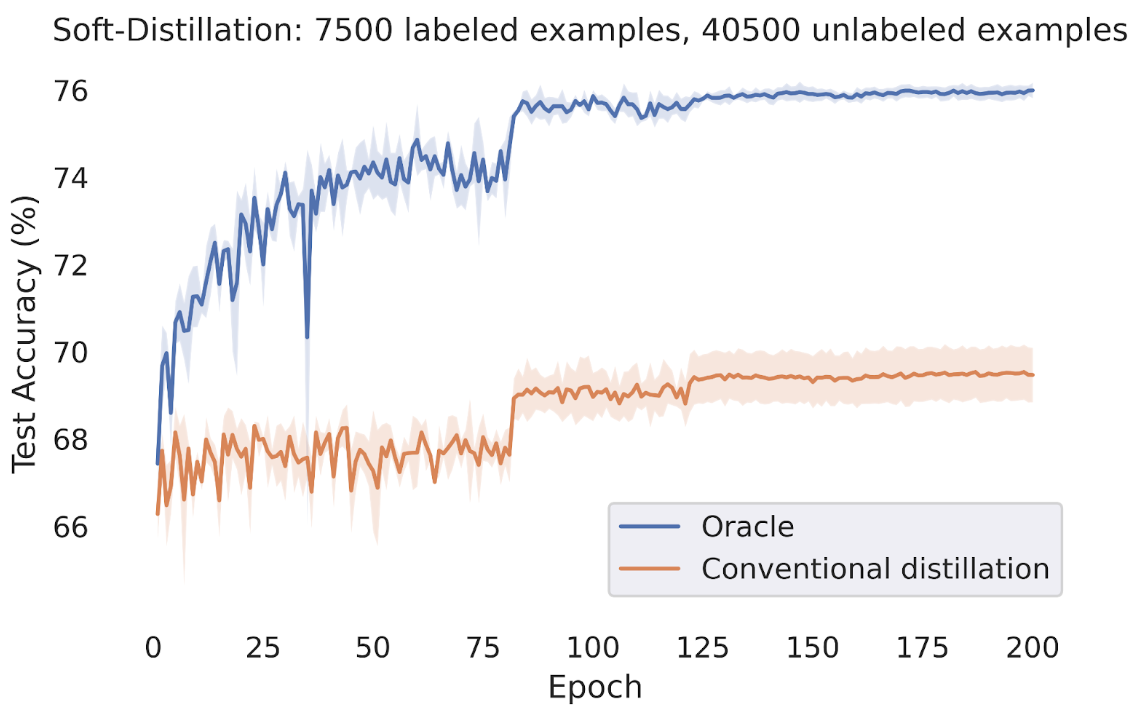} \\
  \end{minipage}%
  \hfill
    \begin{minipage}[t]{0.5\textwidth}
  \centering
 \includegraphics[width=0.8\textwidth]{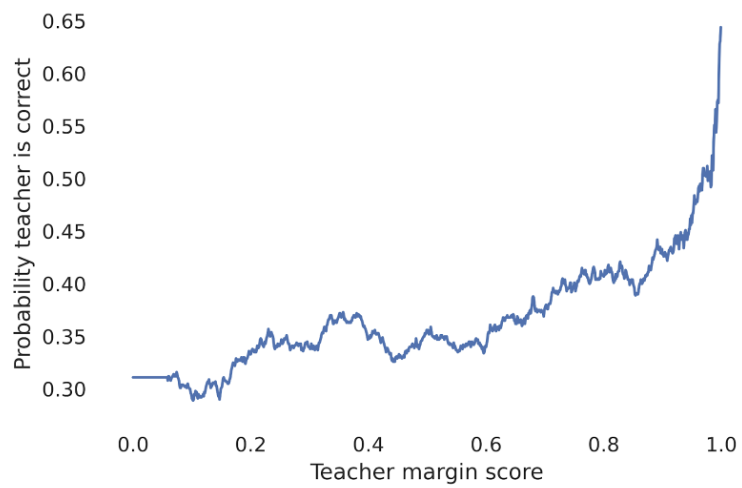}
  \end{minipage}%
  \hfill
    
\caption{\blue{\textbf{Left:} Performance comparison between a student trained using conventional distillation with unlabeled examples (orange) and a student trained only on the examples which are labeled \emph{correctly} by the teacher (blue). Here we assume access to $7500$ labeled examples and $40500$ unlabeled examples of CIFAR-10. The teacher is a MobileNet of depth multiplier $2$, while the student is a MobileNet of depth multiplier $1$.
\textbf{Right:} Plot of teacher's accuracy as a function of the margin score.}
 }
	\label{oracle_and_noise}
\end{figure*}

\subsection{Related work}\label{Related_work}\blue{
\textbf{Fidelity vs Generalization in knowledge distillation.} Conceptually, our work is related to the paper of Stanton et al.~\cite{stanton2021does} where the main message is that ``good student accuracy does not imply good distillation fidelity'', i.e., that more closely matching the teacher does not necessarily lead to better student generalization. In particular,~\cite{stanton2021does} demonstrates that when it comes to enlarging the distillation dataset beyond the teacher training data, there is a trade-off between optimization complexity and distillation data quality. Our work can be seen as a principled way of improving this trade-off.

\noindent
\textbf{Advanced distillation techniques.} Since the original paper of Hinton et. al.~\cite{distillation}, there have been several follow-up works~\cite{ahn2019variational,chen2021wasserstein, muller2020subclass, tian2019contrastive} which develop advanced distillation techniques which aim to  enforce greater consistency between the teacher and the student (typically in the context of distillation on labeled examples). These methods enforce consistency not only between the teacher's predictions and student's predictions, but also between the representations learned by the two models. However, in the context of distillation with unlabeled examples,  forcing the student to match the teacher's inaccurate predictions is still harmful, and therefore weighting the corresponding loss functions via our method is still applicable and beneficial. We demonstrate this fact by considering instances of the Variational Information Distillation for Knowledge Transfer (VID)~\cite{ahn2019variational} framework, and showing how it is indeed beneficial to combine it with our method in Section~\ref{sec:celeba}.
We also show that our approach provides benefits on top of any improvement obtained via temperature-scaling in Appendix~\ref{temperature_effect}.}

\noindent
\textbf{Learning with noisy data techniques.} 
As we have already discussed, the main conceptual contribution of our work is viewing the teacher model in distillation with unlabeled examples as the source of \blue{stochastic noise with certain characteristics}. Naturally, one may wonder what is the relationship between our method and works from  the vast literature of learning with noisy data  (e.g.~\cite{bar2021multiplicative,  frenay2013classification, gamberger1999experiments, jiang2018mentornet, kumar2021constrained, liu2015classification, majidi2021exponentiated,  natarajan2013learning, pleiss2020identifying, ren2018learning}). 
The answer is that our method \emph{does not attempt to solve the generic ``learning with noisy data''} problem, which is a highly non-trivial task in the absence of assumptions (both in theory and in practice).  \blue{Instead, our contribution is to observe and exploit the fact that the noise introduced by the teacher in distillation has  structure, as it correlates with several metrics of confidence such as the margin-score, the entropy of the predictions etc. We use and quantify this empirical observation to our advantage in order to formulate a principled method for developing a debiasing reweighting scheme (an approach inspired by the ``learning with noisy data"-literature) which comes with theoretical guarantees. An additional difference is that }works from the learning with noisy data literature typically assume that the training dataset consists of corrupted \emph{hard} labels, and often times their proposed method is not compatible with \emph{soft}-distillation (in the sense that they degrade its performance, making it much less effective) — see e.g.~\cite{muller2019does} for a study of this phenomenon in the case of label smoothing~\cite{lukasik2020does, szegedy2016rethinking}.

\noindent
\textbf{Uncertainty-based weighting schemes.} \blue{Related  to our method are approaches for semi-supervised learning
where examples are downweighted or filtered out when the teacher is ``uncertain.''
(A plethora of ways for defining and measuring ``uncertainty'' have been proposed in the literature, e.g. entropy, margin-score, dropout variance etc.) These methods are independent of the student model (they only depend on the teacher model), and so they cannot hope to remove the bias from the student's loss function. In fact, these methods can be viewed as preprocessing steps that can be combined with the student-training process we propose. To demonstrate this we both combine  and compare our method to the fidelity-based weighting scheme of~\cite{dehghani2017fidelity} in Section~\ref{fidelity_weights_comparison}}.


\subsection{Organization of the paper.} 
In Section~\ref{weighted_distillation}, we present our method in detail, while in Section~\ref{experimental_results} we present our key experimental results. In Section~\ref{theoretical_aspects}, we discuss the theoretical aspects of our work. In Section~\ref{conclusion}, we summarize our results, we discuss the benefits and limitations of our method, and  future work. Finally, we present extended experimental and theoretical results in the Appendix.


\section{Weighted distillation with unlabeled examples}
\label{weighted_distillation}

\blue{In this section we present our method. In Section~\ref{background}, we review multiclass classification. In Section~\ref{adversarial_setting}, we introduce the noise model for the teacher in distillation, which motivates our approach. And finally, in Section~\ref{our_method}, we describe our method in detail.}

\subsection{Multiclass classification}
\label{background} 
In multiclass classification, we are given a training sample $S = 
\{(x_1, y_1), (x_2, y_2), \ldots (x_n, y_n)\}$ drawn from $\mathbb{P}^n$, where $\mathbb{P}$ is an unknown distribution over instances $\mathcal{X}$ and labels $\mathcal{Y} = [L] = \{1, 2, \ldots, L \}$. Our goal is to learn a predictor $f: \mathcal{X}  \rightarrow \mathbb{R}^{L}$, namely to minimize the \emph{risk} of $f$. The latter is defined as the expected loss of $f$:
\begin{align}\label{risk}
R(f) = \ex[ \ell(y, f(x) ) ]     
\end{align}
where $(x,y)$ is drawn from $\mathbb{P}$, and $\ell: [L] \times \mathbb{R}^L \rightarrow \mathbb{R}_{+} $ is a loss function such that, for a label $y \in [L]$  and prediction vector $f(x) \in \mathbb{R}^L$, $\ell(y, f(x)) $ is the loss incurred for predicting $f(x)$ when the true label is $y$. The most common way to approximate the risk of a predictor $f$  is via the so-called \emph{empirical risk}:
\begin{align}\label{empirical_risk}
R_{S}(f) = \frac{1}{|S|} \sum_{(x,y)\in S} \ell(y, f(x) ).  
\end{align}
That is, given a hypothesis class of predictors $\mathcal{F}$, our goal is typically to find $\min_{f \in \mathcal{F}} R_S(f) $ as a way to estimate $\min_{f \in \mathcal{F}} R(f)$.

\subsection{\blue{Debiasing weights}} 
\label{adversarial_setting}

As we have already discussed, the main drawback of distillation with unlabeled examples is essentially that the empirical risk minimizer corresponding to the  dataset labeled by the teacher cannot be trusted, since the teacher may  generate inaccurate labels. To quantify this phenomenon and guide our algorithmic design, we consider the following simple but natural \blue{noise model for the teacher}.

Let  $\mathbb{X} $ be an unknown distribution over instances $\mathcal{X}$. We assume the existence of a \emph{ground-truth classifier} so that each $x \in \mathcal{X}$ is associated with a ground-truth label $f_{\mathrm{true}}(x) \in [L] = \{1, 2, \ldots, L \}$. In other words, a clean labeled example is of the form $(x, f_{\mathrm{true}}(x)) \sim \mathbb{P}$ (and $x \sim \mathbb{X}$). Additionally, we consider a \blue{stochastic} adversary that given an instance $x \in \mathcal{X}$, outputs a ``corrupted'' label $y_{\mathrm{adv}}(x)$ with probability $p(x)$, and the ground-truth label $f_{\mathrm{true}}(x) $ with probability $1-p(x)$. Let $ \mathbb{D}$ denote the induced adversarial distribution over instances and labels.

It is not hard to see that the empirical risk with respect to a  predictor $f$ and sample from $\mathbb{D}$ is \emph{not} an unbiased estimator of the risk  
\begin{align}\label{clean_risk}
R(f) = \ex_{ x \sim \mathbb{X} } [ \ell(f_{ \mathrm{true} }(x) , f(x) ) ] 
\end{align}
— see Proposition~\ref{expectations}.  On the other hand, the following \emph{weighted} empirical risk~\eqref{weighted_empirical_risk} is indeed an unbiased estimator of $R(f)$.
For each $x \in \mathcal{X}$ let
\begin{align}\label{weight_definition}
w_f(x) =     \frac{1}{ 1 + p(x) \left( \mathrm{distortion}_f(x) -1 \right) },
~~~
\text{where}
~~~
    \mathrm{distortion}_f(x) = \frac{ \ell(y_{\mathrm{adv}}(x) , f(x) )  }{  \ell(f_{\mathrm{true}}(x), f(x)   )    },
\end{align}
and define
\begin{align}\label{weighted_empirical_risk}
R_S^w(f)  = \frac{1}{|S|} \sum_{(x,y)\in S} w_f(x)\, \ell( y, f(x) ),     
\end{align}
where $S = \{(x_i, y_i) \}_{i=1}^n \sim \mathbb{D}^n$. Observe that the weight for each instance $x$ depends on (i) how likely it is the adversary corrupts its label; and on (ii) how this corrupted label ``distorts'' the loss we observe at instance $x$. 
In the following  proposition we establish that the standard (unweighted) empirical risk with respect to distribution $\mathbb{D}$ and a predictor $f$ is a biased estimator of the risk of $f$ under the clean distribution $\mathbb{P}$, while the weighted empirical risk~\eqref{weighted_empirical_risk} is  an unbiased one.

\begin{proposition}[Debiasing Weights]\label{expectations}
Let $S = \{(x_i, y_i) \}_{i=1}^n \sim \mathbb{D}^n$ be a sample from the adversarial distribution. \blue{Defining $\mathrm{Bias}(f) = \ex_{x\sim \mathbb{X} }  \left[   p(x) \cdot( \mathrm{distortion}_f(x) -1  ) \cdot \ell(f_{\mathrm{true}}(x), f(x))    \right]$ we have:}
\begin{align*}
\text{``reweighted''} ~~~
\ex[R_S^w( f)] =  R(f)
~~~~~~ \text{vs} ~~~~~~
\text{``standard''} ~~~ 
\ex [R_S(f) ] &=   R(f) + \mathrm{Bias}(f) \,.
\end{align*}
\end{proposition}
Notice that, as expected, the bias of the unweighted predictor is a function of the ``power'' of the adversary, i.e., a function of how often they can corrupt the label of an instance, and the ``distortion'' this corruption causes to the loss we observe. The proof of Proposition~\ref{expectations} follows from simple, direct calculations and it can be found in Appendix~\ref{main_proposition_proof}.

Intuitively, given a sufficiently large sample $ S \sim \mathbb{D}^n$, optimizing an unbiased estimator for the risk should provide a better approximation for $\min_{f \in F} R(f)$ compared to optimizing an estimator with constant bias. We formalize this intuition in Section~\ref{theoretical_aspects} and in Appendix~\ref{statistical_perspective}.
\subsection{Our method}
\label{our_method}

We consider the standard setting for distillation with unlabeled examples where we are given a  dataset $S_{\ell} = \{ (x_i, y_i) \}_{i=1}^m$ of  $m$ labeled examples from an unknown distribution $\mathbb{P}$, and a dataset   $S_u = \{ x_i \}_{i=m+1}^{m+n}$ of  $n$ unlabeled examples — typically, $n \ge m$. We also assume the existence of a (small) clean validation dataset $S_v = \{ (x_i, y_i) \}_{i = m+ n+1 }^{ i= m+n+q}   $ of size $q$. Finally, let $\ell: \mathbb{R}^L \times \mathbb{R}^L \rightarrow \mathbb{R}_+$ be a loss function that takes as input two  vectors over the set of labels $[L]$.
We describe our method below and more formally in Algorithm~\ref{alg:weighted_distillation} \blue{in Appendix~\ref{estimate_weights_algorithm}}.

\blue{
\begin{remark}
\label{remark_on_validation} 
The only assumption we need to make about the validation set $S_v$ is that it is not in the train set of the teacher model. That is, set $S_v$ can be used in the train set of the student model if needed — we chose to present $S_v$ as a completely independent hold out set to make our presentation as conceptually clear as possible.
\end{remark}
}

\noindent
\textbf{Training the teacher.}
The teacher model is trained on dataset $S_{\ell}$, and then it is used to generate labels for the instances in $S_u$. The labels can be one-hot vectors or probability distributions on $[L]$, depending on whether we apply ``hard'' or ``soft'' distillation, respectively. 

\noindent
\textbf{Training the student.} 
We start by pretraining the student model on dataset $S_{\ell}$. Then, the idea  is to think of the teacher model as the \blue{source of noise} in the setting of Section~\ref{adversarial_setting},  compute a weight $w_f(x)$ for each example $x$ based on~\eqref{weight_definition}, and finally train the student on the union of labeled and teacher-labeled examples by minimizing the weighted empirical risk (examples from $S_{\ell}$ are assigned unit weight).

We point out two remarks. First, in order to apply~\eqref{weight_definition} to compute the weight of an example $x$, we need to have  estimates of $p(x)$ and $\mathrm{distortion}_f(x)$. To obtain these estimates we use the validation dataset $S_v$ as we describe in the next paragraph. Second, observe that, according to~\eqref{weight_definition},  the weight of an example is a function of the predictor, namely the parameters of the model in the case of neural networks.\blue{ This means that ideally we  should be updating our weights assignment every time the parameters of the student model get updated during training. However,  we empirically observe that even  estimating the weight assignment \emph{only once} during the whole training process (right after training the student model on $S_{\ell}$)  suffices for our purposes, and so our method adds minimal overhead to the standard training process. More generally, the fact that our process of computing the weights is simple and inexpensive allows us to recompute them during training (say every few epochs) to improve our approximation of the theoretically optimal weights. We explore the effect of updating the weights during training in Section~\ref{updating_the_weights}.}

\noindent
\textbf{Estimating the weights.}
\blue{We estimate the weights using the Nearest Neighbors method on $S_v$ to learn a map  that takes as input the teacher's and student's  ``confidence'' for the label of a certain example $x$, and outputs estimates for $p(x)$ and $\mathrm{distortion}_f(x)$ so we can apply~\eqref{weight_definition}. In  our experiments we measure the confidence of a model either via the \emph{margin-score}, i.e., the difference between the largest two predicted classes for the label of a given example (see e.g. the so-called “margin” uncertainty sampling variant~\cite{roth2006margin}), or via the \emph{entropy} of its prediction. However, one could use \emph{any} metric (not necessarily confidence) that correlates well with the accuracy of the corresponding models.

More concretely, we reduce the task of estimating the weights to a two-dimensional (i.e., two inputs and  two outputs) regression task over the validation set which we solve using the Nearest Neighbors method. In particular, our Nearest Neighbor data structure is constructed as follows. Each example $x$ of the validation set is assigned the two following pairs of points: (i)  (teacher confidence at $x$, student confidence at $x$) — this is the covariate of the regression task; (ii)  ($1$, distortion at ($x$)) if the teacher correctly predicts the label of $x$, or ($0$, distortion at $x$), if the teacher does not correctly predict the label of $x$ — this is the response of the regression task. The query corresponding to an unlabeled example $x'$ is of the form (teacher confidence at $x'$, student confidence at $x'$). The Nearest Neighbors data structure returns the average response over the $k$ closest in euclidean distance pairs (teacher confidence at $x$, student confidence at $x$) in the validation set. The value of $k$ is specified in the next paragraph. The pseudocode for our method can be found in Algorithm~\ref{alg:estimating_weights} in Appendix~\ref{estimate_weights_algorithm}.
}

We point out two remarks. First, the number $k$ of neighbors we use for our weights-estimation is always $k =  \frac{\sqrt{|S_v|}}{2}  $.  \blue{This is because choosing $k  = \Theta(q^{2/(2 +\mathrm{dim} ) } )  $, where $q$ is the size of the validation dataset and  $\mathrm{dim}$ is the dimension of the underlying metric space ($\mathrm{dim}=2$ in our case),}  is asymptotically optimal, and $1/2$ is a popular choice for the hidden constant used in  practice, see e.g.~\cite{cover1967nearest, hastie2009elements}.  Second, notice that~\eqref{weight_definition} implies that the weight of an example could be larger than $1$ if (and only if) the corresponding distortion value~\eqref{weight_definition} at that example is less than $1$. This could happen for example if both the student and teacher have the same (or very similar) inaccurate prediction for a certain example. In such a case, the value of the weight in~\eqref{weight_definition} informs us that the loss at this example should be larger than the (low) value the unweighted loss function suggests. However, since we do not have the ground-truth label for a point during training — but only the inaccurate prediction of the teacher — having a weight larger than $1$ in this case would most likely guide our student model to fit an inaccurate label. To avoid this phenomenon, we always project our weights onto the $[0, 1]$ interval (Line~\ref{projection} of Algorithm~\ref{alg:estimating_weights}). In Appendix~\ref{MSE} we discuss an additional  reason why it is beneficial to project the weights of examples of low distortion onto  $[0,1]$  based on a MSE analysis.

\section{Experimental results}
\label{experimental_results} 

\blue{In this section we present our  experimental results. In Section~\ref{one-shot} we consider  an experimental setup according to which the weights are estimated only once during the whole training process. In Section~\ref{updating_the_weights} we study the effect of updating the weights during training.  Finally, in Section~\ref{fidelity_weights_comparison} we demonstrate that our method can be combined with uncertainty-based weighting techniques by both combining and comparing our method to the fidelity-based weighting scheme of~\cite{dehghani2017fidelity}.
\begin{figure*}[!ht]
  \centering
  \begin{minipage}[t]{0.33\textwidth}
  \centering
 \includegraphics[width=1\textwidth]{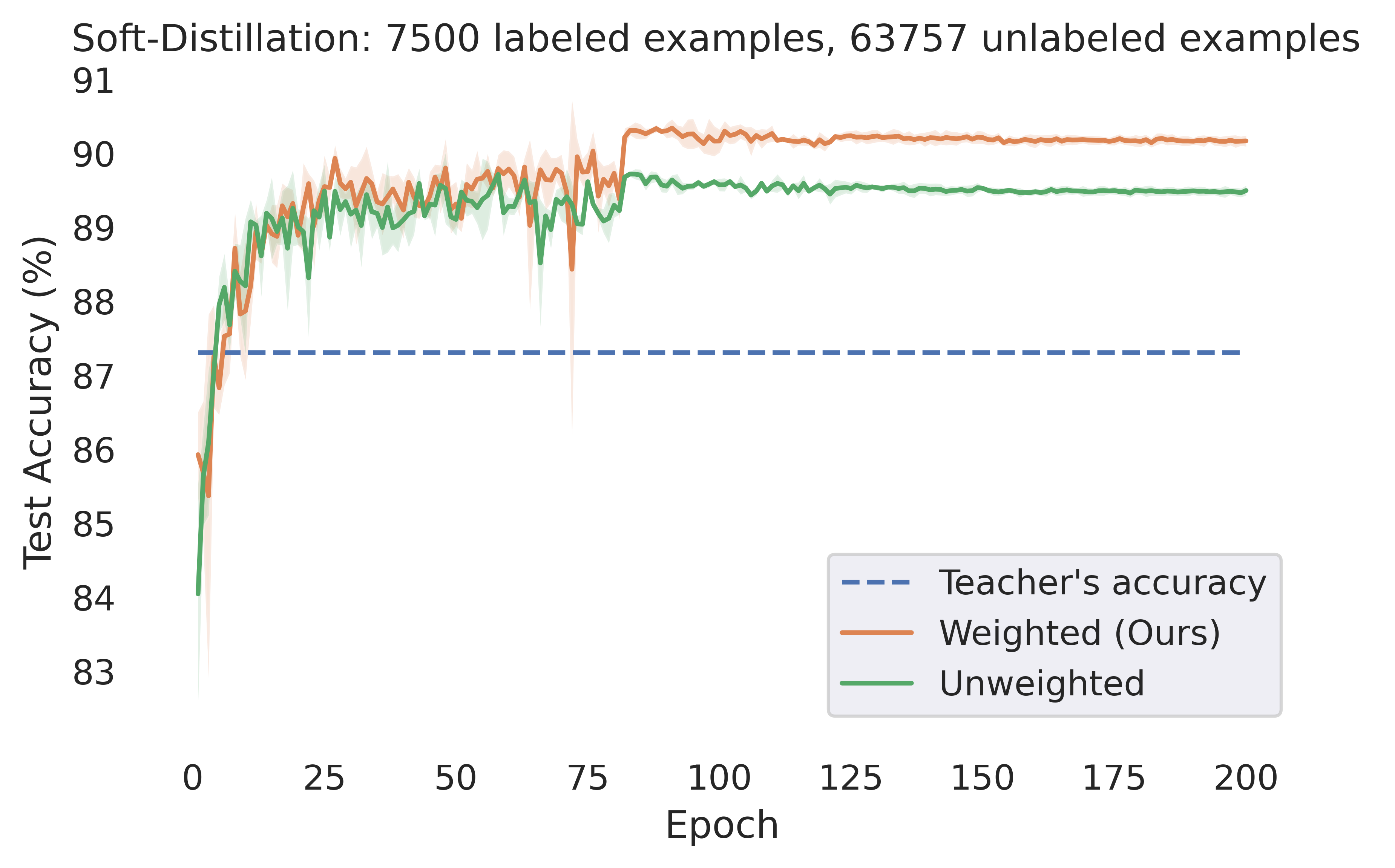} 
  \end{minipage}%
  \hfill
  \begin{minipage}[t]{0.33\textwidth}
  \centering
 \includegraphics[width=1\textwidth]{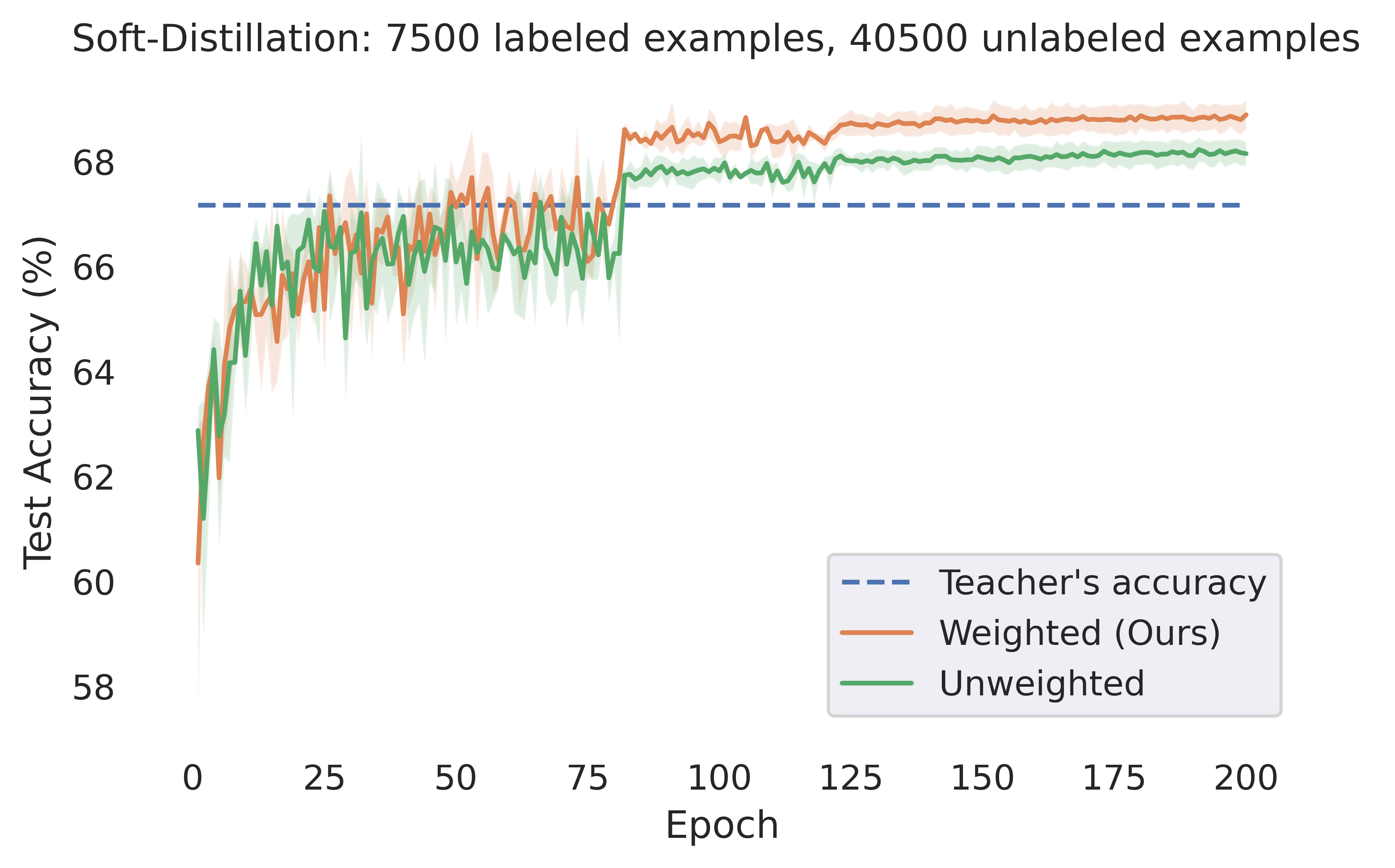} 
  \end{minipage}%
  \hfill
    \begin{minipage}[t]{0.33\textwidth}
  \centering
 \includegraphics[width=1\textwidth]{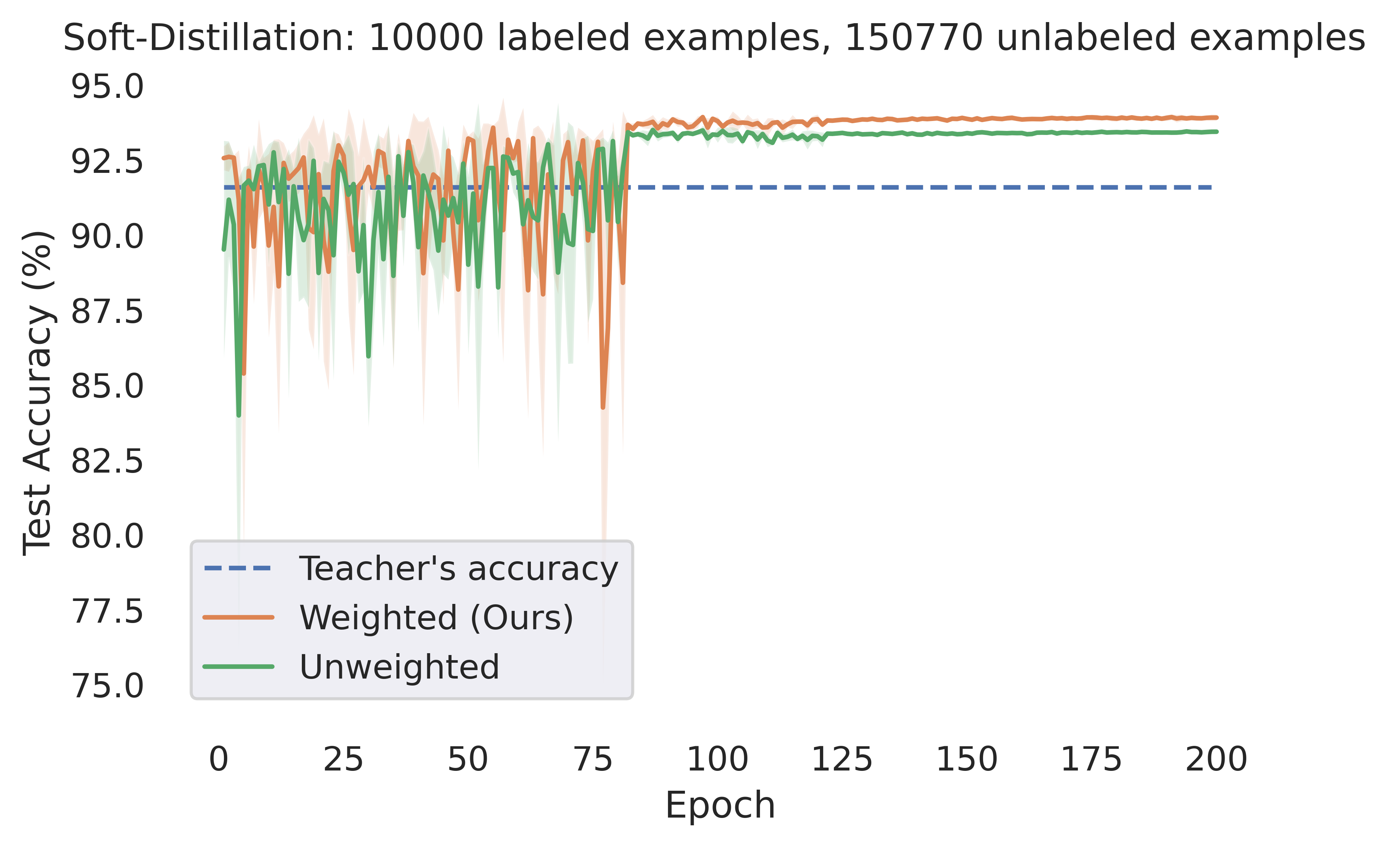}
  \end{minipage}%
  \hfill
  
  \centering
  \begin{minipage}[t]{0.33\textwidth}
  \centering
 \includegraphics[width=1\textwidth]{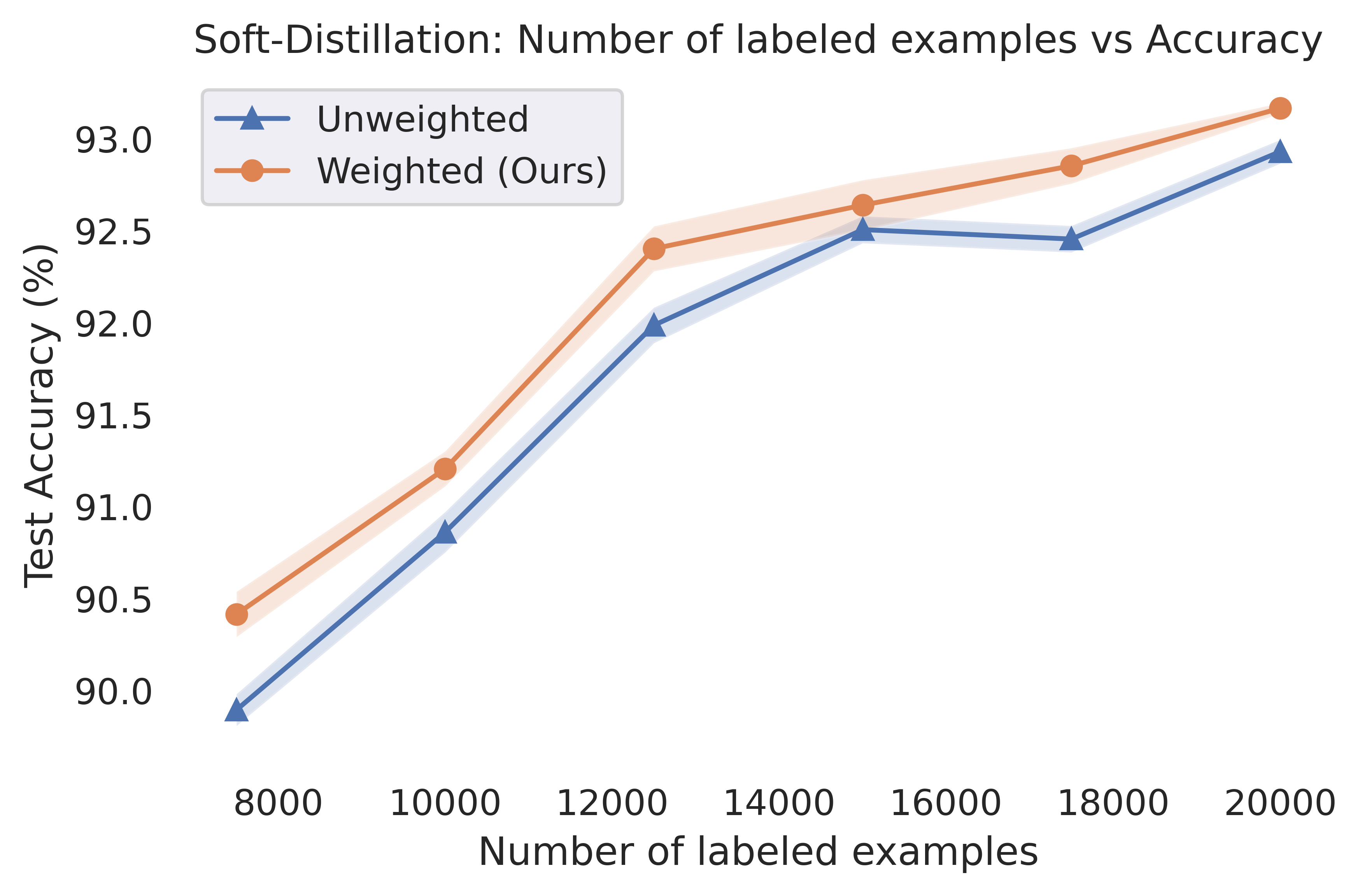} \\
 SVHN 
  \end{minipage}%
  \hfill
  \begin{minipage}[t]{0.33\textwidth}
  \centering
 \includegraphics[width=1\textwidth]{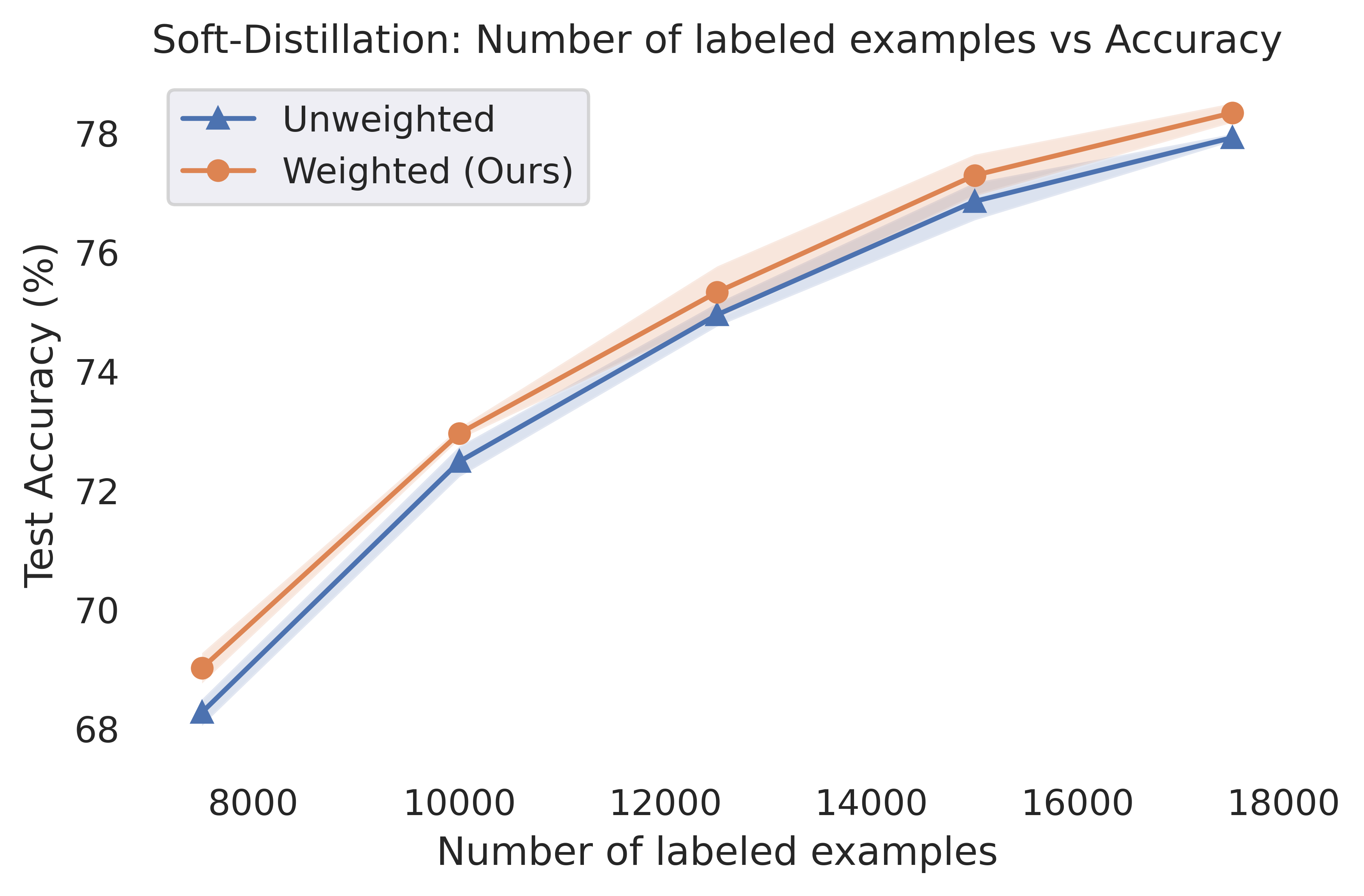} \\
 CIFAR-10
  \end{minipage}%
  \hfill
    \begin{minipage}[t]{0.33\textwidth}
  \centering
 \includegraphics[width=1\textwidth]{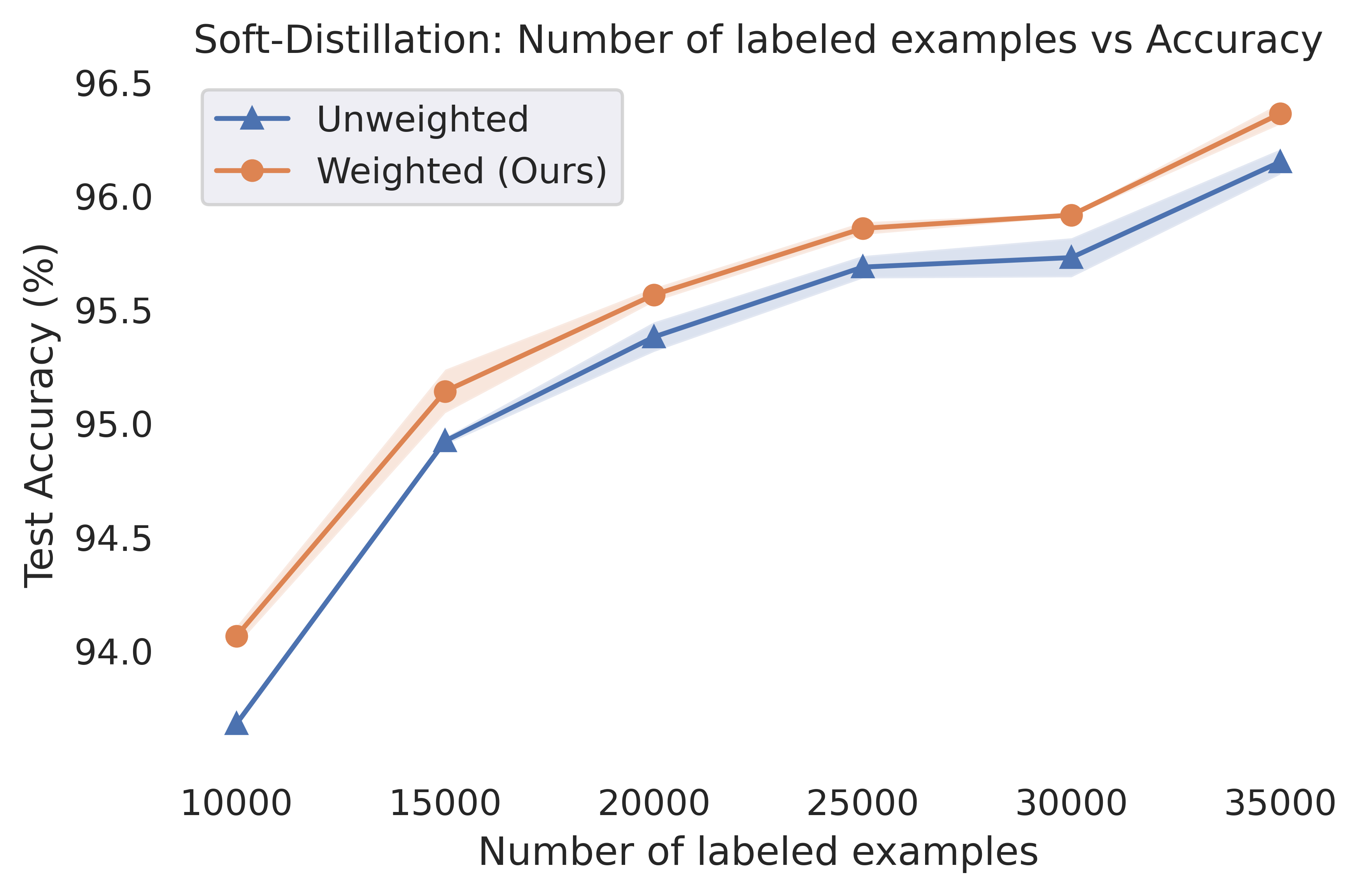}
  CelebA
  \end{minipage}%
  \hfill
    
\caption{\blue{The student's test accuracy over the training trajectory (first row)  and student's best test-accuracy  achieved over all (second row) when applying \textbf{one-shot estimation of the weights}. The teacher model is a MobileNet with depth multiplier $2$. In the cases of SVHN and CIFAR-10 the student model is a MobileNet of depth multiplier $1$, and in the case of CelebA is a ResNet-11. Our approach leads to consistently better models in terms of test-accuracy and convergence speed. Shaded regions correspond to values within one standard deviation of the mean.} 
 }
	\label{soft-training}
\end{figure*}

\begin{figure*}[!ht]

  \begin{minipage}[t]{0.5\textwidth}
  \centering
 \includegraphics[width=0.8\textwidth]{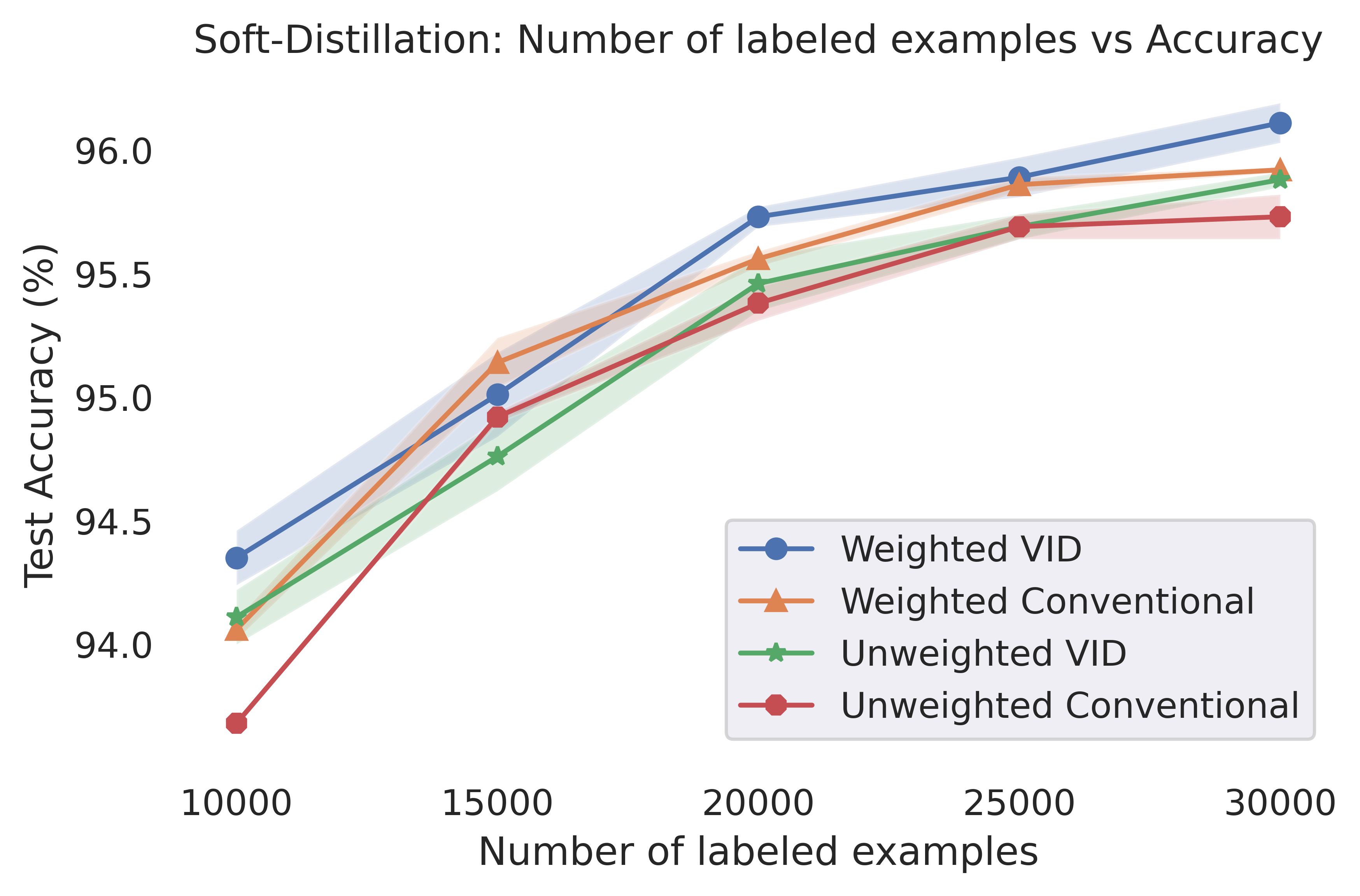} \\
 CelebA
  \end{minipage}%
  \hfill
    \begin{minipage}[t]{0.5\textwidth}
  \centering
 \includegraphics[width=0.8\textwidth]{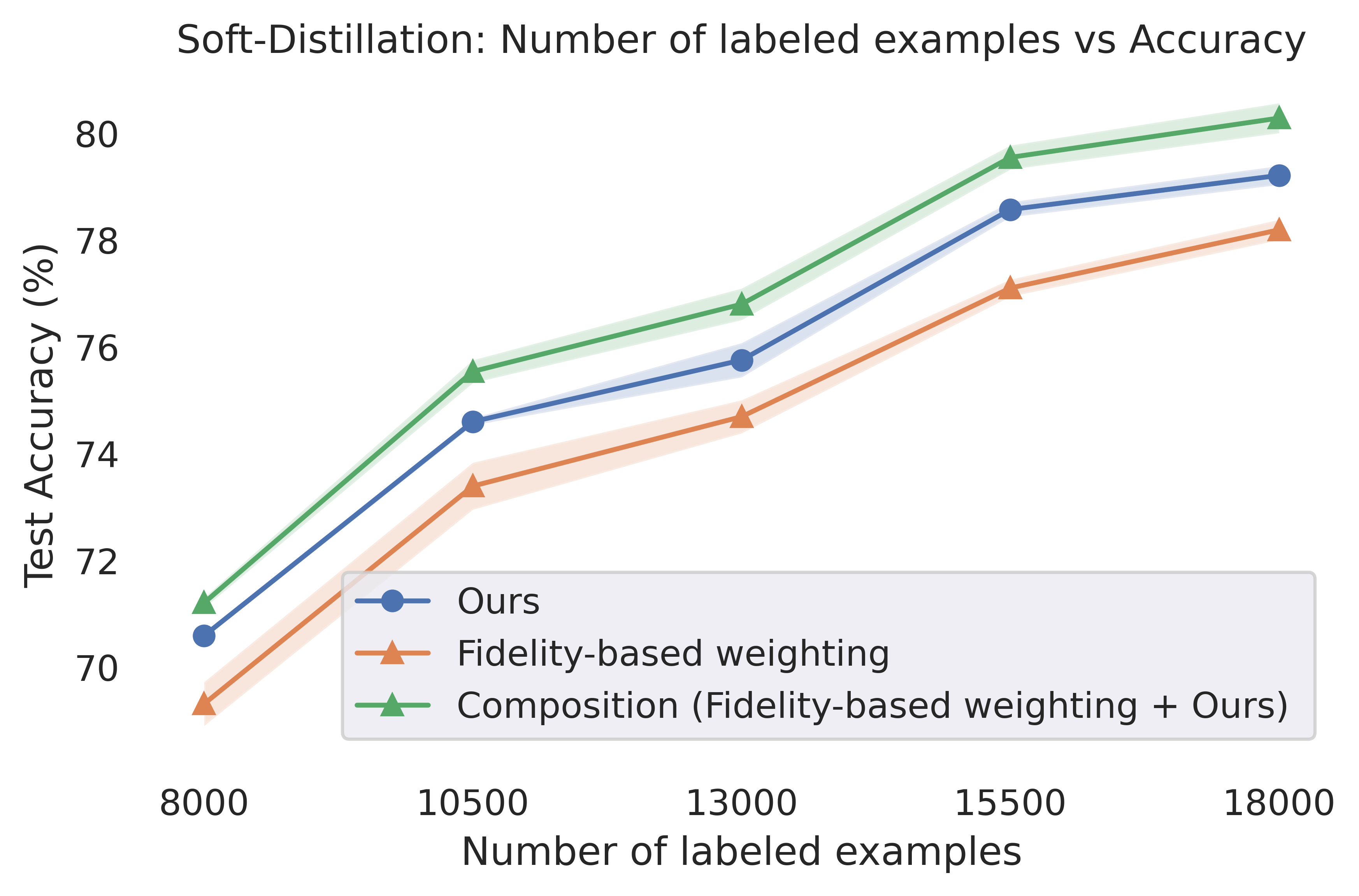}
 CIFAR-10
  \end{minipage}%

\caption{\blue{\textbf{Left:} We compare conventional distillation (unweighted conventional),  standard Variational Information Distillation (VID)~\cite{ahn2019variational} (unweighted VID), and reweighting the conventional loss function (weighted conventional) and the VID loss function  with our method (weighted VID) on CelebA. (These results correspond to \textbf{one-shot estimation of the weights}.) We see that our method can be combined effectively with more advanced distillation techniques such as VID. 
 \textbf{Right:} 
We combine and compare our method with the fidelity-based reweighting scheme of~\cite{dehghani2017fidelity} on CIFAR-10. We see that our method can be combined effectively with weighting schemes based only on the teacher's uncertainty. (These results correspond to \textbf{updating our weights at the end of every epoch}.)
}
 }
	\label{vid_and_fidelity}
\label{vid_and_updates} 
\end{figure*}

\subsection{Improvements via one-shot estimation of the weights}
\label{one-shot}

Here we show how applying our reweighting scheme  provides consistent improvements on several popular benchmarks even when the weights are estimated only  once during the whole training process. We compare our method against conventional distillation with unlabeled examples, but we also show that our method can provide benefits when combined with more advanced distillation techniques such as the framework of~\cite{ahn2019variational} (see Figure~\ref{vid_and_updates}).}

More concretely, we evaluate our method on benchmark vision datasets. We compare our method to conventional distillation with unlabeled examples both in terms of the best test accuracy achieved by the student, and in   terms of convergence speed (see  Figure~\ref{soft-training}). We also evaluate the comparative advantage of our method as a function of the number of labeled examples available (size of dataset $S_{\ell}$). We always choose the temperature in the softmax of the models to be $1$ for simplicity and consistency, \blue{and our metric for confidence is always the margin-score.}   Implementation details for our experiments and additional  results can be found in Appendices~\ref{extended_experiments} and~\ref{implementation}. We implemented all algorithms in Python making use of the TensorFlow  deep learning library~\cite{abadi2016tensorflow}. We use 64 Cloud TPU v4s each with two cores.

\subsubsection{Experimental setup}
\label{experimental_setup}
Our experiments are of the following form. The academic dataset we use each time is first split into two parts $A$ and $B$. Part $A$, which is typically smaller, is used as the labeled dataset $S_{\ell}$ where the teacher model is trained on (recall the setting we described in Section~\ref{our_method}). Part $B$ is randomly split again  into two parts which represent the unlabeled dataset $S_u$ and validation dataset $S_v$, respectively. Then, (i) the teacher and student models are trained once on the labeled dataset $S_{\ell}$; (ii) the teacher model is used to generate soft-labels for the unlabeled dataset $S_{u}$; (iii) we train the student model on the union of $S_{\ell}$ and $S_{u}$ using our method and conventional distillation with unlabeled examples. We repeat step (iii) a number of times: in each trial we partition part $B$ randomly and independently, and then the student model is trained using the (student-)weights reached after completing the  training on dataset $S_{\ell}$ in step  (i) as initialization, both for our method and the competing approaches.

\subsubsection{CIFAR-\{10, 100\} and SVHN experiments} 
\label{cifar10and100}
\label{SVHN_details}

SVHN~\cite{netzer2011reading}  is an image classification dataset where the task is to classify street view numbers ($10$ classes). The train set of SVHN contains $73257$ labeled images and its test set contains $26032$ images. We use a MobileNet~\cite{howard2017mobilenets} with depth multiplier $2$ as the teacher, and a MobileNet with depth multiplier $1$ as the student\footnote{Note that we see the student outperform the teacher here and in other experiments, as can often happen with distillation with unlabeled examples, particularly when the teacher is trained on limited data.}. The tables in Figure~\ref{SVHN_experiments} contain the results of our experiments (averages over $3$ trials). In each experiment we use the first $N \in \{7500, 10000, 12500, 15000, 17500, 20000 \}$ examples as the labeled dataset $S_{\ell}$, and then the rest $73257-N$ images are randomly split to a labeled validation dataset $S_v$ of size $2000$, and an unlabeled dataset $S_u$ of size $71257-N$.

CIFAR-10 and CIFAR-100~\cite{krizhevsky2009learning} are image classification datasets with $10$ and $100$ classes respectively. Each of them consists of $60000$ labeled images, which we split to a training set of $50000$ images, and a test set of $10000$ images. For CIFAR-10,  we use a Mobilenet with depth multiplier $2$ as the teacher, and a Mobilenet with depth multiplier $1$ as the student. For CIFAR-100, we use a ResNet-110 as a teacher, and a ResNet-56 as the student. We use a validation set of $2000$ examples. The results of our experiments (averages over $3$ trials) can be found in the tables of Figures~\ref{CIFAR10_experiments},~\ref{CIFAR100_experiments}.

\begin{figure}[!ht] 

\begin{center}
\tiny
\begin{tabular}{|c| c c c c c  |} 
 \hline
Labeled Examples  & $7500$ & $10000$ & $12500$ & $15000$ & $17500$  \\ [0.1ex] 
 \hline
 Teacher  & $87.31 \%$ & $88.5 \%$ & $88.45 \%$ & $91.38 \%$ & $91.32 \%$  \\ 
 \hline
 Weighted (Ours) & $\mathbf{90.41 \pm 0.12} \%$ & $ \mathbf{91.21  \pm 0.09\%} $ & $ \mathbf{92.4 \pm 0.12 \%}$ & $\mathbf{92.64 \pm 0.14 \%}$ & $\mathbf{92.85 \pm 0.1 \%}$  \\
 \hline
 Unweighted & $89.89 \pm 0.09\%$ & $90.86 \pm 0.11 \%$& $91.99 \pm 0.1 \%$ & $92.51 \pm 0.07\%$ & $92.45 \pm 0.07\%$   \\ [1ex] 
 \hline
\end{tabular}

\end{center}

\caption{Experiments on SVHN. See Section~\ref{SVHN_details} for details.}
\label{SVHN_experiments}
\end{figure}

\begin{figure}[!ht]

\begin{center}
\tiny
\begin{tabular}{||c| c c c c c  ||} 
 \hline
Labeled Examples  & $7500$ & $10000$ & $12500$ & $15000$ & $17500$ \\ [0.1ex] 
 \hline
 Teacher   & $67.17 \%$ & $71.39 \%$ & $74.69 \%$ & $77 \%$ & $78.46 \%$   \\ 
 \hline
 Weighted (Ours) & $\mathbf{69.01 \pm 0.25} \%$ & $ \mathbf{72.96  \pm 0.1\%} $ & $ \mathbf{75.33 \pm 0.43 \%}$ & $\mathbf{77.69 \pm 0.35 \%}$ & $\mathbf{78.34 \pm 0.16 \%}$   \\
 \hline
 Unweighted & $68.27 \pm 0.22\%$ & $72.48 \pm 0.26 \%$& $74.95 \pm 0.2 \%$ & $76.85 \pm 0.32\%$ & $77.92 \pm 0.06\%$  \\ [1ex] 
 \hline
\end{tabular}

\end{center}

\caption{Experiments on CIFAR-10. See Section~\ref{cifar10and100} for details.}
\label{CIFAR10_experiments}
\end{figure}

\begin{figure}[!ht]

\begin{center}
\tiny
\begin{tabular}{||c| c c c c c  ||} 
 \hline
Labeled Examples  & $7500$ & $10000$ & $12500$ & $15000$ & $17500$ \\ [0.1ex] 
 \hline
 Teacher   & $44.1 \%$ & $51.31\%$ & $55.54 \%$ & $59.05 \%$ & $62.17 \%$    \\ 
 \hline
 Weighted (Ours) & $\mathbf{46.81 \pm 0.34} \%$ & $ \mathbf{53.31  \pm 0.66\%} $ & $ \mathbf{57.5 \pm 0.3} \%$ & $\mathbf{60.94 \pm 0.32 \%}$ & $\mathbf{62.86 \pm 0.41 \%}$  \\
 \hline
 Unweighted & $46.29 \pm 0.04\%$ & $52.83 \pm 0.53 \%$& $56.89 \pm 0.65 \%$ & $60.73 \pm 0.03\%$ & $62.79 \pm 0.09\%$   \\ [1ex] 
 \hline
\end{tabular}

\end{center}

\caption{Experiments on CIFAR-100. See Section~\ref{cifar10and100} for details.}
\label{CIFAR100_experiments}
\end{figure}

\subsubsection{CelebA experiments: considering more advanced distillation techniques}\label{sec:celeba}

\blue{As we have already discussed in Section~\ref{Related_work}, our method can be combined with more advanced distillation techniques which aim to enforce greater consistency between the teacher and the student. We demonstrate this fact by implementing the method of Variational Information Distillation for Knowledge Transfer (VID)~\cite{ahn2019variational} and showing how it is indeed beneficial to combine it with our method. We chose the gender binary classification task of CelebA~\cite{furlanello2018born} as our benchmark (see details in the next paragraph), because it is known (see e.g.~\cite{muller2020subclass}) that the more advanced distillation techniques tend to be more effective when applied to classification tasks with few classes. In the table below, ``Unweighted VID'' corresponds to the implementation of loss described in equations (2), (4) and (6) of~\cite{ahn2019variational}, and ``Weighted VID'' corresponds to the reweighting of the latter loss using our method.}

CelebA~\cite{furlanello2018born} is a large-scale face attributes dataset with more than $200000$ celebrity images, each with forty attribute annotations. Here we consider the binary male/female classification task. The train set of CelebA contains $162770$ images and its test set contains $19962$ images. We use a MobileNet with depth multiplier $2$ as the teacher, and a ResNet-11~\cite{he2016deep} as the student. The tables in Figure~\ref{celeba_experiments} contain the results of our experiments (averages over $3$ trials). In each experiment we use the first $N \in \{10000, 15000, 20000, 25000, 30000, 35000\}$ examples as the labeled dataset $S_{\ell}$, and then the rest $162770-N$ images are randomly split to a labeled validation dataset $S_v$ of size $2000$, and an unlabeled dataset $S_u$ of size $160770-N$.

\begin{figure}[!ht]

\begin{center}
\tiny
\begin{tabular}{||c| c c c c c  ||} 
 \hline
Labeled Examples  & $10000$ & $15000$ & $20000$ & $25000$ & $30000$  \\ [0.1ex] 
 \hline
 Teacher   & $91.59 \%$ & $93.76\%$ & $94.41 \%$ & $94.86 \%$ & $94.92 \%$   \\ 
  \hline
 Weighted VID  & $	\mathbf{94.35 \pm 0.11\%}$ & $ 	95.01  \pm  0.17\% $ & $ \mathbf{95.73 \pm 0.04	 \%}$ & $\mathbf{95.89 \pm 0.08} \%$ & $\mathbf{96.11 \pm 0.08\%}$  \\
 \hline
 Weighted Conventional & $94.06 \pm 0.04 \%$ & $ \mathbf{95.14  \pm 0.1\%} $ & $ 95.56 \pm 0.03 \%$ & $\mathbf{95.86 \pm 0.03 \%}$ & $95.92 \pm 0.01 \%$  \\
 \hline 
 Unweighted VID~\cite{ahn2019variational} & $94.11 \pm 0.11\%$ & $94.76 \pm 0.14\%$& $95.46 \pm 0.11 \%$ & $95.69 \pm 0.05\%$ & $95.88 \pm 0.03 \%$  \\
 \hline
 Unweighted Conventional & $93.68 \pm 0.01\%$ & $94.92 \pm 0.02 \%$& $95.38 \pm 0.07 \%$ & $95.69 \pm 0.05\%$ & $95.73 \pm 0.09\%$ \\ [1ex] 
 \hline
\end{tabular}

\end{center}

\caption{Experiments on CelebA. See Section~\ref{sec:celeba} for details and also Figure~\ref{vid_and_fidelity}.}
\label{celeba_experiments}
\end{figure}

\subsubsection{ImageNet experiments}
\label{imagenet_experiments}

ImageNet~\cite{russakovsky2015imagenet} is a large-scale image classification dataset with $1000$ classes consisting of approximately $I \approx 1.3$M images. For our experiments, we use a ResNet-50  as the teacher, and a ResNet-18 as the student. In each experiment we use the first $N \in \{64058, 128116 \}$ labeled examples ($5\%$ and $10\%$ of $I$, respectively) as the labeled dataset $S_{\ell}$, and the rest $I-N$ examples are randomly split to a labeled validation dataset $S_v$ of size $10000$, and an unlabeled dataset $S_u$ of size $I- N - 10000$. The results of our experiments (averages over 10 trials) can be found in Figure~\ref{ImageNet_experiments}.

\begin{figure}[!ht]

\parbox{.45\linewidth}{
\centering
\tiny
\begin{tabular}{||c| c c  ||} 
 \hline
Labeled Examples  &  $5\%$ of ImageNet &  $10\%$ of ImageNet \\ [0.1ex] 
 \hline
 Teacher (soft)   &  $36.74\%$ & $51.88\%$   \\ 
 \hline
 Weighted (Ours) &  $\mathbf{38.60 \pm 0.07} \%$   &    $\mathbf{53.59 \pm 0.09} \%$  \\
 \hline
 Unweighted  & $38.44 \pm 0.06 \%$   &     $53.43 \pm  0.08 \% $  \\ [1ex] 
 \hline
\end{tabular}}
\hspace{0.5cm}
\parbox{.45\linewidth}{
\centering
\tiny
\begin{tabular}{||c| c c  ||} 
 \hline
Labeled Examples  &  $5\%$ of ImageNet & $10\%$ of ImageNet \\ [0.1ex] 
 \hline
Teacher (hard)   &  $36.74\%$ & $51.88\%$   \\ 
 \hline
 Weighted (Ours) &  $\mathbf{38.56 \pm 0.07\%}$   &    $\mathbf{53.34 \pm 0.11\%} $  \\
 \hline
 Unweighted  & $38.42 \pm 0.06 \%$   &     $53.18 \pm  0.07 \% $  \\ [1ex] 
 \hline
\end{tabular}
}
\caption{Experiments with soft-distillation (left) and hard-distillation (right) on ImageNet.}
\label{ImageNet_experiments}
\end{figure}

\blue{
\subsection{Updating the weights during training}
\label{updating_the_weights}

In this section we consider the effect of updating our estimation of the optimal  weights during training. For each dataset we consider, the experimental setup is  identical to the  corresponding setting in Section~\ref{one-shot}, except for that we always use a validation set of size $500$ and the entropy of a model's prediction as the metric for its confidence. We note that the time required for computing the weight for each example is insignificant compared to total training time  (less than 1\% of the total training time in all of our experiments), which allows us to conduct experiments in which we update our estimation at the end of every epoch. We see that doing this typically  significantly improves the resulting student's performance (however,  in CIFAR-100 we do not observe substantial benefits).}

\begin{figure*}[!ht]

  \begin{minipage}[t]{0.33\textwidth}
  \centering
 \includegraphics[width=1\textwidth]{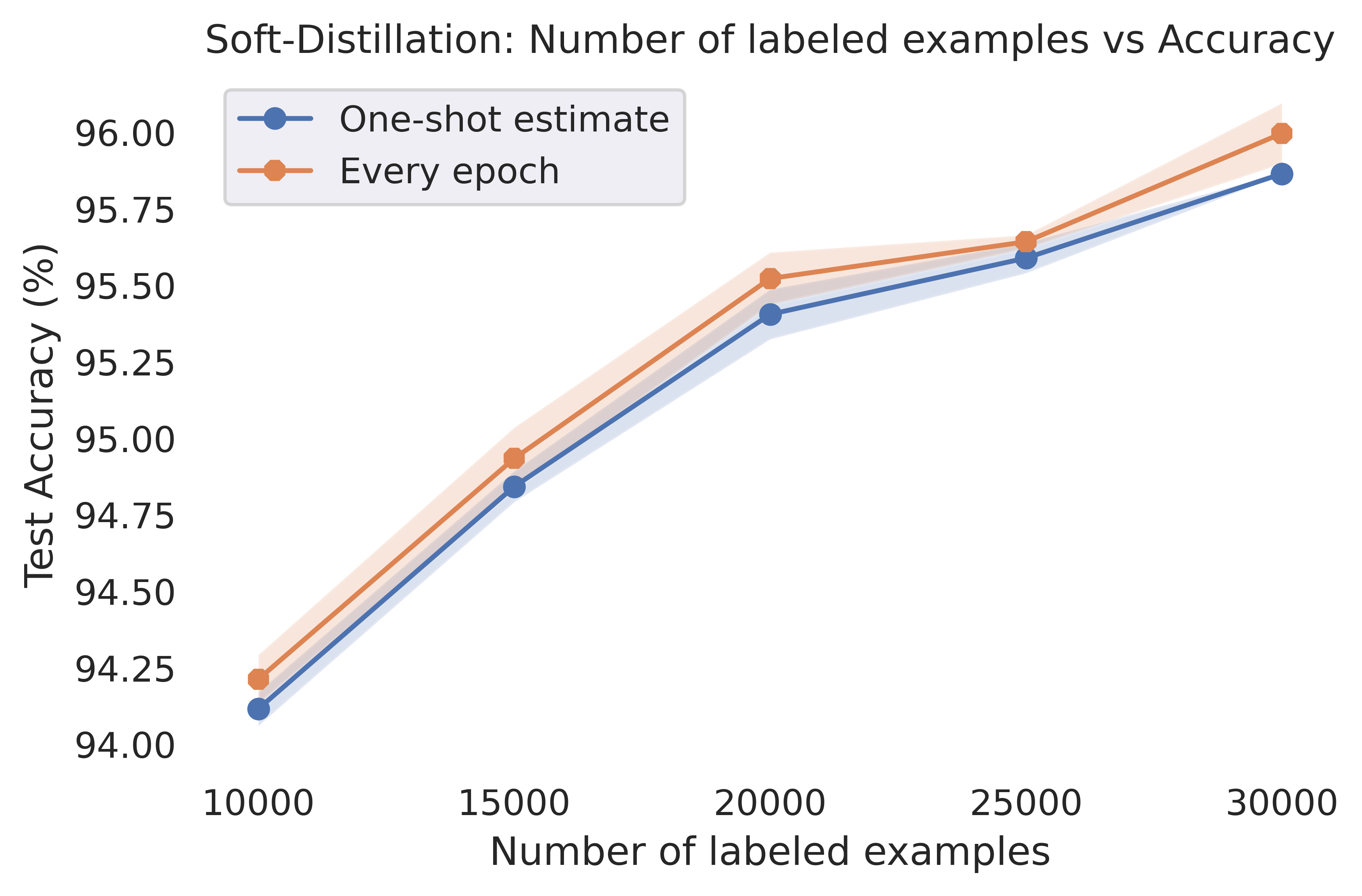} \\
 CelebA
  \end{minipage}%
  \hfill
    \begin{minipage}[t]{0.33\textwidth}
  \centering
 \includegraphics[width=1\textwidth]{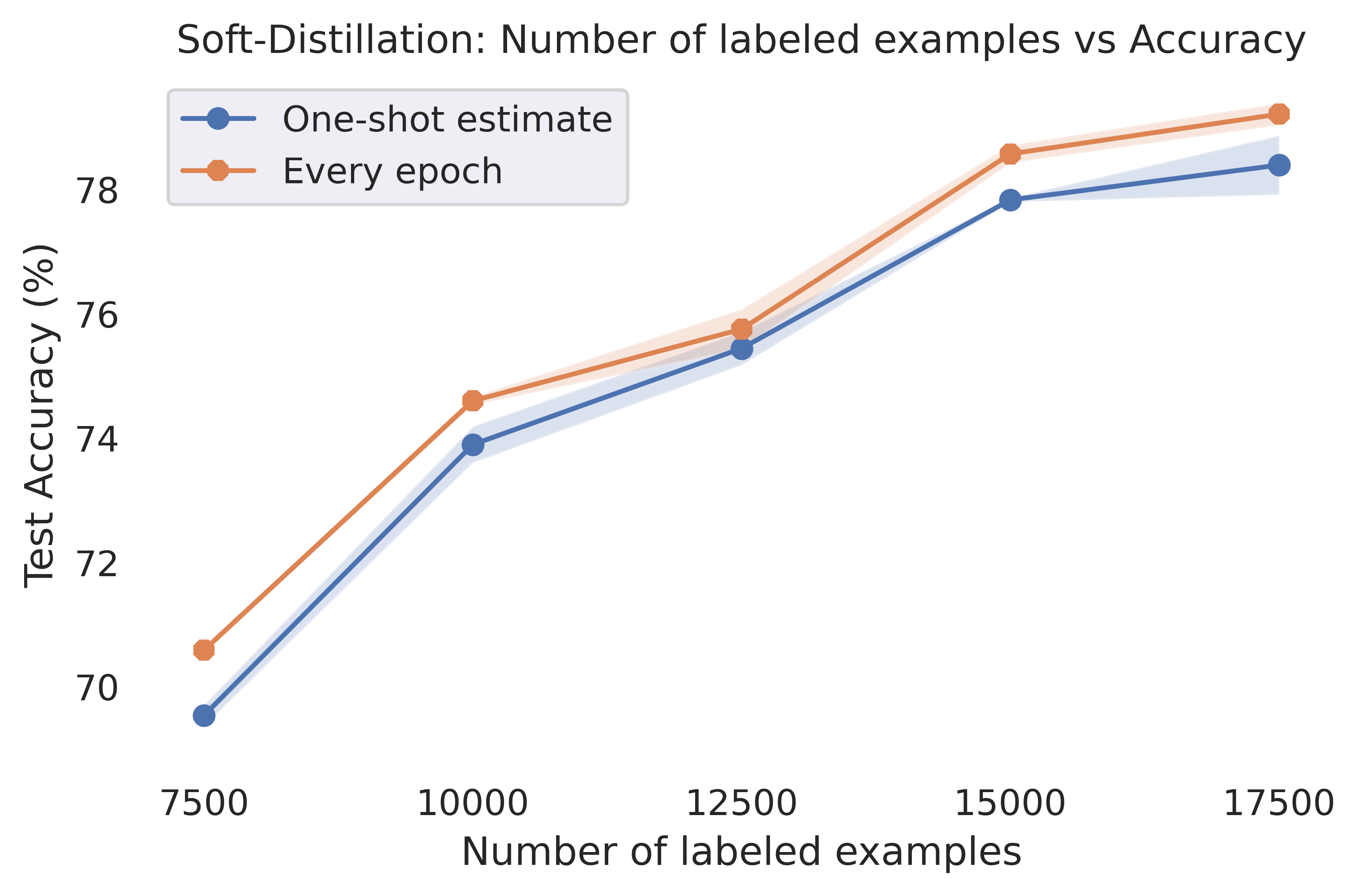}
 CIFAR-10
  \end{minipage}%
  \hfill
    \begin{minipage}[t]{0.33\textwidth}
  \centering
 \includegraphics[width=1\textwidth]{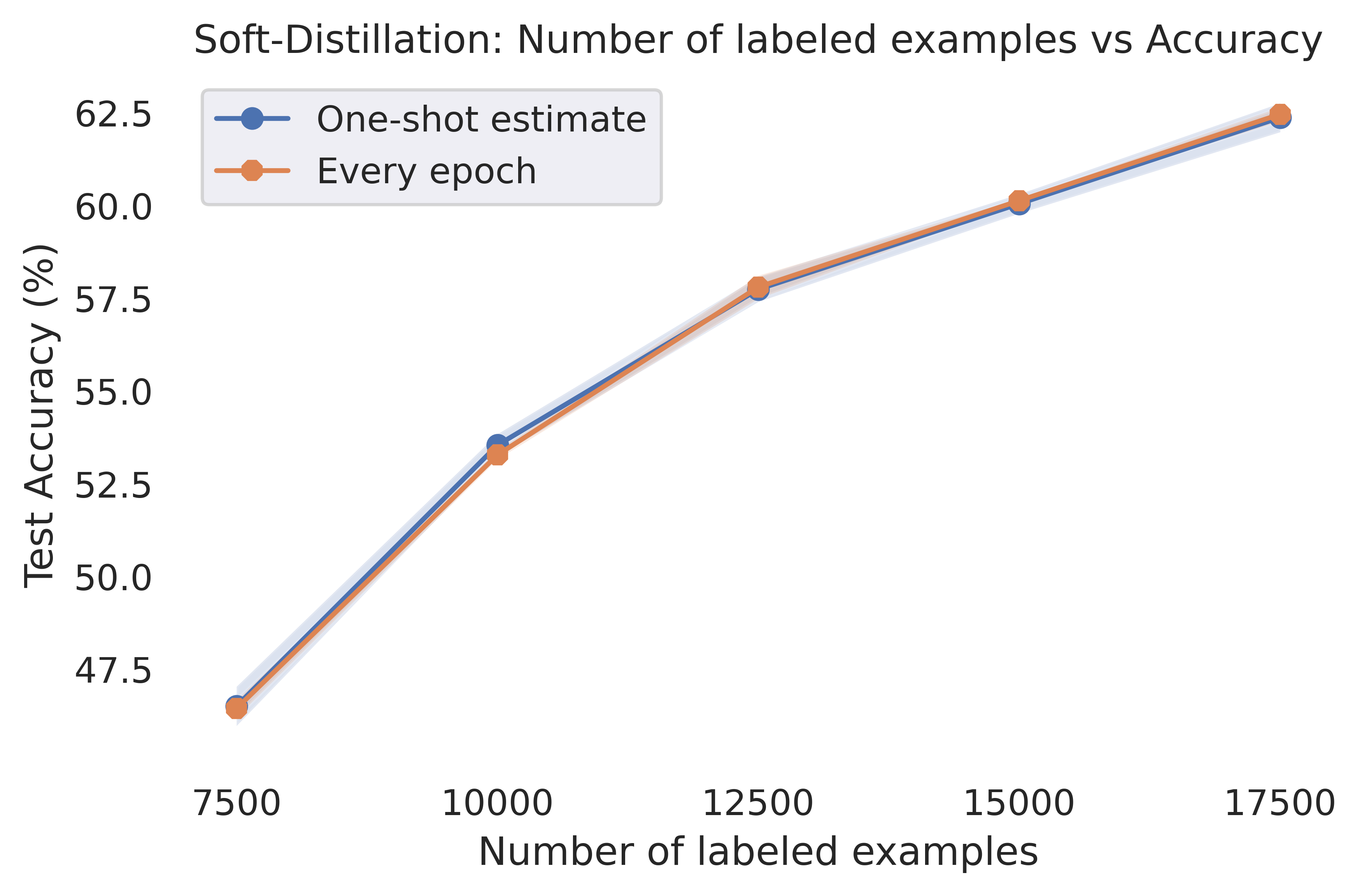}
 CIFAR-100
  \end{minipage}%

\caption{The effect of updating the weights during training. We compare our method when (i) weights are estimated only once and; (ii) when we update our estimation  at the end of every epoch. 
 }
	\label{updating_the_weights_figure}
\label{vid_and_updates} 
\end{figure*}

\blue{

\subsection{Combining and comparing with uncertainty-based weighting schemes}
\label{fidelity_weights_comparison}

 As we have already discussed in Section~\ref{Related_work}, our method can be combined with uncertainty-based weighting schemes as these are independent of the student model and, therefore, they can be  seen as a preprocessing step (modifying the loss function) before applying our method. We demonstrate this by combining and comparing our method with the filelity-based weighting scheme of~\cite{dehghani2017fidelity} on CIFAR-10 and CIFAR-100. Details about our experimental setup can be found in Appendix~\ref{fidelity_details}.
 }
\begin{figure}[!ht] 
\begin{center}
\tiny
\begin{tabular}{||c| c c c c c  ||} 
 \hline
Labeled Examples  & $7500$ & $10000$ & $12500$ & $15000$ & $17500$ \\ [0.1ex] 
 \hline
 Teacher (soft)   & $67.55 \%$ & $72.85 \%$	&  $ 74.85\%$ & $77.63\% $ & 	$78.40\%$ \\ 
 \hline
 Our Method & $70.59 \pm 0.05 \% $& $ 74.59 \pm 0.07 \% $& $75.75 \pm 0.32\% $&$ 78.56 \pm 0.14\% $ & $ 79.21 \pm 0.17 \%$  \\
 \hline
 Fidelity-based weighting~\cite{dehghani2017fidelity} & $69.31 \pm 0.41\%$&	$73.39 \pm 0.44\%$ &	$74.69 \pm 0.31\%$	& $77.10 \pm 0.16\%$ &	$78.19 \pm 0.19 \%$ \\ [1ex] 
 \hline
  Composition& $\mathbf{71.20 \pm 0.097\%}$ & $\mathbf{75.54 \pm 0.21\%}$ &$\mathbf{ 76.80 \pm 0.29\%} $& $\mathbf{79.55 \pm 0.22\%}$ & $\mathbf{80.28 \pm 0.28 \%}$ \\ [1ex] 
\hline
\end{tabular}
\end{center}

\begin{center}
\tiny
\begin{tabular}{||c| c c c c c  ||} 
 \hline
Labeled Examples  & $7500$ & $10000$ & $12500$ & $15000$ & $17500$ \\ [0.1ex] 
 \hline
 Teacher (soft)   &$ 44.70 \%$ &	$51.45 \%$ &	$56.00 \%$&	$58.70\%$ &	$61.82\%$\\ 
 \hline
 Our Method & $46.18 \pm 0.25\%$ &	$53.15 \pm 0.16 \%$& $\mathbf{57.51 \pm 0.4\%}$ & $\mathbf{59.65 \pm 0.4\%}$ &	$62.02 \pm 0.25 \%$\\
 \hline
 Fidelity-based weighting~\cite{dehghani2017fidelity} & $46.32 \pm 0.23\%$&	$52.92 \pm 0.20\%$&	$57.40 \pm 0.47\%$ & 	$59.48 \pm 0.17 \%$ &	$61.39 \pm 0.09 \%$ \\ [1ex] 
 \hline
  Composition& $\mathbf{46.62 \pm 0.44\%}$ &	$\mathbf{53.39 \pm 0.16\%}$	& $57.11 \pm 0.41\%$ &	$\mathbf{59.63 \pm 0.15\%}$ &	$\mathbf{62.30 \pm 0.26\%}$
 \\ [1ex] 
\hline
\end{tabular}
\end{center}

\caption{Combining and comparing our method with the  weighting scheme of~\cite{dehghani2017fidelity} on CIFAR-10 (top table) and CIFAR-100 (bottom table). See also Figure~\ref{vid_and_fidelity}.}
\end{figure}

\section{Theoretical motivation}
\label{theoretical_aspects}

\renewcommand\vec[1]{#1}
\newcommand{\x}{\vec x}
\newcommand{\weight}{w}
\newcommand{\density}{\gamma}
\newcommand{\W}{ {\cal W} }
\newcommand{\lp}{\left}
\newcommand{\rp}{\right}
\newcommand\norm[1]{\left\| #1 \right\|}
\newcommand\snorm[2]{\left\| #2 \right\|_{#1}}
\newcommand{\proj}{\mathrm{proj}}
\newcommand{\me}{\mathrm{e}}
\newcommand{\mrm}{\mathrm}
\def\d{\mathrm{d}}
\newcommand{\normal}{\mathcal{N}}

\newcommand{\sample}[2]{#1^{(#2)}}
\newcommand{\tr}{\mathrm{tr}}
\newcommand{\bx}{\mathbf{x}}
\newcommand{\by}{\mathbf{y}}
\newcommand{\bv}{\mathbf{v}}
\newcommand{\bu}{\mathbf{u}}
\newcommand{\bz}{\mathbf{z}}
\newcommand{\bw}{\mathbf{w}}
\newcommand{\br}{\mathbf{r}}
\newcommand{\bp}{\mathbf{p}}
\newcommand{\bc}{\mathbf{c}}
\newcommand{\e}{\mathbf{e}}
\newcommand{\Sp}{\mathbb{S}}

\newcommand{\R}{\mathbb{R}}
\newcommand{\Z}{\mathbb{Z}}
\newcommand{\N}{\mathbb{N}}

\newcommand{\eps}{\epsilon}
\newcommand{\poly}{\mathrm{poly}}
\newcommand{\polylog}{\mathrm{polylog}}
\newcommand{\var}{\mathbf{Var}}
\newcommand{\cov}{\mathbf{Cov}}

\newcommand{\sgn}{\mathrm{sign}}
\newcommand{\sign}{\mathrm{sign}}
\newcommand{\calN}{{\cal N}}
\newcommand{\calL}{{\cal L}}
\newcommand{\opt}{\mathrm{opt}}

\newcommand{\LR}{\mathrm{LeakyRelu}}
\newcommand{\Ind}{\mathds{1}}
\newcommand{\1}{\Ind}
\newcommand{\matr[1]}{\boldsymbol{#1}}
\newcommand{\littleint}{\mathop{\textstyle \int}}
\newcommand{\littlesum}{\mathop{\textstyle \sum}}
\newcommand{\littleprod}{\mathop{\textstyle \prod}}
\newcommand{\wt}{\widetilde}
\newcommand{\wh}{\widehat}

\newcommand{\ltwo}[1]{\left\lVert#1\right\rVert_2}
\newcommand{\dotp}[2]{ #1 \cdot #2 }
\newcommand{\miscl}{\err_{0-1}^{\D}}
\newcommand{\wstar}{\bw^{\ast}}
\newcommand{\normald}[1]{\mathcal{N}_{#1}}

\newcommand{\ith}{^{(i)}}
\newcommand{\tth}{^{(t)}}
\newcommand{\Exx}{\E_{\x\sim \D_\x}}
\newcommand{\Ey}{\E_{(\x,y)\sim \D}}
\newcommand{\vpar}[2]{\vec #1^{\|_{\vec #2}}}
\newcommand{\vperp}[2]{\vec #1^{\perp{\vec #2}}}
\newcommand{\ocoreg}{\bar{R}}
\newcommand{\X}{\vec X_{1,2}^{(t)}}
\newcommand{\risk}{R}
\newcommand{\Lc}{\risk}
\newcommand{\Lw}{\risk^{\mathrm{weighted}}}
\newcommand{\Ln}{\risk^{\mathrm{naive}}}
\newcommand{\D}{\mathbb{D}}
\newcommand{\Dadv}{\mathbb{D}}
\newcommand{\Dcln}{\mathbb{P}}
\newcommand{\Dx}{\mathbb{X}}
\newcommand{\pr}{\mathbf{Pr}}
\renewcommand{\R}{\mathbb{R}}
\newcommand{\ftrue}{f_{\mathrm{true}}}
\newcommand{\E}{\ex}
\newcommand{\yadv}{y_{\mathrm{adv}}}

In this section we show that our debiasing reweighting method described 
in Section~\ref{adversarial_setting} comes with provable \textbf{statistical} and
\textbf{optimization} guarantees.
We first show that the method is statistically consistent in the sense that,
with a sufficiently large dataset, the reweighted risk converges to the 
true or ``clean'' risk.  We show (see Appendix~\ref{statistical_perspective}) the following convergence guarantee.

\begin{theorem}[Uniform Convergence of Reweighted Risk]
\label{stats_theorem}
Assume that $S$ is a dataset of i.i.d. ``noisy'' samples from $\D$.
Under standard capacity assumptions for the class of models  $\mathcal{F}$ and 
regularity assumptions for the loss $\ell(\cdot)$, 
for every $f \in \mathcal{F}$ it holds that 
\[
\lim_{|S| \to \infty} R_S^{\weight}(f) = R(f) 
~~~~
\text{and}
~~~~
\lim_{|S| \to \infty} R_S(f) = R(f) + \mathrm{Bias}(f)
\]
\end{theorem}
To prove our optimization guarantees we analyze the reweighted objective in the fundamental case where the model 
$f(\x;\vec \Theta)$ is linear, i.e., $f(\x; \vec \Theta) = \vec \Theta \x \in \R^L$, and the loss 
$\ell(\vec y, \vec z)$ is convex in $\vec z$ for every $\vec y$.  
In this case, the composition of the loss and the model $f(\x;\vec \Theta)$ is 
convex as a function of the parameter $\vec \Theta \in \R^{L \times d}$. 
Recall that we denote by $\ftrue(\x):\R^d \mapsto \R^L$ the ground truth classifier and by $\Dcln$ 
the ``clean'' distribution, i.e., a sample from $\Dcln$ has the form 
$(\x, \ftrue(\x))$ where $\x$ is drawn from a distribution $\Dx$
supported on (a subset of) $\R^d$.  
Finally, we denote by $\D$ the ``noisy'' labeled distribution on
$\R^d \times \R^L$ and assume that the $\x$-marginal of $\D$ is also $\Dx$.

We next give a general definition of debiasing weight functions,
i.e., weighting mechanisms that make the corresponding objective function an unbiased estimator of the clean objective 
$\Lc(\vec \Theta)$ for every parameter vector $\vec \Theta \in \R^d$.  
Recall that the weight function defined in Section~\ref{adversarial_setting} is debiasing.
\begin{definition}[Debiasing Weights]\label{def:debiasing-weights} 
We say that a weight function 
$\weight(\x, \yadv; \vec \Theta)$
is a debiasing weight function if it holds that
\[
\risk^\weight(\vec \Theta) \triangleq
\E_{(\x, \yadv) \sim \D}[\weight(\x,\yadv ; \vec \Theta) \ell(\yadv, f(\x;\vec \Theta))]
=
\Lc(\vec \Theta)
\,.
\]
\end{definition}
 Since the loss is convex in $\vec \Theta$, one could try to optimize 
 the naive objective that does not reweight and simply minimizes 
 $\ell(\cdot)$ over the noisy examples,
$\Ln(\vec \Theta) \triangleq \E_{(\x,\yadv) \sim \D}[\ell(\yadv, \vec \Theta \x)]$.  We show (unsurprisingly) that doing this is a bad idea: there are instances where optimizing the naive
objective produces classifiers with bad generalization error over clean examples.  For the formal statement and proof we refer the reader to Appendix~\ref{optimization_perspective}.
\begin{center}
    \emph{
        SGD on the naive objective $\Ln(\cdot)$ learns parameters with arbitrarily \\
        bad generalization error over the ``clean'' data.
    }
\end{center}
Our main theoretical insight is that optimizing linear models with 
the reweighted loss leads to parameters with almost optimal generalization.
\begin{center}
    \emph{
        Given a debiasing weight function $\weight(\cdot)$,
        SGD on the reweighted objective $\risk^{\weight}(\cdot)$ learns a parameter with almost optimal generalization error over the ``clean'' data.
    }
\end{center}

The main issue with optimizing the reweighted objective is that, in general, we have no guarantees
that the weight function preserves its convexity (recall that it depends on the parameter $\vec
\Theta$).  However, we know that its population version corresponds to the clean objective
$\Lc(\cdot)$ which is a convex objective.  We show that we can use the convexity of the underlying
clean objective to show results  for  stochastic gradient descent, by proving the following key structural property.

\begin{proposition}[Stationary Points of the Reweighted Objective Suffice (Informal)]
Let $S$ be a dataset of
$n =  \poly(d L / \eps) $ i.i.d.\ samples from the noisy distribution $\D$.
Let $\widehat{\vec \Theta}$ be any stationary point of the weighted objective
$\risk_S^\weight(\vec \Theta)$ constrained on the unit ball 
(with respect to the Frobenious norm $\|\cdot\|_F$).
Then, with probability at least $99\%$, it holds that
\[
\Lc( \widehat{\vec \Theta}) \leq
\min_{\|\vec \Theta\|_F \leq 1} \Lc(\vec \Theta)  + \eps \,.
\]
\end{proposition}

\section{Conclusion} 
\label{conclusion}
\blue{
We propose a principled reweighting scheme for distillation with unlabeled examples. Our method  is hyper-parameter free, adds minimal implementation overhead, and comes with theoretical guarantees.
We evaluated our method on standard benchmarks and we showed that it consistently provides significant improvements. We note that investigating improved data-driven ways  of estimating the  weights~\eqref{weight_definition} could be of interest, since potential inaccurate estimation of the weights is the main limitation of our work. We leave this question open for future work. 

}

\section{Acknowledgements}
We are grateful to anonymous reviewers for detailed comments and feedback.

\bibliographystyle{plain}
\bibliography{distillation}

\input{supplementary_material}

\end{document}

%% file: supplementary_material.tex
\appendix

\newpage
\section{\blue{Formal description of our method}}
\label{estimate_weights_algorithm}

In this section we present pseudocode for our method.

\begin{algorithm}
\begin{algorithmic}[1] 
\Require model $\mathrm{Teacher}$, model $\mathrm{Student}$,  labeled dataset $S_{\ell} = \{ (x_i, y_i)  \}_{i=1}^{m}$, unlabeled dataset $ S_u = \{ x_i \}_{i=m+1}^{m+n}$, validation dataset $S_v = \{ (x_i, y_i ) \}_{i=m+n+1}^{m+n+q }$, \blue{number of weights-estimating iterations $r$}

\State  Train $\mathrm{Teacher}$ and $\mathrm{Student}$ on $S_{\ell}$

\State  Use $\mathrm{Teacher}$ to  generate labels for $S_u$ to obtain set $S_u^{\ell} = \{ (  x, \mathrm{Teacher}(x) ) \mid x \in S_u    \}$

\For {$i=m+1$ to $n+m$}
    \State  $y_i \leftarrow \mathrm{Teacher}(x_i)$
\EndFor

\State $S \leftarrow \{ (x_i, y_i) \}_{i=1}^{m+n}$

\For { $i=1$ to $m$  }

\State $w(x_i) \leftarrow 1$

\EndFor

\For {$i=1$ to $r$ }

\State $ \{w(x_{m+1}), \ldots, w(x_{m+n} )   \}   \leftarrow $ {\sc EstimateWeights}($\mathrm{Teacher}$, $\mathrm{Student}$, $S_v$, $S_u^{\ell}$     )

\State Train $\mathrm{Student}$ on $S$ using the weighted empirical risk: 
\begin{align*}
    \frac{1}{m +n}    \sum_{i=1 }^{m+n} w(x_i) \ell(y_i, \mathrm{Student}(x_i) ) 
\end{align*}

\EndFor

\end{algorithmic}
\caption{Weighted distillation with unlabeled examples}
\label{alg:weighted_distillation}
\end{algorithm}

\begin{algorithm}[!ht] 
\begin{algorithmic}[1] 

\Procedure{EstimateWeights}{ $\mathrm{Teacher}, \mathrm{Student},  V, D  $    }

\State \Comment $V$ is the validation dataset and $D$ is the teacher-labeled dataset
 \State $U \leftarrow \emptyset$, $k \leftarrow \lceil \frac{1}{2} \sqrt{ |V | } \rceil  $
\For {every $(x, y) \in V $}
                        \State $X  \leftarrow (\mathrm{Confidence}(\mathrm{Teacher}(x)), \mathrm{Confidence}(\mathrm{Student}(x)) )  $
                        
             \If {$\argmax(\mathrm{Teacher}(x) ) = \argmax( y) $}:
                        \State $(p, \mathrm{distortion} ) \leftarrow (0, 1)$
            \Else:
                  \State  $(p, \mathrm{distortion} )  \leftarrow \left(1, \frac{ \ell(\mathrm{Teacher}(x), \mathrm{Student}(x) ) }{\ell(y, \mathrm{Student}(x) )   } \right) $              
              \EndIf
              \State $Y \leftarrow (p, \mathrm{distortion} )$
              \State $U \leftarrow U \cup \{ (X, Y) \} $
\EndFor

\State $\mathrm{Weights} =  \varnothing$  \Comment Initialize and empty list for the weights
\For {every $(x,y) \in D$  }
\State $ \mathrm{Query} \leftarrow (\mathrm{Confidence}(\mathrm{Teacher}(x)), \mathrm{Confidence}(\mathrm{Student}(x)) )  $
\State $ (\hat{p}, \hat{d} )       \leftarrow$  $k$-$\mathrm{NN}(U, \mathrm{Query}  )$ \Comment Predict $p(x)$ and  $\mathrm{distortion}_f(x)$ from the $k$ nearest neighbors of $\mathrm{Query}$ in $U$ 
                   
\State $w(x)  \leftarrow \min \left\{1, \frac{1 }{1 + \hat{p} ( \hat{d} -1 )  }    \right\} $        \label{projection}        
\State Append $w(x)$ to $\mathrm{Weights}$              
  
\EndFor

\State \Return $\mathrm{Weights} $
\EndProcedure

\end{algorithmic}
 \caption{Procedure for estimating the weights}
\label{alg:estimating_weights}
\end{algorithm}

\section{Extended experiments}
\label{extended_experiments}

\subsection{\blue{The student's test-accuracy-trajectory}} 

In this section we provide extended experimental results that show the student's test accuracy over the training trajectory corresponding to experiments we mentioned in Section~\ref{one-shot}. Notice that in the vast majority of cases our method significantly outperforms the conventional approach almost throughout the training process.


\begin{figure*}[!ht]
  \begin{minipage}[t]{0.25\textwidth}
  \centering
 \includegraphics[width=1\textwidth]{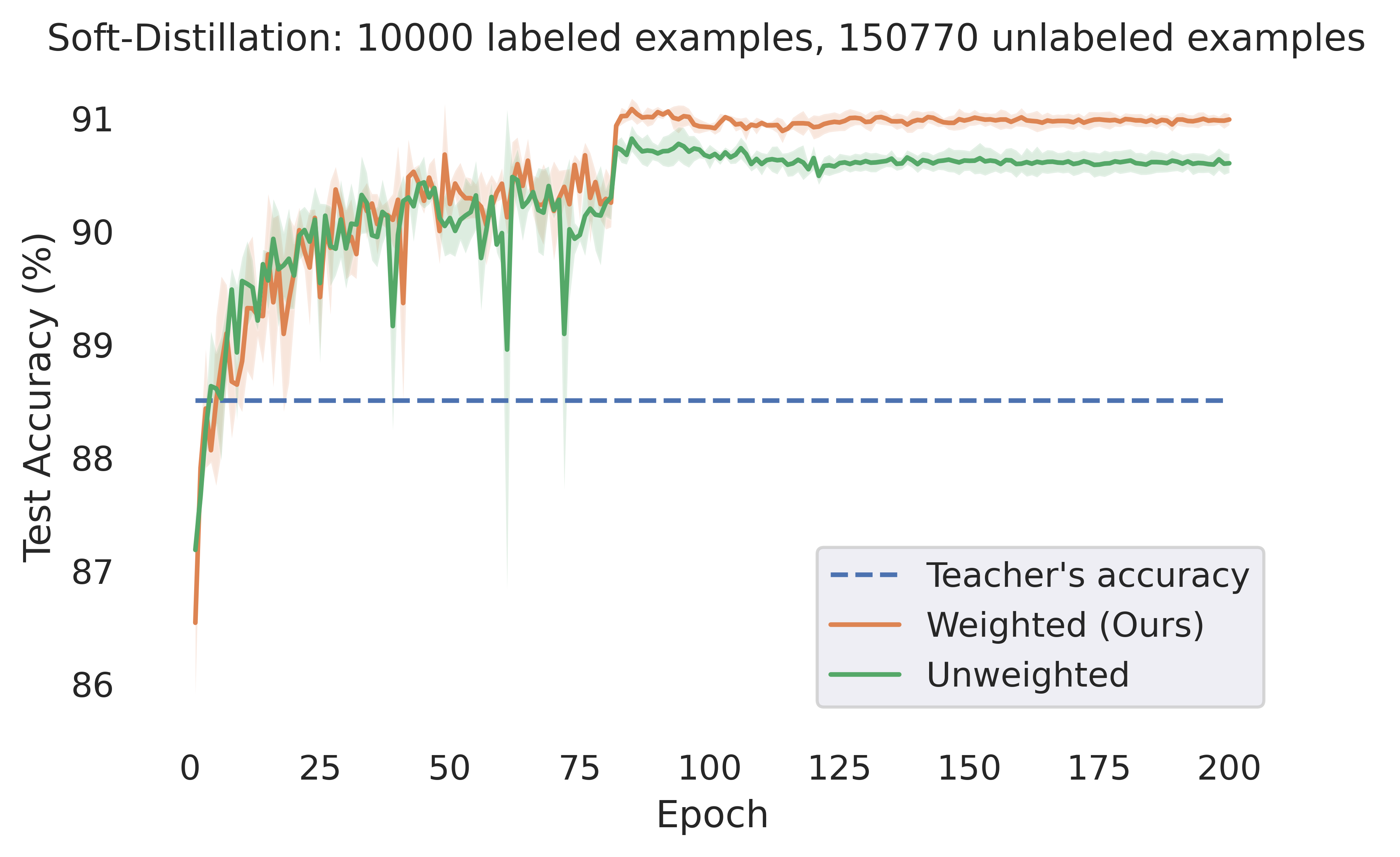} 
  \end{minipage}%
  \hfill
  \centering
  \begin{minipage}[t]{0.25\textwidth}
  \centering
 \includegraphics[width=1\textwidth]{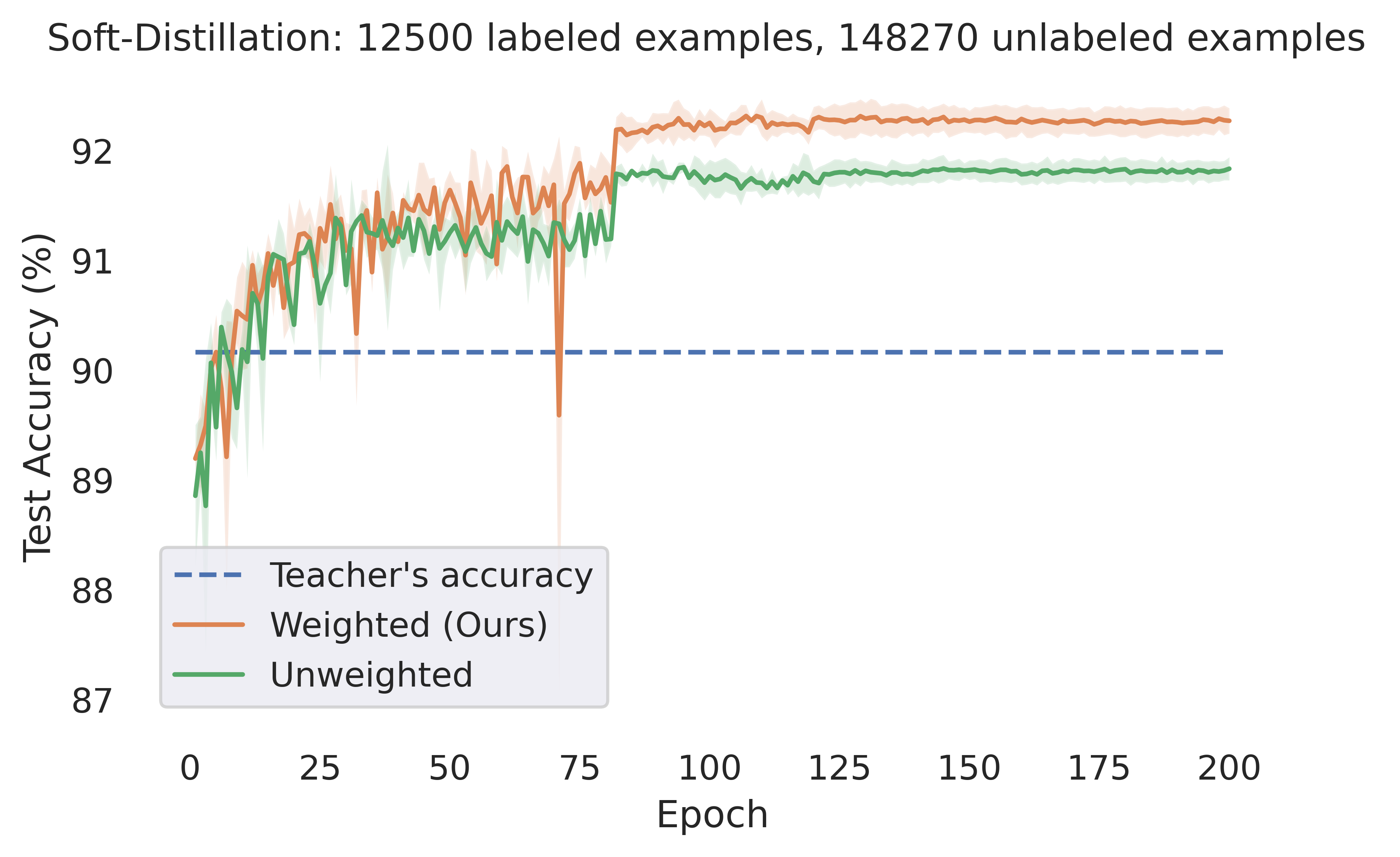} 
  \end{minipage}%
  \hfill
  \centering
  \begin{minipage}[t]{0.25\textwidth}
  \centering
 \includegraphics[width=1\textwidth]{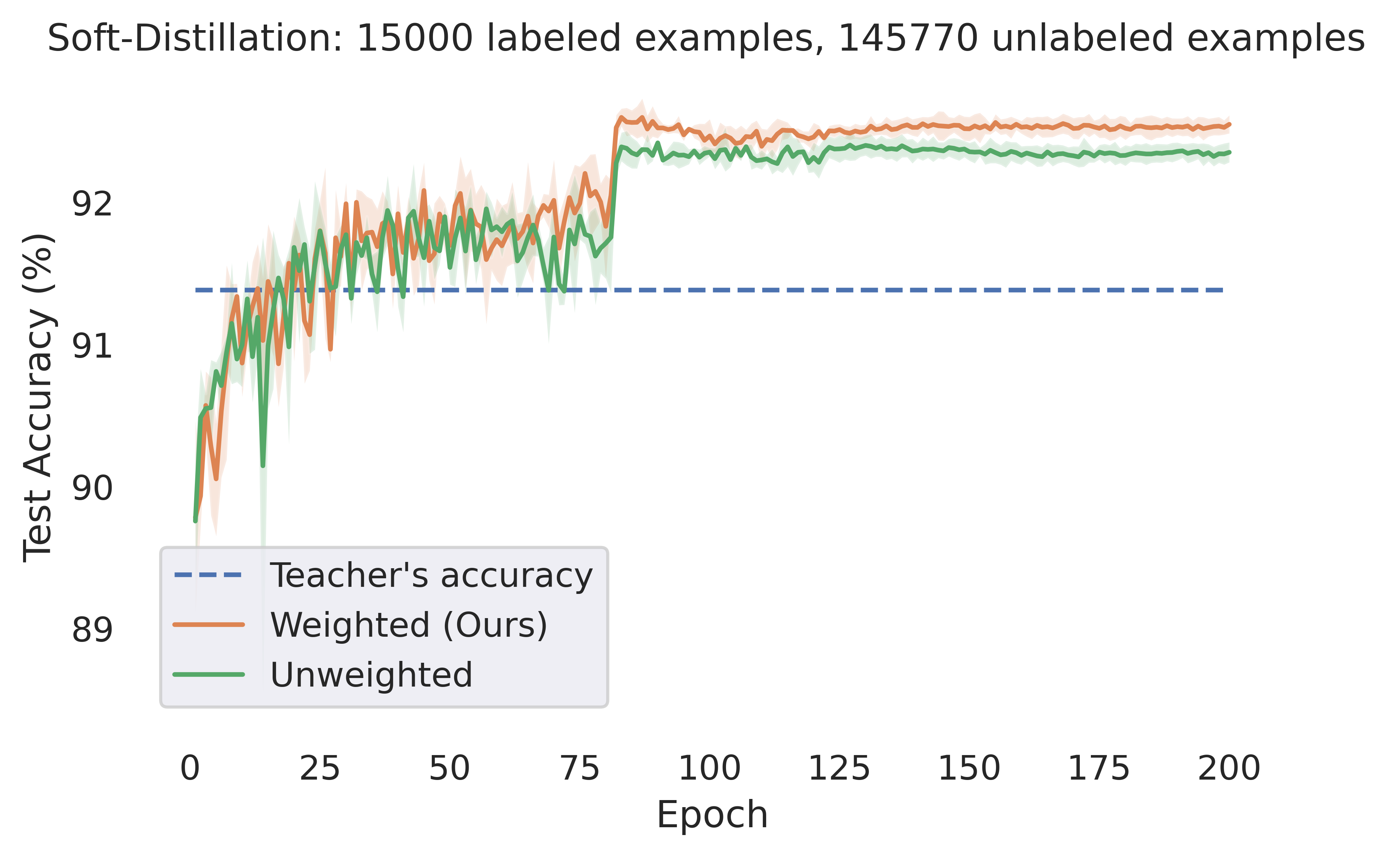} 
  \end{minipage}%
\hfill
  \centering
  \begin{minipage}[t]{0.25\textwidth}
  \centering
 \includegraphics[width=1\textwidth]{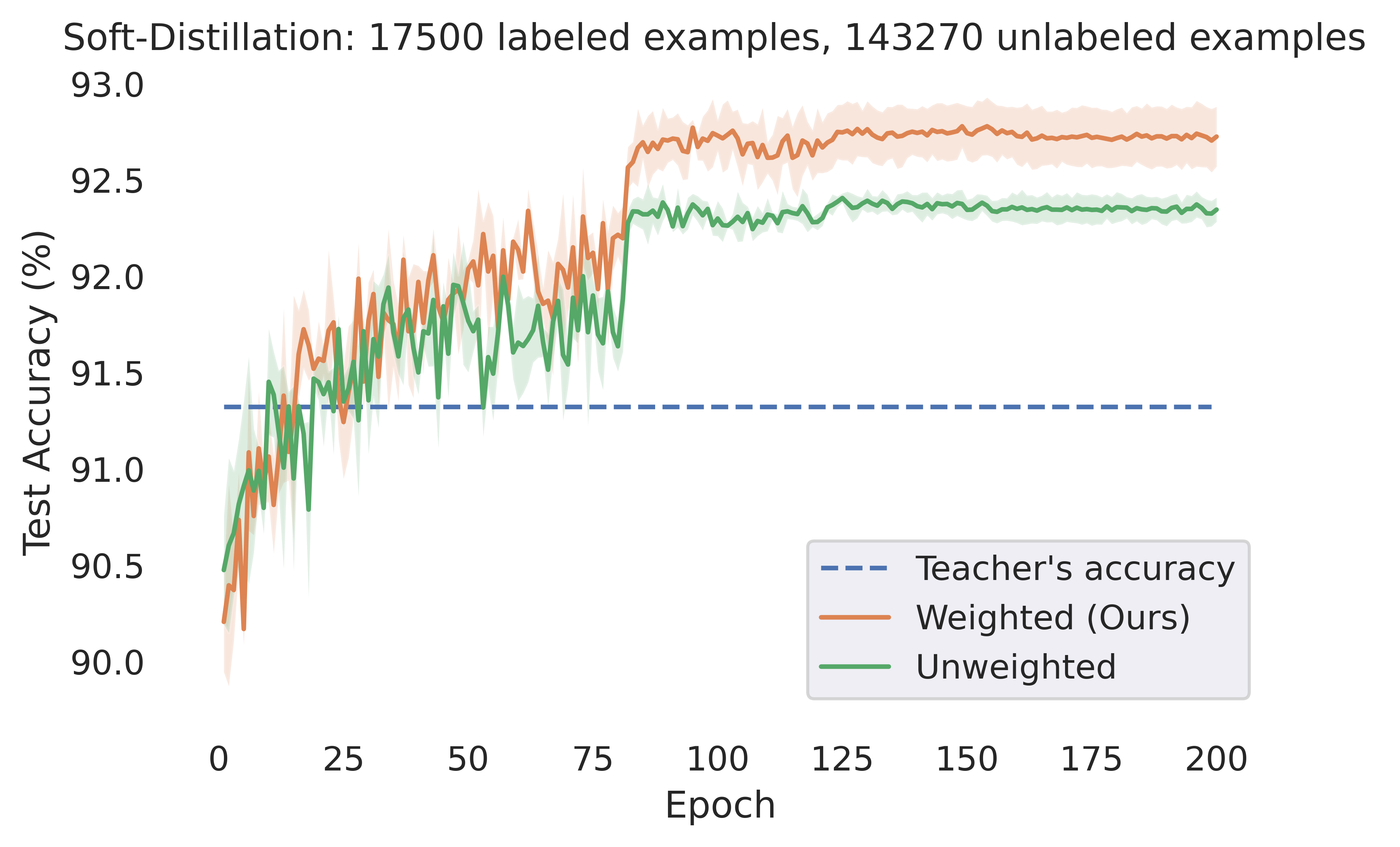} 
  \end{minipage}%

\caption{\textbf{SVHN} experiments. The student's test accuracy over the training trajectory using hard-distillation corresponding to the experiments of Figure~\ref{SVHN_experiments}. See Section~\ref{SVHN_details} for more details.
 }
	\label{SVHNFull}
\end{figure*}



\begin{figure*}[!ht]
  \centering
  \begin{minipage}[t]{0.25\textwidth}
  \centering
 \includegraphics[width=1\textwidth]{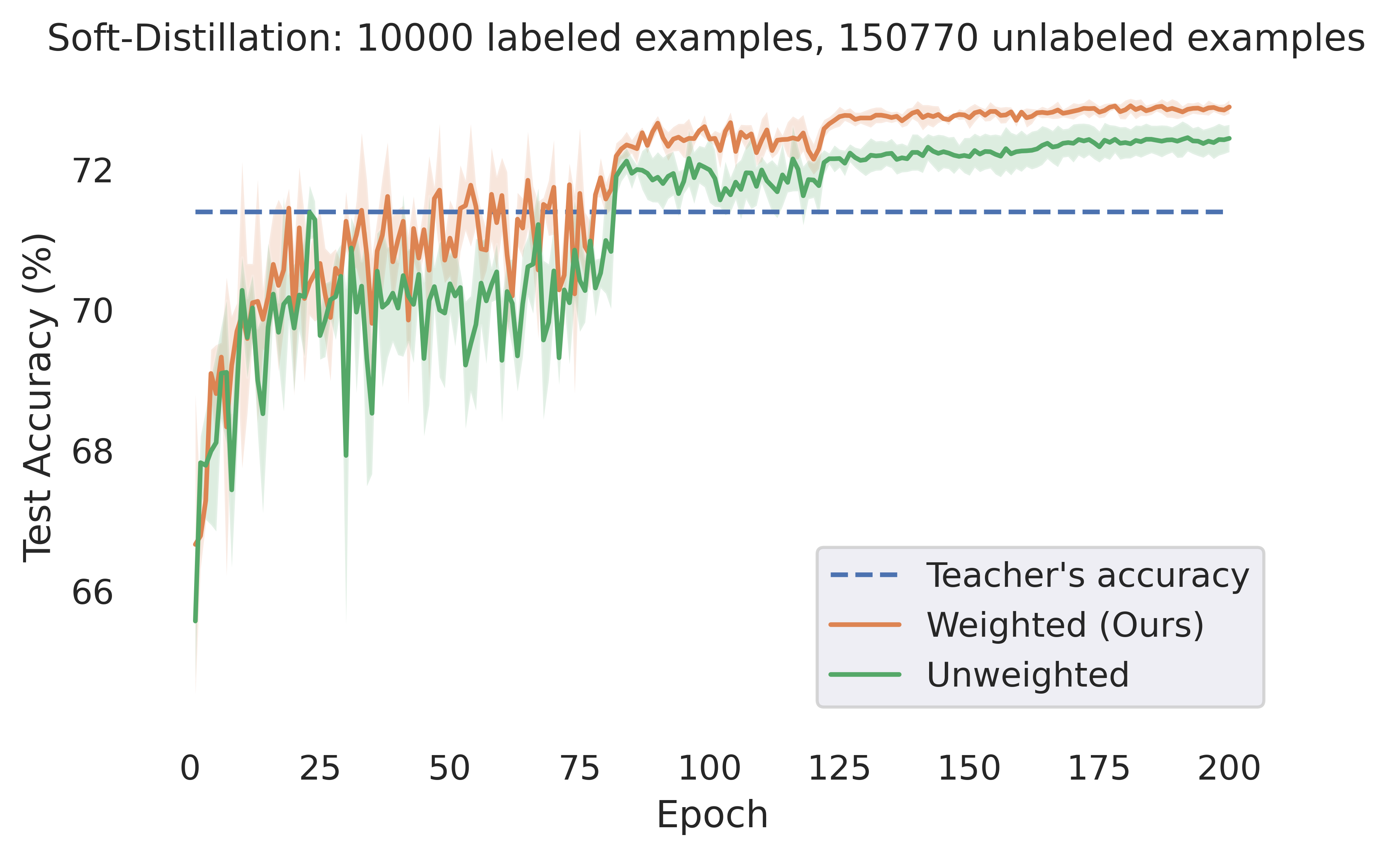} 
  \end{minipage}%
  \hfill
  \centering
  \begin{minipage}[t]{0.25\textwidth}
  \centering
 \includegraphics[width=1\textwidth]{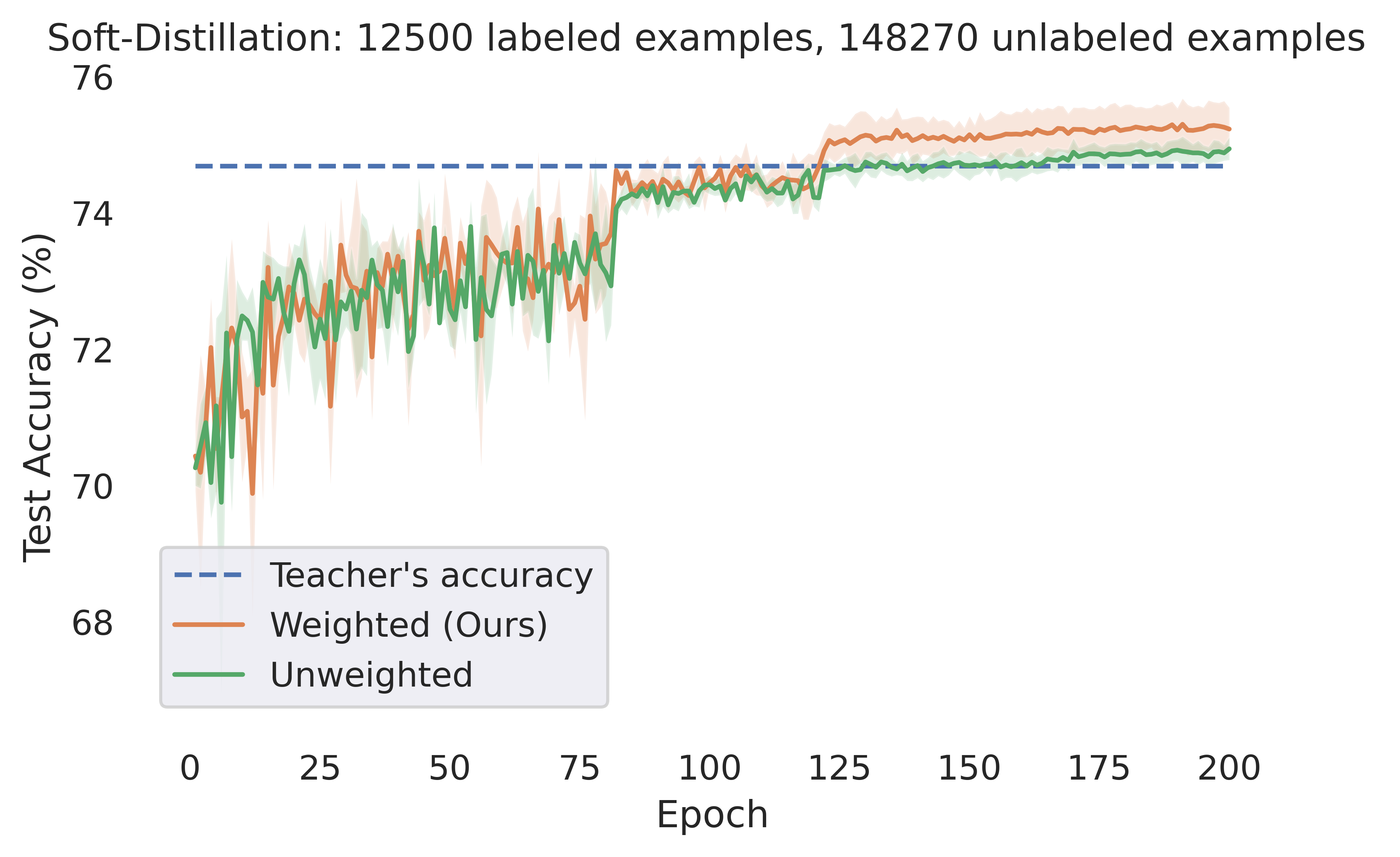} 
  \end{minipage}%
  \hfill
  \centering
  \begin{minipage}[t]{0.25\textwidth}
  \centering
 \includegraphics[width=1\textwidth]{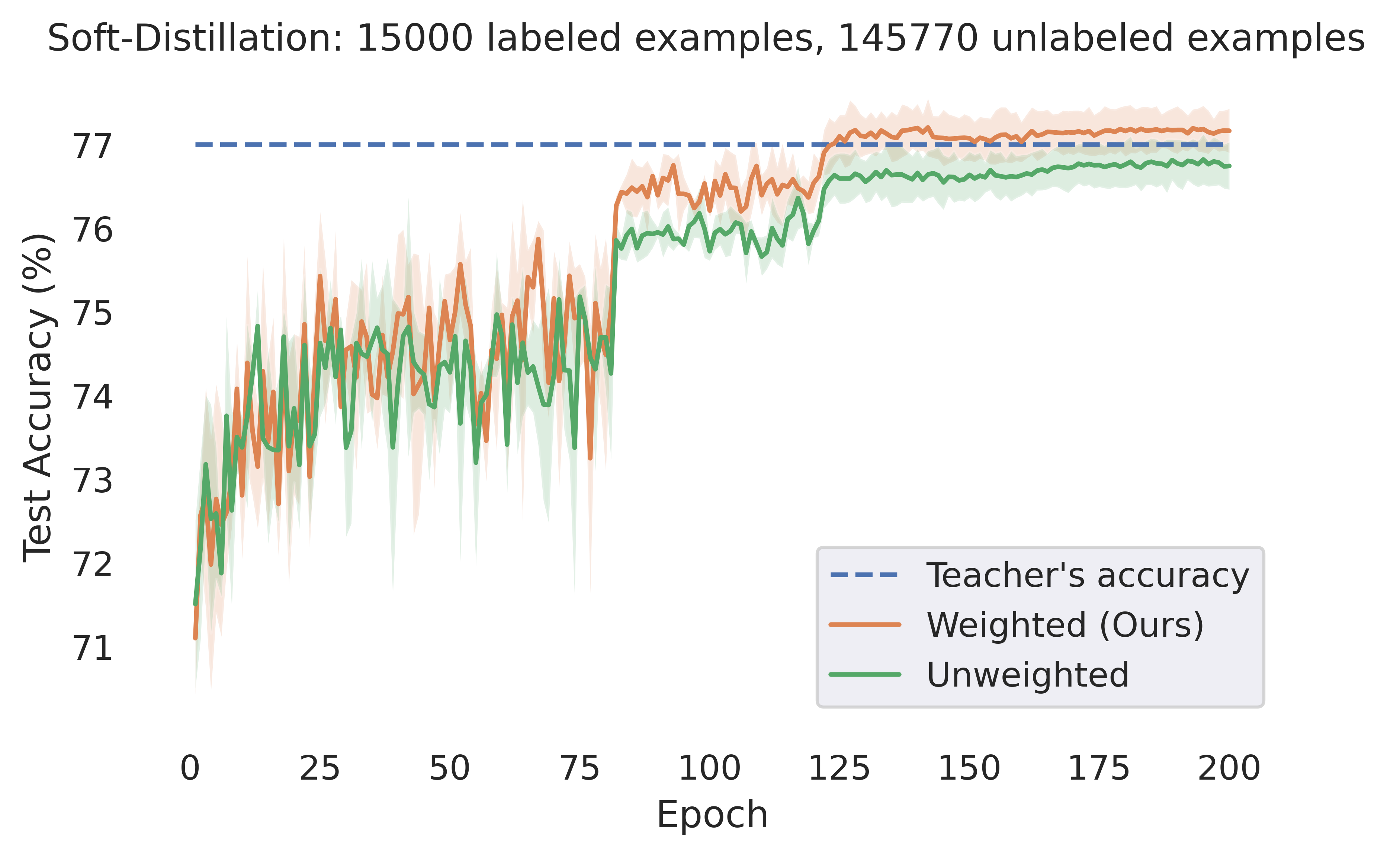} 
  \end{minipage}%
\hfill
  \centering
  \begin{minipage}[t]{0.25\textwidth}
  \centering
 \includegraphics[width=1\textwidth]{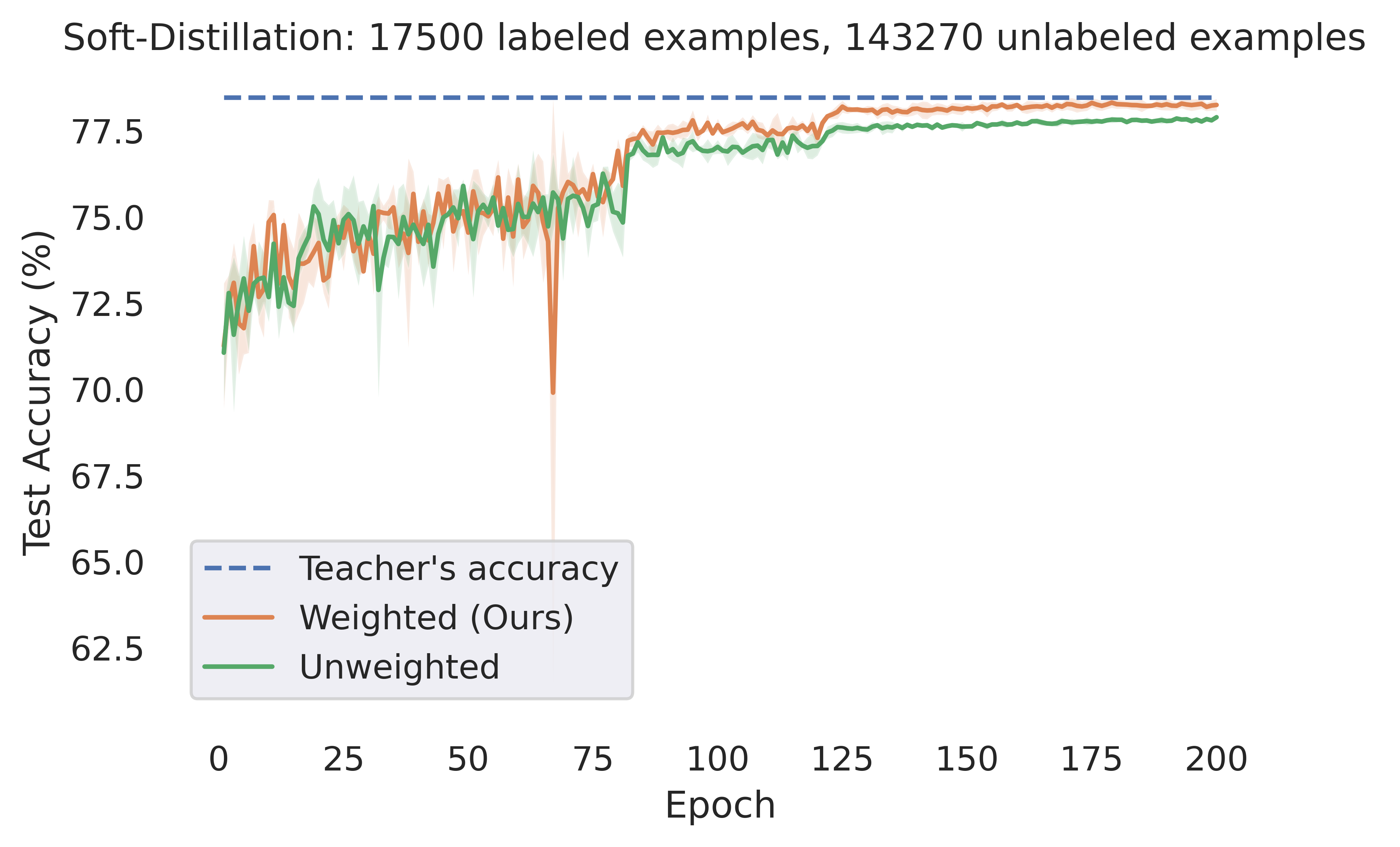} 
  \end{minipage}%

\caption{\textbf{CIFAR-10} experiments. The student's test accuracy over the training trajectory corresponding to the experiments of Figure~\ref{CIFAR10_experiments}. See Section~\ref{cifar10and100} for more details.
 }
	\label{CIFAR10ull}
\end{figure*}

\begin{figure*}[!ht]
  \centering
  \begin{minipage}[t]{0.25\textwidth}
  \centering
 \includegraphics[width=1\textwidth]{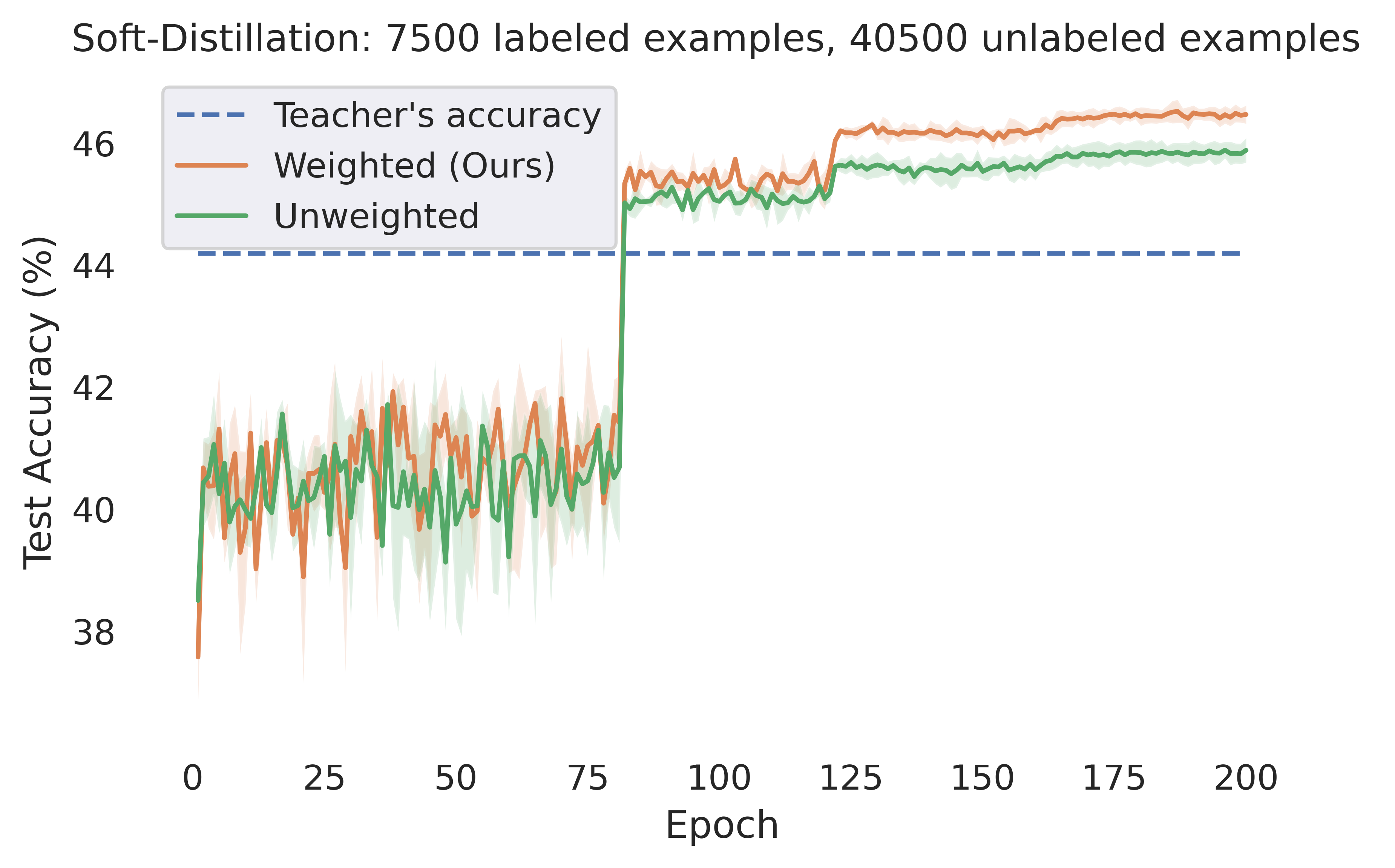} 
  \end{minipage}%
  \hfill
  \centering
  \begin{minipage}[t]{0.25\textwidth}
  \centering
 \includegraphics[width=1\textwidth]{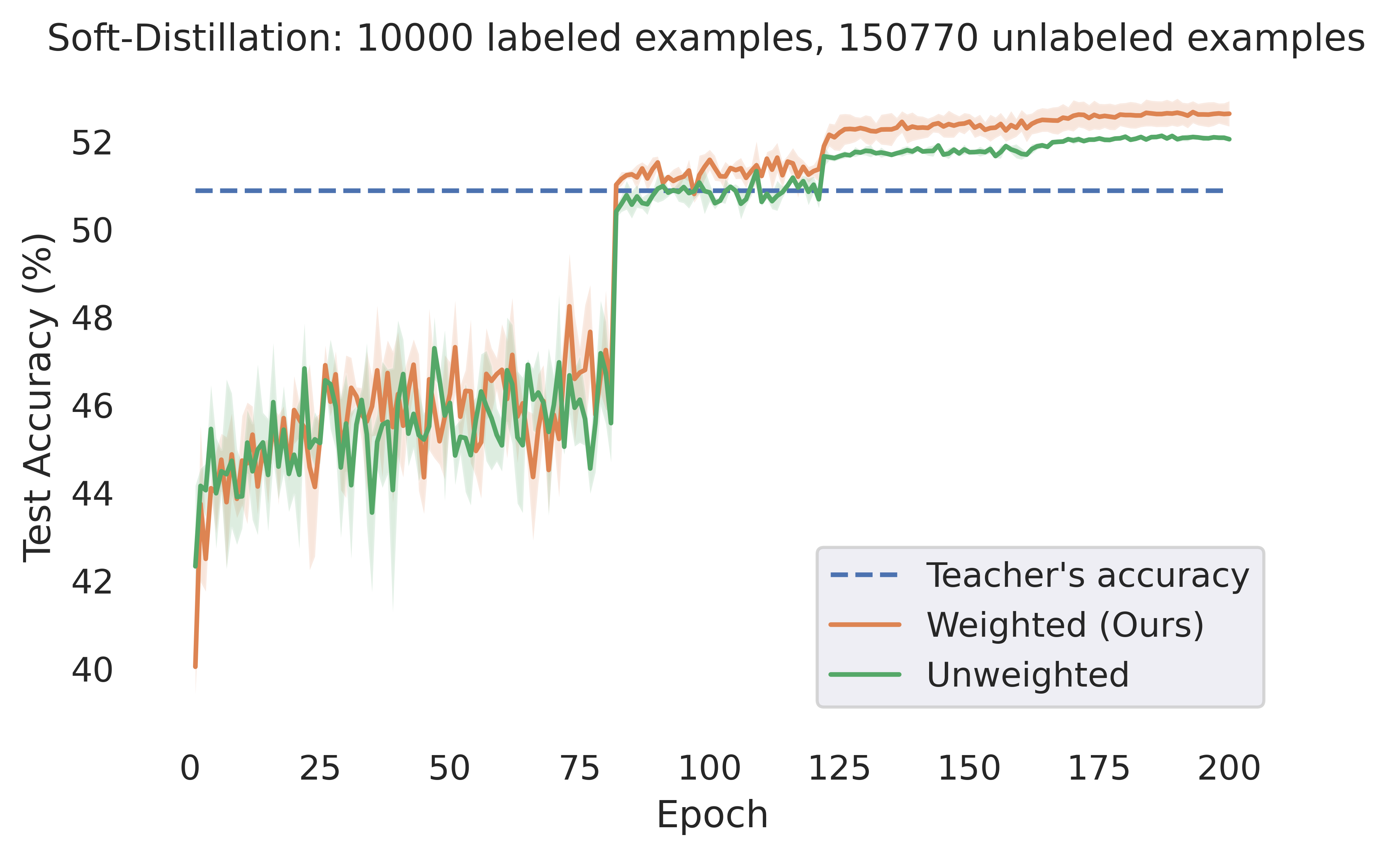} 
  \end{minipage}%
  \hfill
  \centering
  \begin{minipage}[t]{0.25\textwidth}
  \centering
 \includegraphics[width=1\textwidth]{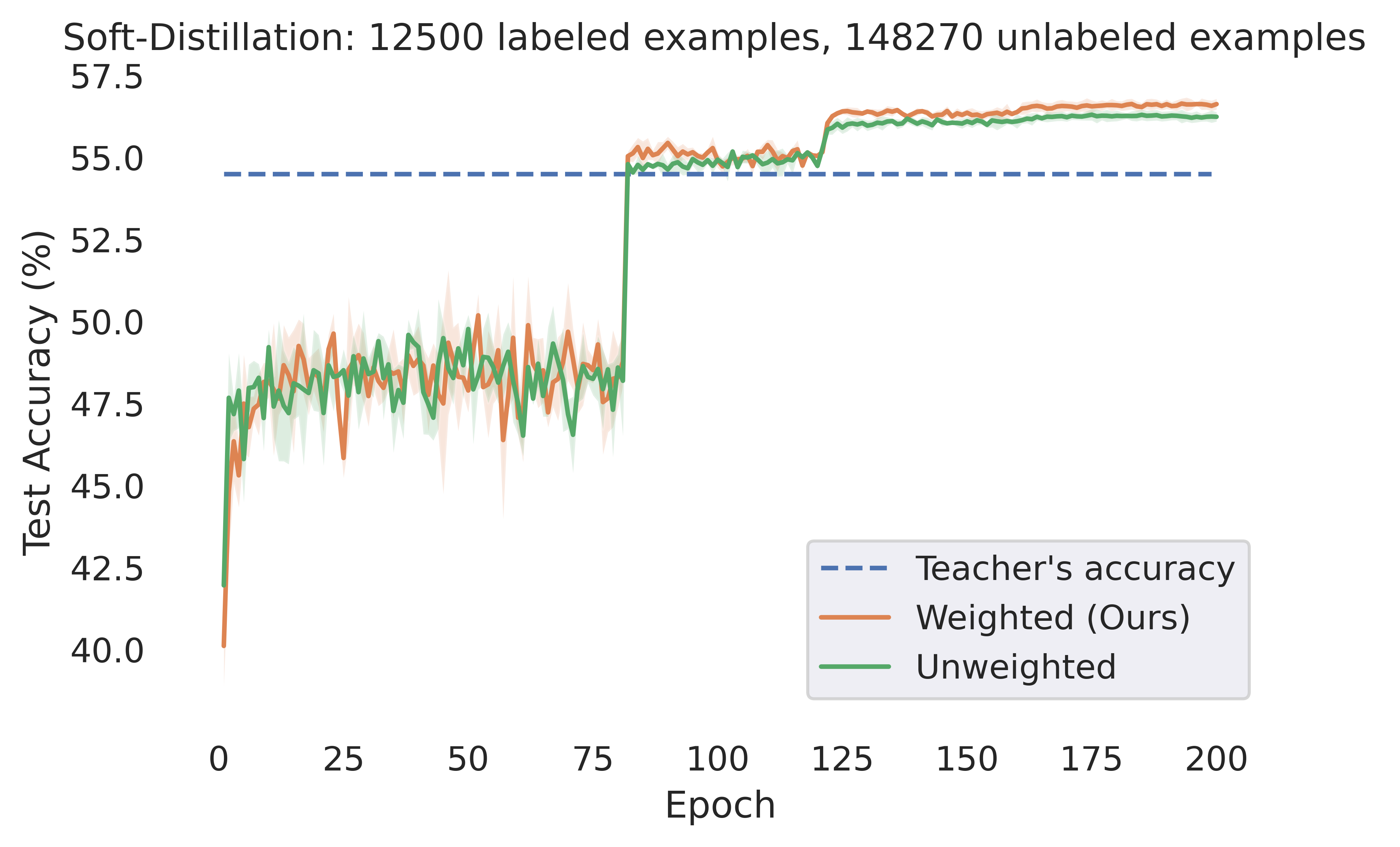} 
  \end{minipage}%
\hfill
  \centering
  \begin{minipage}[t]{0.25\textwidth}
  \centering
 \includegraphics[width=1\textwidth]{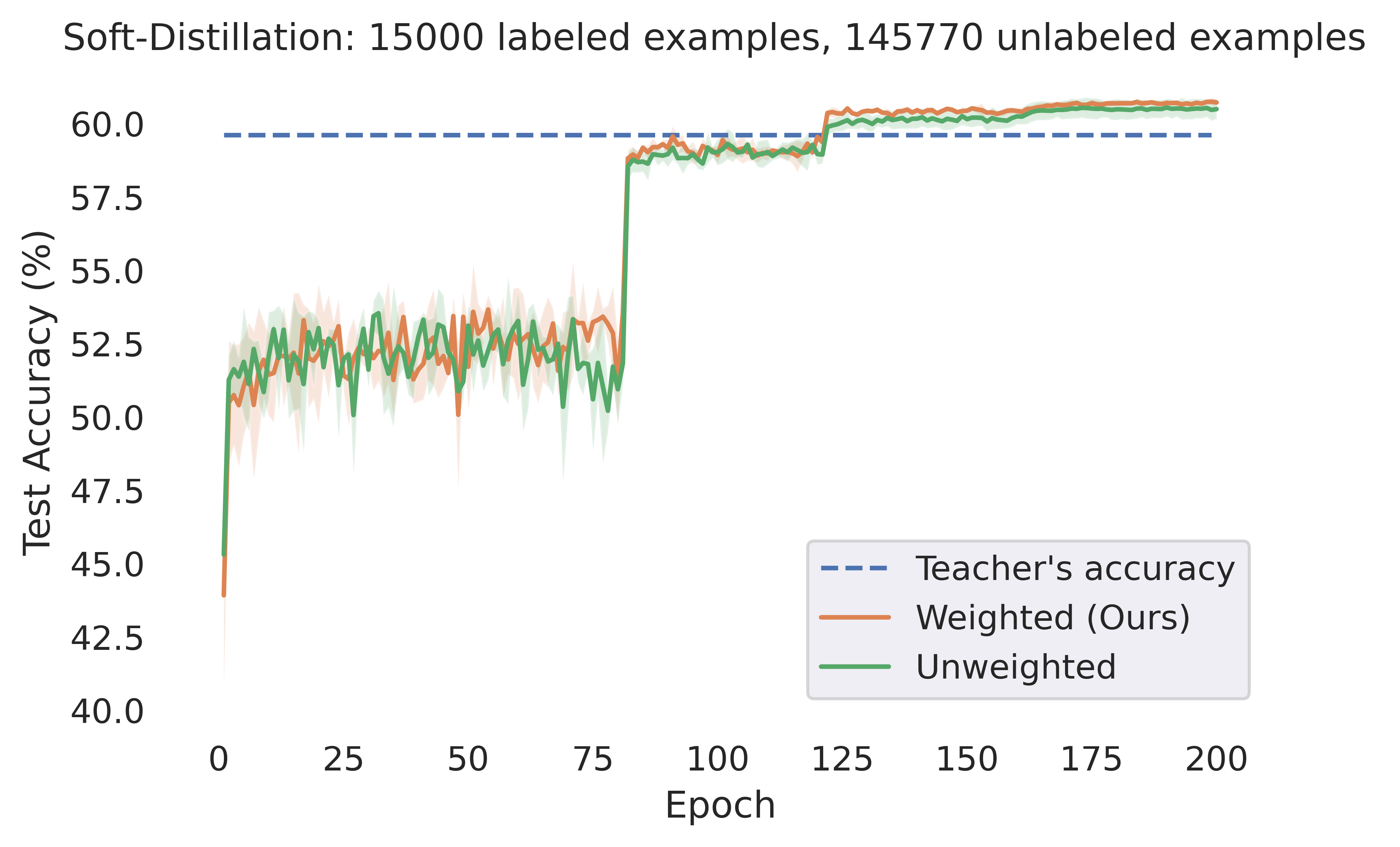} 
  \end{minipage}%

\caption{\textbf{CIFAR-100} experiments. The student's test accuracy over the training trajectory corresponding to the experiments of Figure~\ref{CIFAR100_experiments}. See Section~\ref{cifar10and100} for more details.
 }
	\label{CIFAR10ull}
\end{figure*}

\begin{figure*}[!ht]
  \centering
  \begin{minipage}[t]{0.25\textwidth}
  \centering
 \includegraphics[width=1\textwidth]{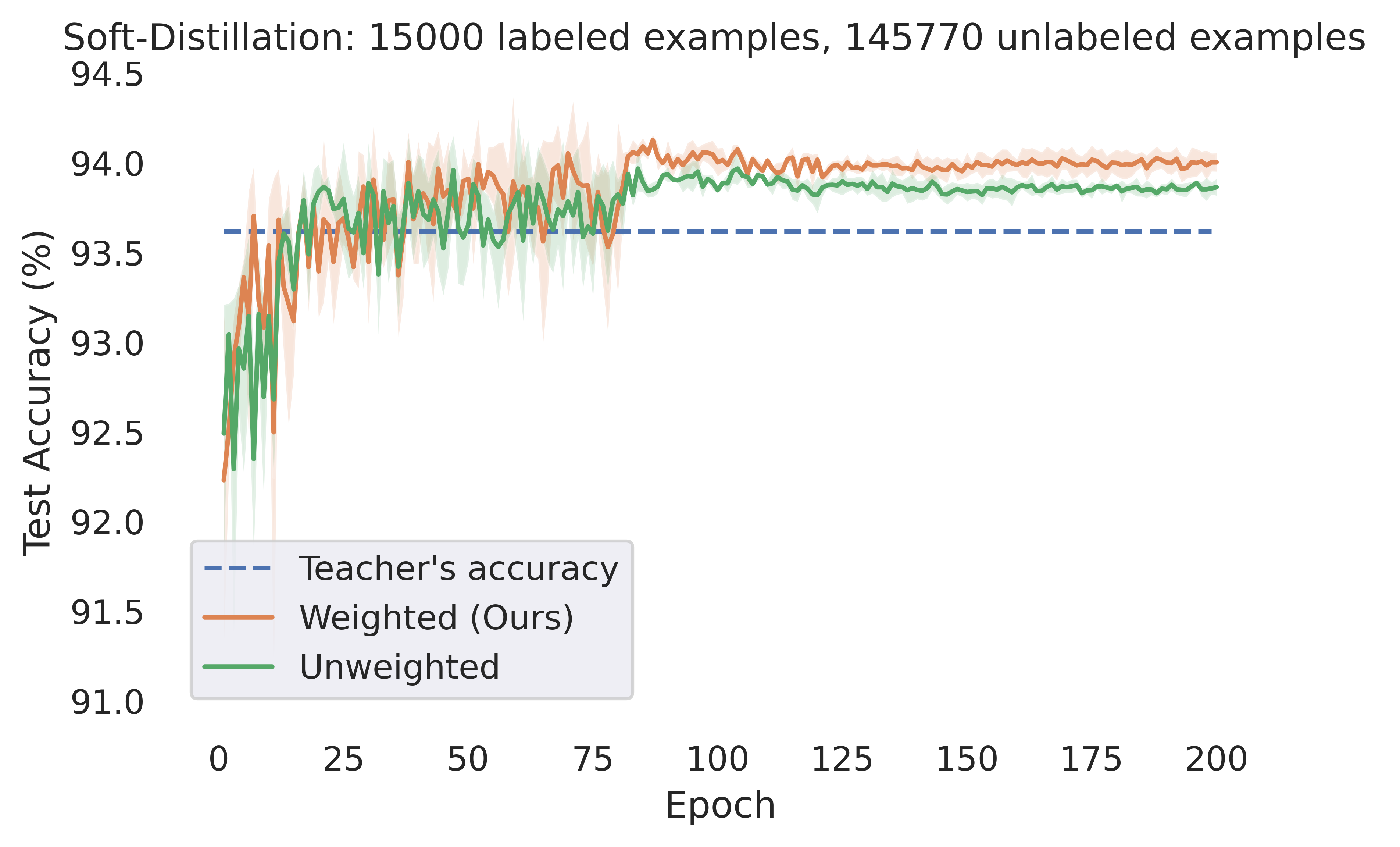} 
  \end{minipage}%
  \hfill
  \centering
  \begin{minipage}[t]{0.25\textwidth}
  \centering
 \includegraphics[width=1\textwidth]{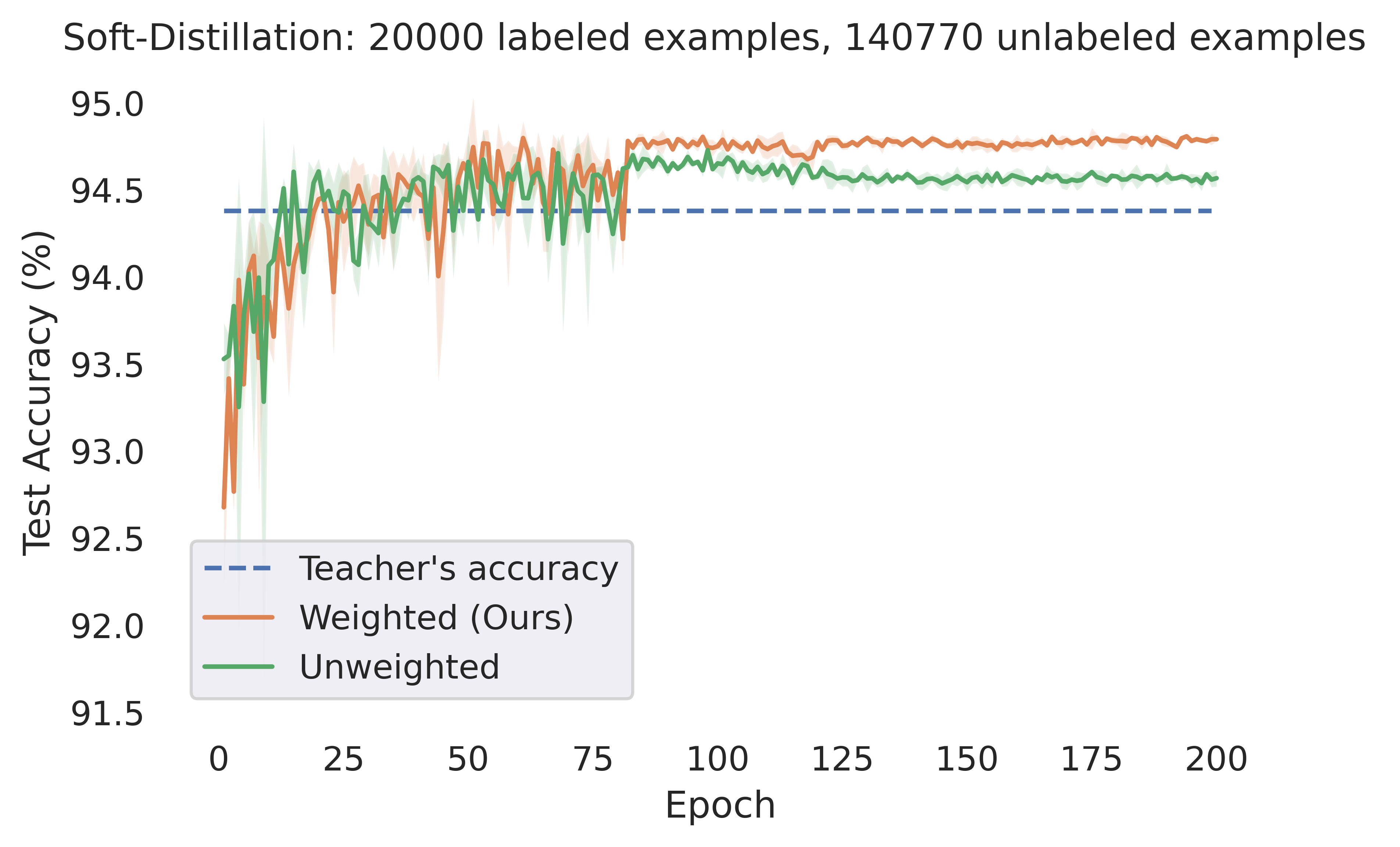} 
  \end{minipage}%
  \hfill
  \centering
  \begin{minipage}[t]{0.25\textwidth}
  \centering
 \includegraphics[width=1\textwidth]{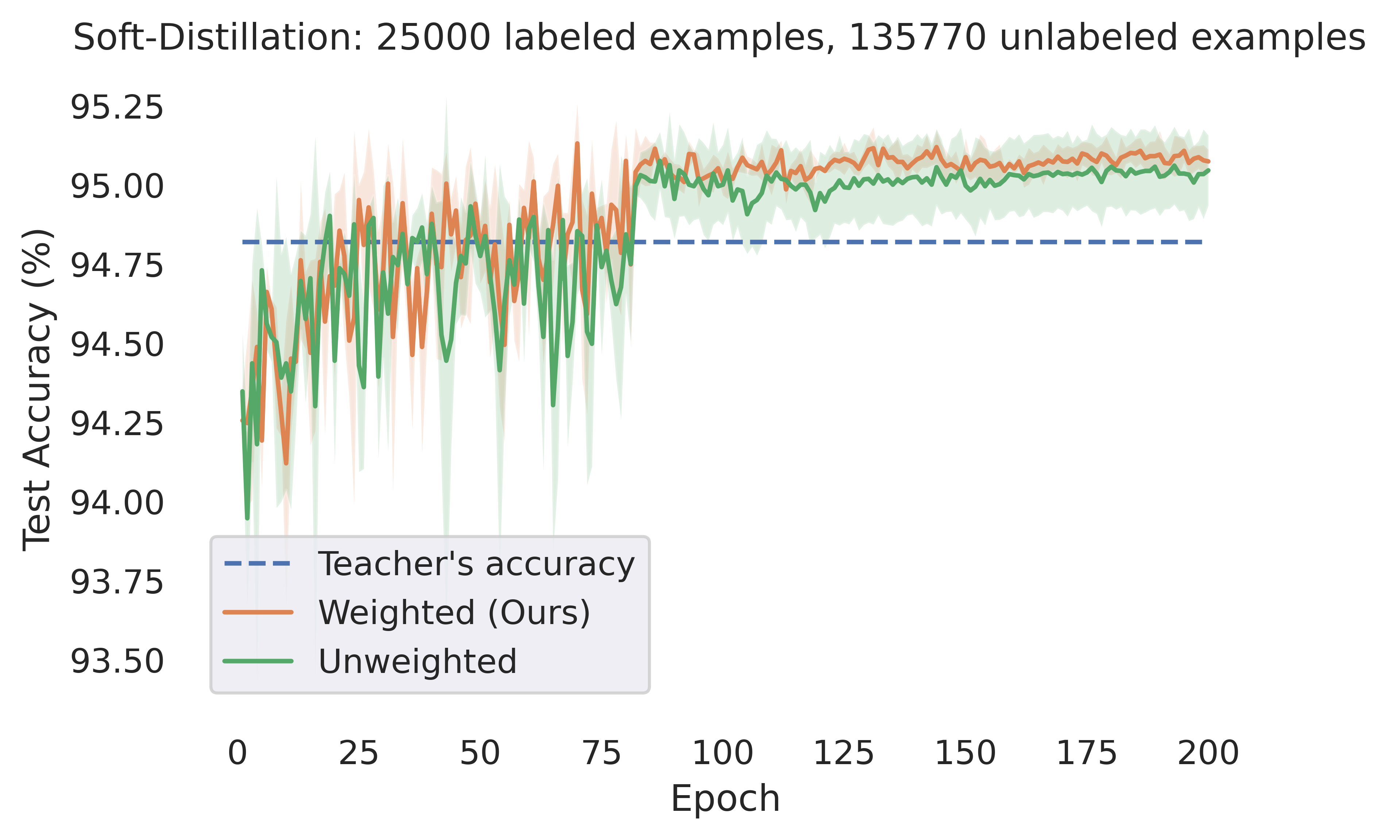} 
  \end{minipage}%
\hfill
  \centering
  \begin{minipage}[t]{0.25\textwidth}
  \centering
 \includegraphics[width=1\textwidth]{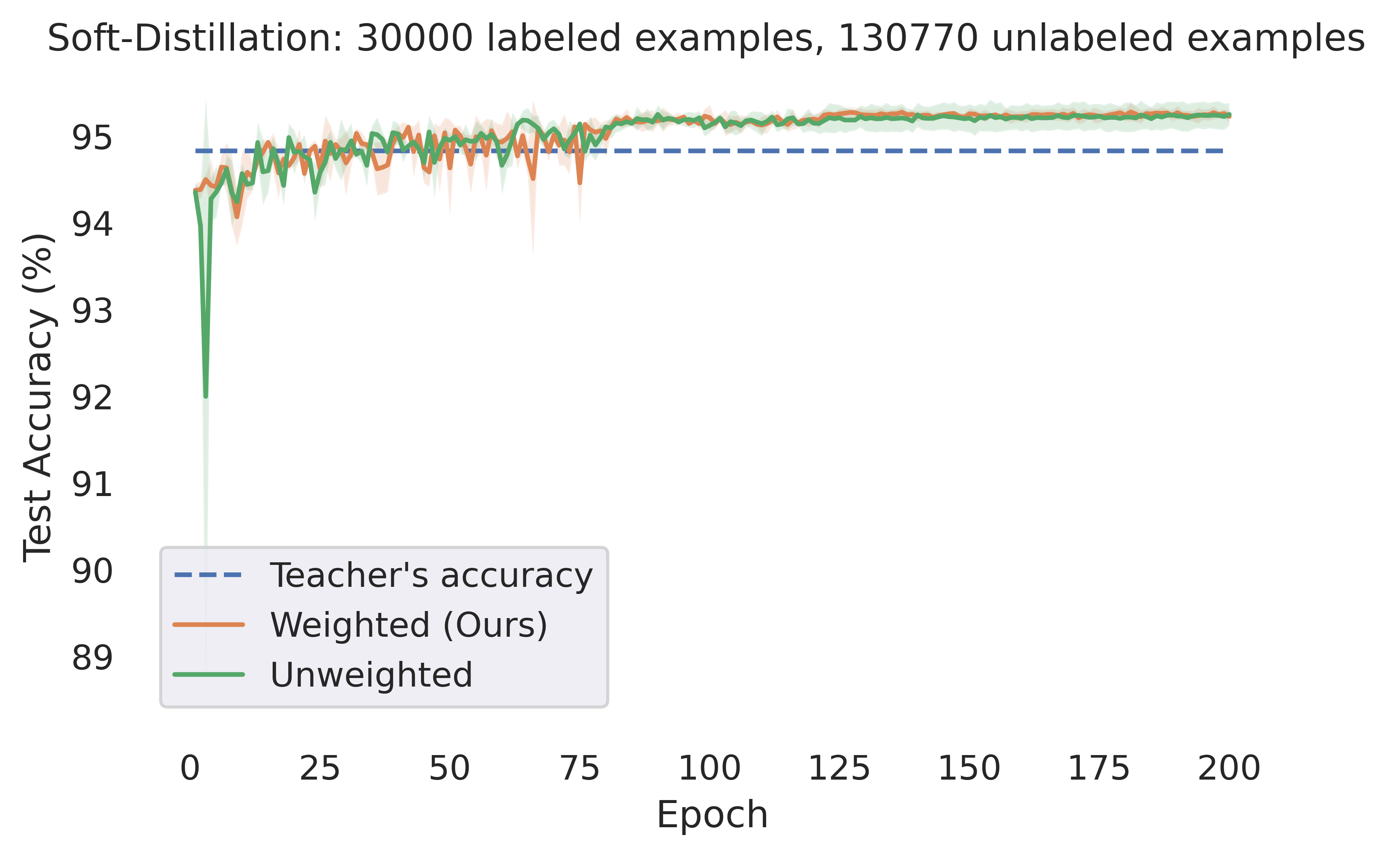} 
  \end{minipage}%

\caption{\textbf{CelebA} experiments. The student's test accuracy over the training trajectory corresponding to the experiments of Figure~\ref{celeba_experiments}. See Section~\ref{sec:celeba} for more details.
 }
	\label{SVHNFull}
\end{figure*}


\begin{figure*}[!ht]
  \centering
  \begin{minipage}[t]{0.5\textwidth}
  \centering
 \includegraphics[width=0.7\textwidth]{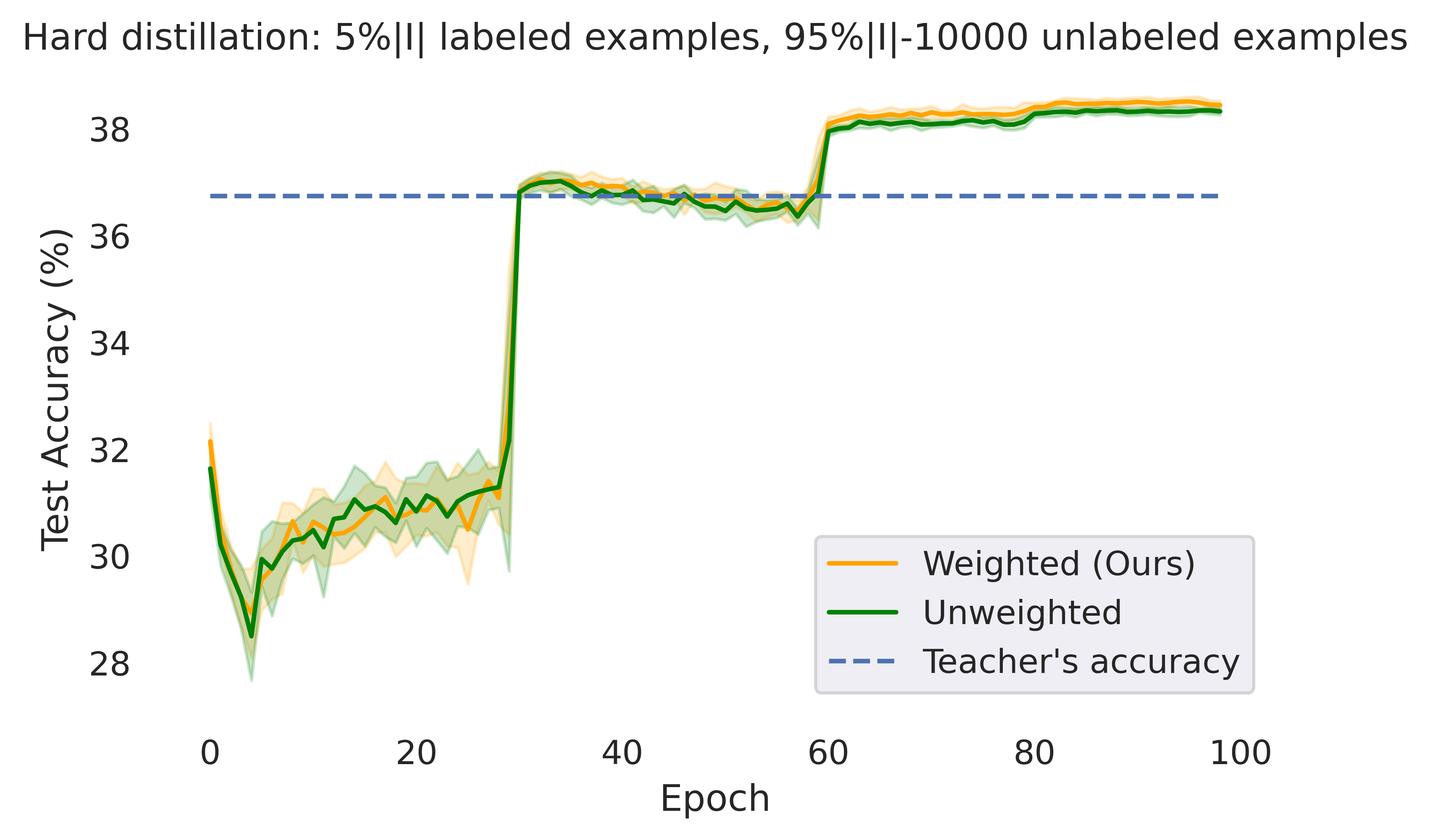} 
  \end{minipage}%
  \hfill
  \centering
  \begin{minipage}[t]{0.5\textwidth}
  \centering
 \includegraphics[width=0.7\textwidth]{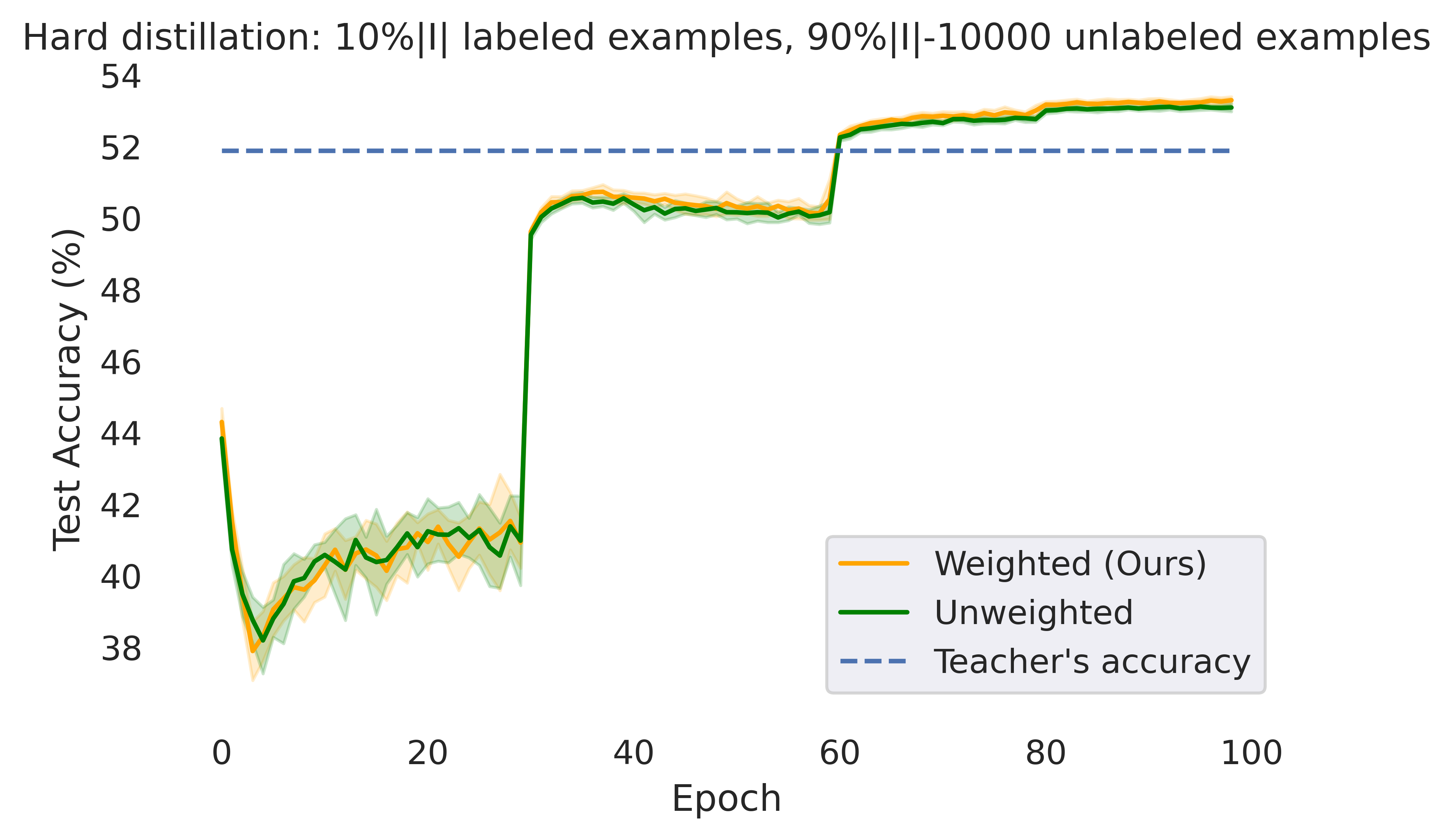} 
  \end{minipage}%
  \hfill
  
  \centering
  \begin{minipage}[t]{0.5\textwidth}
  \centering
 \includegraphics[width=0.7\textwidth]{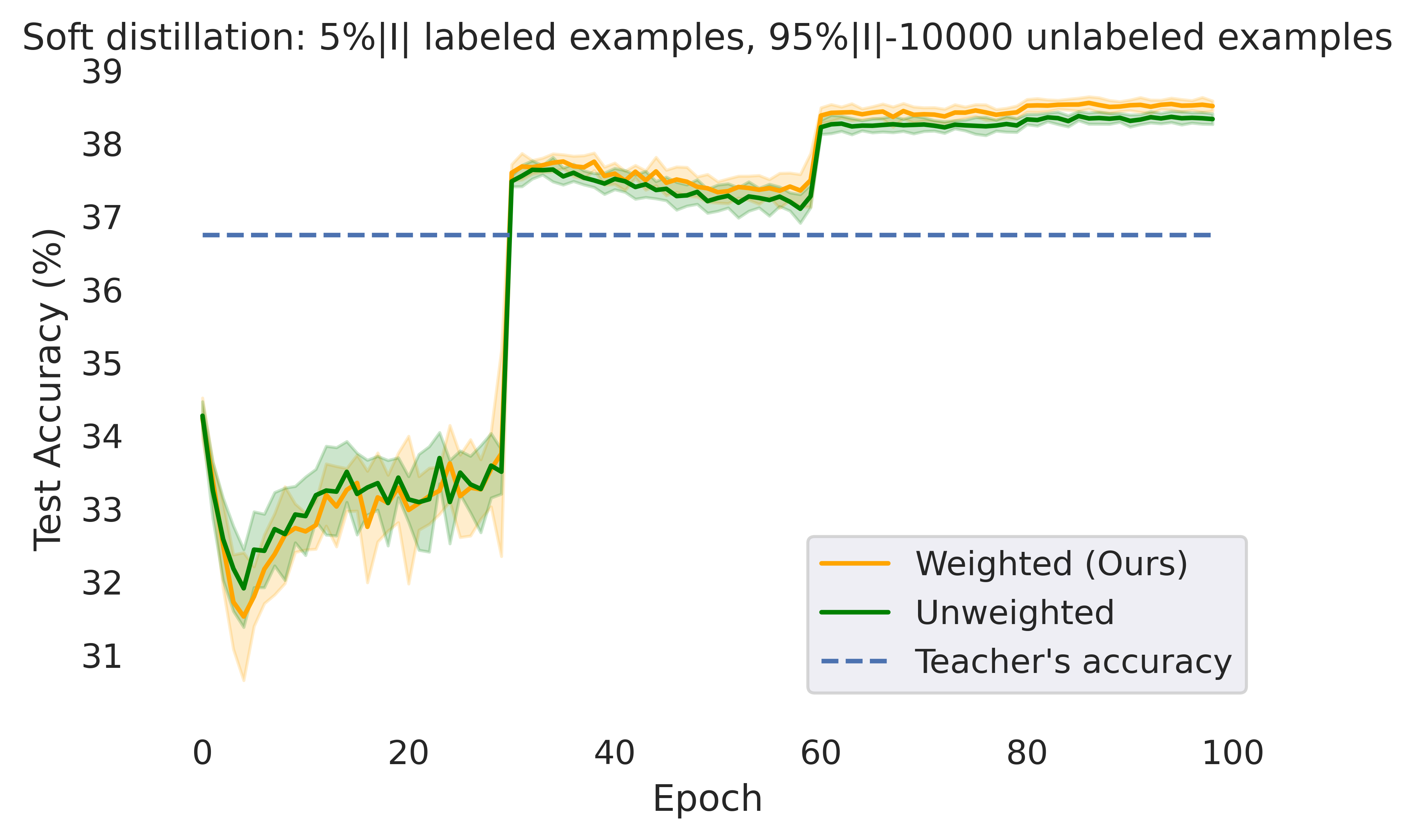} 
  \end{minipage}%
\hfill
  \centering
  \begin{minipage}[t]{0.5\textwidth}
  \centering
 \includegraphics[width=0.7\textwidth]{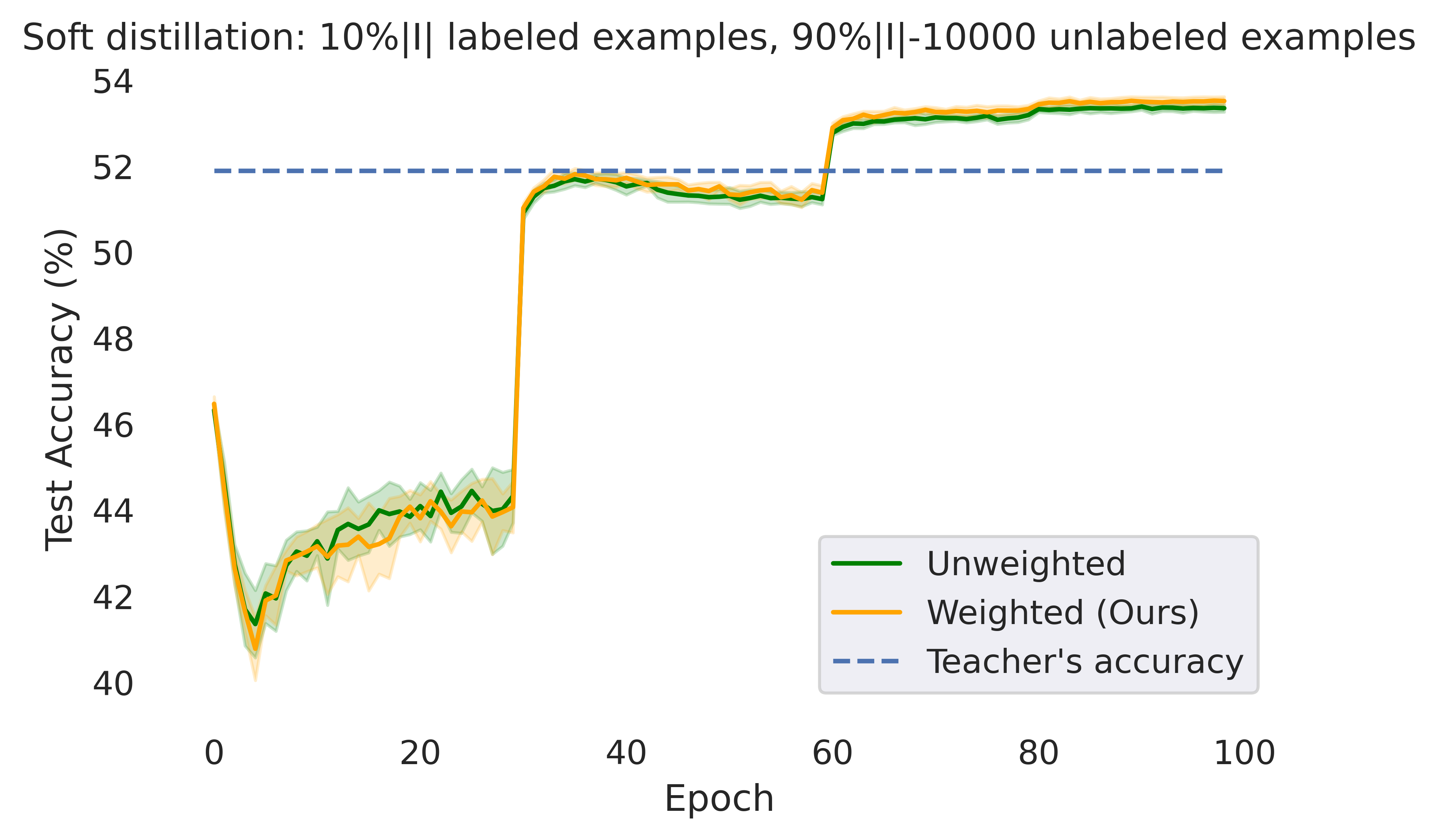} 
  \end{minipage}%

\caption{\textbf{ImageNet} experiments. The student's test accuracy over the training trajectory using hard-distillation (first row) and soft-distillation (second row) corresponding to the experiments of Figure~\ref{ImageNet_experiments}. See Section~\ref{imagenet_experiments} for more details.
 }
	\label{ImageNetFull}
\end{figure*}
 
 \blue{
\subsection{Considering the effect of temperature} 
 \label{temperature_effect}
 
 Temperature-scaling, a technique introduced in the original paper of Hinton et. al.~\cite{distillation},  is one the most common ways for improving student's performance in distillation. Indeed, it is known (see e.g.~\cite{stanton2021does}) that choosing the right value for the temperature can be quite beneficial, to the point it can outperform other more advanced techniques for improving distillation. Here we demonstrate that our approach provides benefits on top of any improvement on can get via temperature-scaling by conducting an ablation study on the effect of temperature on CIFAR-100. In our experiment, the teacher model is a Resnet-110 achieving accuracy $56.0\%$, the student model is a Resnet-56, the number of labeled examples is $12500$, the validation set consists of $500$ examples, and we use the entropy of a prediction  as a metric of confidence. We apply our method using one-shot estimation of the weights. We compare training the student model using conventional distillation to using our method for different values of temperature. The results can be found in the table below. We see that in almost all cases  the student-model trained using our method outperforms the student-model trained using conventional distillation and, in particular, the best student overall is the result of choosing $ 2.0$ for the value of temperature and applying our method.

 }

\begin{figure}[!ht] 

\begin{center}
\tiny
\begin{tabular}{||c| c| c ||} 
 \hline
Temperature  & Unweighted & Weighted (ours) \\ [0.1ex] 
 \hline
$0.01$ &	$52.84 \pm 0.08\%$ & 	$53.73 \pm 0.11\%$  \\ 
  \hline
$0.10$&	$54.63 \pm 0.09 \%$ &	$54.84  \pm 0.12 \%$  \\
 \hline
$0.50$&	$56.45 \pm 0.12 \%$ &	$57.01 \pm 0.1 \%$ \\
 \hline 
$0.80$&	$56.67 \pm 0.12 \%$ &	$57.60 \pm 0.15 \%$ \\
 \hline
$1.00$& $	57.17 \pm 0.15 \%$ & 	$57.56 \pm  0.09\%$ \\ 
 \hline
$2.00$&	$57.54 \pm  0.11 \%$&	$\mathbf{57.8 \pm 0.21 \%}$ \\
\hline
$3.00$&	$57.20 \pm 0.18\%$ &	$57.09 \pm 0.25 \%$ \\
\hline
$5.00$&	$56.92 \pm 0.11 \%$	& $57.01 \pm 0.2 \%$ \\ [1ex] 
 \hline
\end{tabular}
\end{center}

\caption{Ablation study on the effect of temperature on CIFAR-100. See Appendix~\ref{temperature_effect} for details}

\end{figure}

\section{Implementation details}
\label{implementation}

In this section we describe the implementation details of our experiments. Recall the description of our method in Section~\ref{our_method}.

\subsection{Experiments on CelebA, CIFAR-10, CIFAR-100, SVHN}
\label{dets}

 All of our experiments are performed according to the following recipe. In all cases, the loss function $\ell: \mathbb{R}^{L} \times \mathbb{R}^L \rightarrow \mathbb{R}_+$ we use is the cross-entropy loss.  We   train the teacher model for $200$ epochs on dataset $S_{\ell}$. We pretrain the student model for $200$ epochs on dataset $S_{\ell}$ and save its  parameters. Then, using the latter saved parameters for initialization each time, we train the student model for  $200$ epochs optimizing either the weighted or conventional (unweighted) empirical risk,  and report its average performance over three trials.
 
We use the Adam optimizer. The initial learning rate is $\mathrm{lr} = 0.001$. We proceed according to the following learning rate schedule (see e.g.,~\cite{he2016deep}):
\begin{align*}
\mathrm{lr} \leftarrow
\begin{cases}
\mathrm{lr} \cdot 0.5 \cdot 10^{-3}, & \text{if  $\#\mathrm{epochs} > 180$ } \\
\mathrm{lr}  \cdot 10^{-3}, & \text{if  $\#\mathrm{epochs} > 160$ } \\
\mathrm{lr}  \cdot 10^{-2}, & \text{if  $\#\mathrm{epochs} > 120$ } \\
\mathrm{lr}  \cdot 10^{-1}, & \text{if  $\#\mathrm{epochs} > 80$ }
\end{cases}
\end{align*}
Finally, we use data-augmentation. In particular, we use random horizontal flipping and random width and height translations with width and height factor, respectively, equal to $0.1$.

\subsection{Experiments on ImageNet}

For the ImageNet experiments we follow a similar although not identical recipe to the one described in Appendix~\ref{dets}. In each training stage above, we train the model (teacher or student) for $100$ epochs instead of $200$. We also use SGD with momentum $0.9$ instead of Adam as the optimizer. For data-augmentation we use only random horizontal flipping. Finally, the learning rate schedule is as follows. For the first $5$ epochs the learning rate $\mathrm{lr}$ is increased from $0.0$ to $0.1$ linearly. After that, the learning rate changes as follows:
\begin{align*}
\mathrm{lr} =
\begin{cases}
0.01, & \text{if  $\#\mathrm{epochs} > 30$ } \\
0.001, & \text{if  $\#\mathrm{epochs} > 60$ } \\
0.0001, & \text{if  $\#\mathrm{epochs} > 80$ }
\end{cases}
\end{align*}

\blue{
\subsection{Details on the experimental setup of Section~\ref{fidelity_weights_comparison} } 
\label{fidelity_details}

In the CIFAR-10 experiments of Section~\ref{fidelity_details} the teacher model is a MobileNet with depth multiplier 2, and the student model is a MobileNet with depth multiplier 1. In the CIFAR-100 experiments of the same section, the teacher model is a ResNet-110, and the student model is a ResNet-56.
We use a validation set consisting of $500$ examples (randomly chosen as always). The student of each method has access to the same number of labeled examples, i.e., the validation set is used for training the student model as we describe in Remark~\ref{remark_on_validation}. We compare the following three methods:

\begin{itemize}

\item \textbf{Fidelity weighting scheme~\cite{dehghani2017fidelity}.} For every example $x$ we use the entropy of the teacher's prediction as an uncertainty/confidence measure, which we denote by $\mathrm{entropy}(x)$. We then compute the exponential weights described in~\cite{dehghani2017fidelity} as 
$w(x) = \mathrm{exp}(- \mathrm{entropy}(x)/ \overline{\mathrm{entropy}} )  $, where $\overline{\mathrm{entropy}}$  is the average entropy of the teacher's predictions over all training examples. 

\item   \textbf{Our method.} We use the entropy as the metric for confidence. In the case of CIFAR-10 we re-estimate the weights at the end of  every epoch. In the case of CIFAR-100 the weights are estimated only once in the beginning of the process.

\item \textbf{Composition.}  We reweight each example in the loss function by multiplying the weights resulting from the two methods above.

\end{itemize}

}

\section{Extended theoretical motivation: statistical aspects}
\label{statistical_perspective}

In this section we study the statistical aspects of our approach. In Section~\ref{statistical_motivation} we revisit and formally state Theorem~\ref{stats_theorem} — see Corollary~\ref{main_corollary} and Remark~\ref{limits_remark}. In Section~\ref{MSE} we perform a Mean-Squared-Error analysis that provides additional justification of our choice to always project the weights on the $[0,1]$ interval (recall Line~\ref{projection} of Algorithm~\ref{alg:estimating_weights}). Finally, in Section~\ref{main_proposition_proof} we provide the proof of Proposition~\ref{expectations} which was omitted from the main body of the paper.

\subsection{Statistical motivation}
\label{statistical_motivation}

Recall the background on multiclass classification in Section~\ref{background}. In this section we study hypothesis classes $\mathcal{F}$ and loss functions $\ell: \mathbb{R}^L \times \mathbb{R}^L: \rightarrow \mathbb{R}$ that are ``well-behaved'' with respect to a certain (standard in the machine learning literature) complexity measure we describe below.

For $\epsilon >0$, a class $\mathcal{H}$ of functions $h: \mathcal{X} \rightarrow [0, 1]$ and an integer $n$, the ``growth function" $\mathcal{N}_{\infty}(\epsilon, \mathcal{H}, n)$  is defined as
\begin{align}
    \mathcal{N}_{\infty}(\epsilon, \mathcal{H}, n)  = \mathrm{sup}_{\mathbf{x} \in \mathcal{X}^n }
\mathcal{N} (\epsilon, \mathcal{H}(\mathbf{x} ), \| \cdot \|_{\infty}  ),
\end{align}
where $\mathcal{H}(\mathbf{x} ) = \{ (h(x_1), \ldots, h(x_n) ):  h \in \mathcal{H}    \} \subseteq \mathbb{R}^n$ and for $ A \subseteq \mathbb{R}^n$ the number $\mathcal{N} (\epsilon, A, \| \cdot \|_{\infty}  )$ is the smallest cardinality $A_0$ of a set $A_0 \subseteq A$ such that $A$ is contained in the union of $\epsilon$-balls centered at points in $A_0$, in the metric induced by $\| \cdot \|_{\infty}$
The growth number is a complexity measure of function classes commonly used in the machine learning literature~\cite{anthony1999neural, guo1999covering}. 

The following theorem from~\cite{maurer2009empirical} provides large deviation bounds for function classes of polynomial growth.

\begin{theorem}[Theorem 6, \cite{maurer2009empirical}]\label{empirical_bernstein}
Let $Z$ be a random variable taking values in $\mathcal{Z}$ distributed according to distribution $\mu$,  and let $\mathcal{H}: \mathcal{Z} \rightarrow [0,1] $ be a class of functions. Fix $\delta \in (0,1 ), n \ge 16$ and set
\begin{align*}
\mathcal{M}(n) = 10 \mathcal{N}_{\infty}(1/n, \mathcal{H}, 2n ).
\end{align*}
Then with probability at least $1-\delta$ in the random vector $Z= (Z_1, \ldots, Z_n) \sim \mu^n$, for \emph{every} $h \in \mathcal{H}$ we have:
\begin{align*}
    \left| \ex[h(Z) ] - \frac{1}{n} \sum_{i=1}^{n} h(Z_i) \right| \le \sqrt{\frac{18 \mathbb{V}_n(h, Z )\ln (2\mathcal{M}(n)  / \delta )  }{ n}  }  +  \frac{15 \ln \left( 2\mathcal{M}(n) / \delta \right) }{n-1 },
\end{align*}
where $\mathbb{V}_n(h,Z )$ is the sample variance of the sequence $\{h(Z_i )\}_{i=1}^n$.
\end{theorem}

A straightforward corollary of Theorem~\ref{empirical_bernstein} and Proposition~\ref{expectations}, and the main motivation for our method, is the following corollary.

\begin{corollary}\label{main_corollary}
Let $\ell: \mathbb{R}^L \times \mathbb{R}^L \rightarrow [0,1]$ be a loss function and fix $\delta > 0$. Consider any hypothesis class $\mathcal{F}$
of predictors $f: \mathcal{X} \rightarrow \mathbb{R}^L$, and the two induced classes $\mathcal{H} \subseteq [0,1]^{\mathbb{R}^L \times \mathbb{R}^L  }  $, $\mathcal{H}^w \subseteq [0,1]^{\mathbb{R}^L \times \mathbb{R}^L  }  $ of functions $h_f(x, y) :=  \ell(y, f(x) )  $ and $h_f^w(x, y) := w_f(x) \ell(y, f(x) )  $, respectively. 
Fix $\delta > 0$, $n \ge 16$, and set $\mathcal{M}(n) = 10 \mathcal{N}_{\infty}(1/n, \mathcal{H}, 2n )  $ and $\mathcal{M}^w(n) = 10 \mathcal{N}_{\infty}(1/n, \mathcal{H}^w, 2n )  $. Then, with probability at least $1-\delta$ over $S = \{x_i, y_i \}_{i=1}^n \sim \mathbb{D}^n$,
\begin{eqnarray}
\left| R(f) + \mathrm{Bias}(f) - R_{S}(f) \right| & = &  O \left( \sqrt{ \mathbb{V}_S(f) \cdot \frac{\ln \frac{\mathcal{M}(n) }{\delta } }{n } }  +  \frac{\ln \frac{\mathcal{M}(n) }{\delta } }{n } \right) \\
\left| R(f) - R^w_S(f) \right| & = &  O \left( \sqrt{ \mathbb{V}_S^w(f) \cdot \frac{\ln \frac{\mathcal{M}^w(n) }{\delta } }{n } }  +  \frac{\ln \frac{\mathcal{M}^w(n) }{\delta } }{n } \right)
\end{eqnarray}
where $\mathbb{V}_S(f), \mathbb{V}_S^w(f) $ are the sample variances of the loss values $\{ h_f(x_i, y_i ) \}_{i=1}^n$, $\{ h_f^w(x_i,y_i)  \}_{i=1}^n$, respectively.

\end{corollary}
The following remark formally captures Theorem~\ref{stats_theorem}.
\begin{remark}\label{limits_remark}
Under the assumptions of Corollary~\ref{main_corollary}, if we additionally have that $\mathcal{M}(n)$ and $\mathcal{M}^w(n)$ are polynomially bounded in $n$, then, 
for every $f \in \mathcal{F}$ it holds that 
\[
\lim_{|S| \to \infty} R_S^{\weight}(f) = R(f) 
~~~~
\text{and}
~~~~
\lim_{|S| \to \infty} R_S(f) = R(f) + \mathrm{Bias}(f).
\]
\end{remark}

\subsection{Studying the MSE of a fixed prediction } \label{MSE} 

In this section we study the Mean-Squared-Error (MSE) of a fixed prediction $f(x)$ for an arbitrary instance $x \in \mathcal{X}$, predictor $f$, and loss function $\ell: \mathbb{R}^L \times \mathbb{R}^L \rightarrow \mathbb{R}_{+} $, in order  to gain some understanding on when the importance weighting scheme could potentially underperform the standard unweighted approach from a bias-variance perspective (i.e.,  when the training sample is ``small enough'' so that asymptotic considerations are ill-suited). These considerations lead us to an additional justification for always projecting the weights to the $[0,1]$ interval (recall Line~\ref{projection} of Algorithm~\ref{alg:estimating_weights}). 

Formally,  we study the behavior of the quantities:
\begin{eqnarray}
\mathrm{MSE}(x) &=& \ex_{y\mid x} \left[ (\ell(f_{\mathrm{true}}(x), f(x) )  - \ell(y, f(x) )  )^2   \right], \nonumber  \\
\mathrm{MSE}^w(x) &=& \ex_{y\mid x  }[ (\ell(f_{\mathrm{true}}(x) , f(x) )  - w_f(x)\ell(y(x), f(x) )  )^2     ]. \nonumber
\end{eqnarray}

Recalling the definition of distortion~\eqref{weight_definition} we have the following proposition.
\begin{proposition}\label{when_it_fails}
Let $\ell: \mathbb{R}^L \times \mathbb{R}^L \rightarrow \mathbb{R}_+$ be a bounded loss function. Fix $x \in \mathcal{X}$ and a predictor $f: \mathcal{X}  \rightarrow \mathbb{R}^{L}$.  We have $\mathrm{MSE}(x) < \mathrm{MSE}^w (x) $ if and only if:
\begin{enumerate}
    \item $\mathrm{distortion}_f(x) < 1/2$; and
    \item $p(x) \in \left(0,  \frac{ 1- 2 \cdot \mathrm{distortion}_f(x) }{ (1 - \mathrm{distortion}_f(x)  )^2}   \right)$.
\end{enumerate}
 
\end{proposition}
\begin{proof}[Proof sketch]
Via direct calculations we obtain:
\begin{eqnarray}
\mathrm{MSE}(x) & = & \ex_{y\mid x} \left[ (\ell(f_{\mathrm{true}}(x), f(x) )  - \ell(y, f(x) )  )^2   \right] \nonumber \\
&=&  p(x) (\ell(f_{\mathrm{true}}(x), f(x) ) - \ell (y_{\mathrm{adv}}(x), f(x) ))^2  \nonumber \\
& =&p(x) \ell(f_{\mathrm{true}}(x), f(x) )^2 ( 1 - \mathrm{distortion}_f(x)    )^2                   \label{one}
\end{eqnarray}
and 
\begin{eqnarray}
\mathrm{MSE}^w(x) &=& \ex_{y\mid x  }[ (\ell(f_{\mathrm{true}}(x) , f(x) )  - w_f(x)\ell(y, f(x) )  )^2     ] \nonumber \\ 
&= &(1-p(x)) \ell(f_{\mathrm{true}}(x) , f(x) )^2 ( 1 - w_f(x))^2  \nonumber \\
 && + p(x) (\ell(f_{\mathrm{true}}(x)  , f(x)) - w_f(x) \ell( y_{\mathrm{adv}}(x), f(x))  )^2 \nonumber  \\
&= &(1-p(x)) \ell(f_{\mathrm{true}}(x) , f(x) )^2 ( 1 - w_f(x))^2  \nonumber \\
 && + p(x) \ell(f_{\mathrm{true}}(x), f(x) )^2 ( 1 - w_f(x) \mathrm{distortion}_f(x)    )^2   \label{two} 
\end{eqnarray}
Recalling the definition of weights~\eqref{weight_definition}, and combining it with~\eqref{one} and~\eqref{two} implies the claim.

\end{proof}
In words, Proposition~\ref{when_it_fails} implies that when the adversary does not have the power to corrupt the label of an instance $x$ with high enough probability, i.e., $p(x)$ is sufficiently small, and the prediction of the student is ``close enough" to the adversarial label  (i.e., when $\mathrm{distortion}_f(x) $ is small enough), then it potentially makes sense to use the unweighted estimator instead of the weighted one from a bias-variance trade-off perspective, as the former has smaller MSE in this case. Notice  that this observation  aligns well with our method as  we always project $w_f(x)$ to $[0, 1]$ (observe that $w_f(x) > 1$ iff $\mathrm{distortion}_f(x) <1$ and $p(x) > 0$ ).

\subsection{Proof of Proposition~\ref{expectations}}
\label{main_proposition_proof} 

Recall the weight, distortion and bias definitions in~\eqref{weight_definition} and Proposition~\ref{expectations}. We prove the first claim of Proposition~\ref{expectations} via direct calculations:
\begin{eqnarray*}
\ex[ R_S(f) ] &=&  \ex_{ S \sim \mathbb{D}^n } \left[ \frac{1}{n}  \sum_{i=1}^n \ell(y_i ,  f(x_i) )   \right]  \\
& =& \ex_{(x,y ) \sim \mathbb{D} } \left[  \ell(y, f(x) )  \right]  \\
& =& \ex_{x \sim \mathbb{X} } [ \ex_{y \mid x } [  \ell( y, f(x)  ) ]  ]    \\
& =& \ex_{ x\sim \mathbb{X} }[ p(x) \ell(y_{\mathrm{adv}}(x) ,f(x) )  + (1-p(x)) \ell(f_{\mathrm{true}}(x), f(x) )  ]    \\
& =& \ex_{x \sim \mathbb{X} }[ \ell( f_{\mathrm{true}}(x), f(x) )  ]  + \ex_{x \sim \mathbb{X} }[  p(x) \cdot( \ell(y_{\mathrm{adv}}(x) ,f(x) )  - \ell(f_{\mathrm{true}}, f(x) )  ) ]  \\
& =& \ex_{x \sim \mathbb{X} }[ \ell( f_{\mathrm{true}}(x), f(x) )  ]  + \ex_{x \sim \mathbb{X} }\left[  p(x) \cdot \left( \frac{ \ell(y_{\mathrm{adv}}(x) ,f(x) )}{ \ell(f_{\mathrm{true}}(x), f(x) )  }  - 1  \right) \cdot \ell(f_{\mathrm{true}}(x), f(x) )  \right] \\
& =& \ex_{x \sim \mathbb{X} }[ \ell( f_{\mathrm{true}}(x), f(x) )  ]   + \ex_{x \sim \mathbb{X} }\left[  p(x) \cdot \left( \mathrm{distortion}_f(x)   - 1  \right) \cdot \ell(f_{\mathrm{true}}(x), f(x) )  \right] \\
& =& R(f) + \mathrm{Bias}(f). 
\end{eqnarray*}
Similarly for the second claim:
\begin{eqnarray}
   \ex[  R_{S}^w(f) ] &=& \ex_{ S \sim \mathbb{D}^n } \left[ \frac{1}{n}  \sum_{i=1}^n  w_f(x_i) \ell(y_i ,  f(x_i) )   \right]    \\
    & = & \ex_{x \sim \mathbb{X} } \left[ \ex_{y \mid x }    \left[ w_f(x) \ell(y, f(x) )  \right]      \right] \nonumber \\
    &= &\ex_{ x \sim \mathbb{X} }[w_f(x) \cdot \left( p(x) \ell( y_{\mathrm{adv}}(x) , f(x) ) +
    (1-p(x) ) \ell( f_{\mathrm{true}}(x) , f(x) )      \right)        ]  \nonumber  \\
    &=& \ex_{x \sim \mathbb{X}}\left[  \frac{  p(x) \ell( y_{\mathrm{adv}}(x) , f(x) ) +
    (1-p(x) ) \ell( f_{\mathrm{true}}(x) , f(x) )         }{1 + p(x) \cdot \left( \mathrm{distortion}_f(x) -1 \right) }      \right]    \nonumber \\
        &=& \ex_{x \sim \mathbb{X}}\left[  \frac{ \ell(f_{\mathrm{true}}(x), f(x) ) + \ell(f_{\mathrm{true}}(x), f(x) ) \cdot  p(x) \cdot (\mathrm{distortion}_f(x)  -1)         }{1 + p(x) \cdot \left( \mathrm{distortion}_f(x) -1 \right) }      \right]    \nonumber  \\
 & = & R(f) \nonumber,
\end{eqnarray}
concluding the proof.

\section{Extended theoretical motivation: optimization aspects}
\label{optimization_perspective}

To prove our optimization guarantees, we analyze the reweighted objective in the fundamental
case where the model 
$f(\x;\vec \Theta)$ is linear, i.e., $f(\x; \vec \Theta) = \vec \Theta \x \in \R^L$, and the loss 
$\ell(\vec y, \vec z)$ is convex in $\vec z$ for every $\vec y$.  
In this case, the composition of the loss and the model $f(\x;\vec \Theta)$ is 
convex as a function of the parameter $\vec \Theta \in \R^{L \times d}$. 
Recall that we denote by $\ftrue(\x):\R^d \mapsto \R^L$ the ground truth classifier and by $\Dcln$ 
the ``clean'' distribution, i.e., a sample from $\Dcln$ has the form 
$(\x, \ftrue(\x))$ where $\x$ is drawn from a distribution $\Dx$
supported on (a subset of) $\R^d$.  
Finally, we denote by $\D$ the ``noisy'' labeled distribution on
$\R^d \times \R^L$ and assume that the $\x$-marginal of $\D$ is also $\Dx$.

\paragraph{Notation}
In what follows, for any elements $\vec r, \vec q$ 
of the same dimensions we denote by $\vec r \cdot \vec q$ their inner product.
For example for two vectors $\vec r, \vec q\in \R^d$ we have 
$\vec r \cdot \vec q = \sum_{i=1}^d \vec r_i \vec q_i$.
Similarly, for two matrices $\Theta, Q \in \R^{L \times d}$ we have 
$\Theta \cdot Q = \sum_{i=1}^L \sum_{j=1}^d \Theta_{i j} Q_{i j} $.
We denote by $\|\cdot\|_2$ the $\ell_2$ for vectors and the spectral
norm for matrices.  
We use $\otimes$ to denote the standard tensor (Kronecker) product between 
two vectors or matrices.  
\blue{
For example, for two matrices $A, B$ we have
\(
(A \otimes B)_{ijkl} = A_{ij} B_{kl}
\)
and for two vectors $v, u$ we have 
\(
(v \otimes u)_{ij} = v_i u_j 
\). }
We denote by $\|\cdot\|_F$ the Frobenious norm
for matrices.  We remark that we use standard asymptotic notation
$O(\cdot)$, etc. and $\wt{O}(\cdot)$ to omit factors that are 
poly-logarithmic (in the appearing arguments).

For example, training a linear
model $f(\x;\vec \Theta) = \vec \Theta \x$ with the Cross Entropy
loss corresponds to using $\ell(t, y) = \sum_{i=1}^L t_i \log( \frac{e^{y_i}}{\sum_{j=1}^L e^{y_j}} )$ and minimizing the objective 
\[
\Lc(\vec \Theta) = 
\E_{(\x, y) \sim \Dcln}[ \ell(\vec y, f(\x; \vec \Theta)) ] =
\E_{(\x, y) \sim \Dcln}[ \ell(\vec y, \vec \Theta \x)  ]  \,.
\]
More generally, in what follows we shall refer to the population loss over the clean distribution $\Dcln$ as $\Lc(\cdot)$, i.e., 
\[ 
\Lc(\vec \Theta) \triangleq 
\E_{(\x, y) \sim \Dcln}[\ell(y, f(\x;\vec \Theta))] \,.
\]
We next give a general definition of debiasing weight functions, i.e., weighting mechanisms that
make the corresponding objective function an unbiased estimator of the clean objective $\Lc(\vec
\Theta)$ for every parameter vector $\vec \Theta \in \R^d$.
\begin{definition}[Debiasing Weights]\label{def:debiasing-weights}
We say that a weight function $\weight(\x, \yadv; \vec \Theta) : 
\R^d \times \R^L \mapsto \R$ is a 
debiasing weight function if it holds that
\[
\risk^\weight(\vec \Theta) \triangleq 
\E_{(\x, \yadv) \sim \D}[\weight(\x,\yadv ; \vec \Theta) \ell(\yadv, f(\x, \vec \Theta))] 
=
\Lc(\vec \Theta)
\,.
\]
\end{definition}
\begin{remark}
We remark that the weight function $\weight(\cdot)$ depends on the 
current hypothesis,
$\vec \Theta$,  and also on the noise advice $p(\x)$ that we are given with every example.
In order to keep the notation simple, we do not explicitly track these dependencies
and simply write $\weight(\x, \yadv;\vec \Theta)$.
We also remark that, in general, in order to construct the weight function 
$\weight$ we may also use ``clean'' data, which may be available, e.g., 
as a validation dataset, as we did in Section~\ref{adversarial_setting}.
\end{remark}
 
Our main result is that, given a convex loss $\ell(\cdot)$ and a debiasing weight
function $\weight(\cdot)$ that satisfy standard regularity assumptions, stochastic gradient descent on the reweighted objective produces a parameter vector with good generalization error.  We first
present the assumptions on the example distributions, the loss, and the weight function.
In what follows, we view the gradient of a function $q(\Theta): \R^{L \times d} \mapsto \R$
as an $L\times d$-matrix and the hessian $\nabla^2 q(\Theta)$ as an $(L\times d) \times (L \times d)$-tensor (or equivalently as a $dL \times dL$-matrix).
\begin{definition}[Regularity Assumptions]
    \label{def:regularity}
        The $\x$-marginal $\Dx$ of $\D$ and $\Dcln$ is supported on (a subset of) the ball
        of radius $R>0$, $\mathcal{B}_R \triangleq \{\x \in \R^d : \|\x\|_2 \leq R\}$.  
        
        The training model is linear $f(\x; \vec \Theta) = \vec \Theta \x$ and the parameter
        space is the unit ball, i.e., $\|\vec \Theta \|_F \leq 1$.

        For every label $\yadv \in \R^L$ in the support of $\D$, 
        the loss  $\vec z \mapsto \ell(\yadv, \vec z)$ is a twice differentiable, convex function in $\vec z$. 
        Moreover $\ell(\yadv, \vec z)$ is 
        $M_\ell$-bounded, $L_\ell$-Lipschitz, and 
        $B_\ell$-smooth, i.e., 
        $|\ell(\yadv, \vec z)| \leq M_\ell$,
        $\|\nabla_{\vec z} \ell(\yadv, \vec z)\|_2 \leq L_\ell$, and 
        $\|\nabla_{\vec z}^2 \ell(\yadv, \vec z) \|_2 \leq B_\ell$,
        for all $\vec z $ with $\|\vec z\|_2 \leq R$.

        For every example $(\x, \yadv) \in \R^{d} \times \R^L$ in the 
        support of $\D$ the weight function $\vec \Theta
        \mapsto \weight(\x, \yadv; \vec \Theta)$ is twice differentiable, $M_\weight$-bounded,
        $L_\weight$-Lipschitz, and $B_\weight$-smooth, i.e., $|\weight(\x,\yadv;
        \vec \Theta )| \leq M_\weight$, $\|\nabla_{\vec \Theta} \weight(\x,\yadv; \vec \Theta) \|_F
        \leq L_\weight$, and $\|\nabla_{\vec \Theta}^2 \weight(\x,\yadv; \vec \Theta )\|_2 \leq B_\weight$ 
        for all $\Theta$ with $\|\Theta\|_F \leq 1$
        \footnote{Recall that, formally, 
        $\nabla_{\vec \Theta}^2 \weight(\x, \yadv; \vec \Theta)$
        is a $(L \times d) \times (L \times d)$-tensor $G$.
        For this tensor $G$ we overload notation and set 
        $\|G\|_2$ to be the standard $\ell_2$ operator norm when 
        we view $G$ as an $(L d) \times (L d)$-matrix.
        }.
\end{definition}
\begin{remark}
Observe that if a property in the above definition is satisfied by some parameter-value $Q$, then it is also satisfied for any other $Q' > Q$. For example,
if the loss function is $0.5$-Lipschitz it is also $1$-Lipschitz.
Therefore, to simplify the expressions, in what follows we shall assume (without loss of generality) 
that all the regularity parameters, i.e.,  $R, M_\ell, L_\ell, B_\ell, M_\weight, L_\weight, B_\weight$, are larger than $1$. 
\end{remark}

Since the loss is convex, it is straightforward to optimize the naive objective that 
does not reweight the loss and simply minimizes $\ell(\cdot)$ over the noisy examples, $\Ln(\vec \Theta)
\triangleq \E_{(\x,\yadv) \sim \D}[\ell(\yadv, \vec \Theta \x)]$.  We first show that
(unsurprisingly) it is not hard to construct instances (even in binary 
classification) where optimizing the naive objective produces
classifiers with large generalization error over clean examples.
For simplicity, since in the following lemma we consider binary classification, 
we assume that the labels $y \in \{\pm 1\}$ and the parameter of the linear model 
is a vector $\vec \theta \in \R^d$.
\begin{proposition}[Naive Objective Fails]
    \label{pro:naive-fails}
    Fix any $c \in [0,1]$.  Let $\ell(\cdot)$ be the Binary Cross Entropy loss, i.e.,
    $\ell(t) = \log(1+e^{-t})$.  There exists a ``clean'' distribution $\Dcln$ and 
    a noisy distribution $\D$ on $\R^d \times \{ \pm 1\}$ so that the following hold.
    \begin{enumerate}
    \item  The $\x$-marginal of both $\Dcln$ and $\D$ is uniform on a sphere.
    \item The clean labels of $\Dcln$ are consistent with 
    a linear classifier $\sgn(\vec \theta^\ast \cdot \x)$.
    \item $\D$ has (total) label noise 
    $\Pr_{(\x,\yadv) \sim \D}[\yadv\neq \sgn(\theta^\ast\cdot \x)] = c \in [0,1] $.
    \item  The minimizer $\widehat{\vec \theta}$ of the (population) 
    naive objective $\Ln(\vec \theta) = \E_{(\x,y) \sim \D}
    [\ell(\yadv\vec \theta \cdot \x)]$,
    constrained on the unit has generalization error
    \[
    \Lc(\widehat{\vec \theta}) \geq 
    \min_{\|\vec \theta\|_2 \leq 1} \Lc(\vec \theta)  + c/2 \,,
    \]
    \end{enumerate}
    where $\Lc(\vec \theta)$ is the ``clean'' risk, 
    $\Lc(\vec \theta) = \E_{(\x, y) \sim \Dcln}[\ell(y \vec \theta \cdot \x ) ]$.
    \end{proposition}

Our positive results show that, having a debiasing weight function $\weight(\cdot)$ that is not very
``wild'' (see the regularity assumptions of Definition~\ref{def:regularity}) and optimizing the
corresponding weighted objective $\risk^\weight(\vec \Theta) = \E_{(\x,\yadv) \sim \D}[\weight(\x,\yadv;
\vec \Theta) \ell(\yadv, \Theta \x)]$ with SGD, gives models with almost optimal generalization.
The main issue with optimizing the reweighted objective is that, in general, we have no guarantees
that the weight function preserves its convexity (recall that it depends on the parameter $\vec
\Theta$).  However, we know that its population version corresponds to the clean objective
$\Lc(\cdot)$ which is a convex objective.  We show that we can use the convexity of the underlying
clean objective to show results for both single- and multi-pass stochastic gradient descent.  
We first focus on single-pass stochastic gradient descent where at every iteration a fresh noisy sample
$(\x, \yadv)$ is drawn from $\D$, see Algorithm~\ref{alg:one-pass-sgd}.
\begin{algorithm}
    {\bf Input:} Number of iterations $T$, Step size sequence $\eta\tth$ \\
    {\bf Output:} Parameter vector $\vec \Theta^{(T)}$.
    \begin{enumerate}[label={}]
    \item Initialize $\vec \Theta^{(1)} \gets \vec 0$.
    \item For $t = 1,\ldots, T$: 
    \begin{enumerate}[label={}]
    \item Draw sample $(\x^{(t)}, \yadv\tth) \sim \D$.
    \item 
    Update using the gradient of the weighted objective:
    \[\vec \Theta^{(t+1)} \gets 
    \proj_{\mathcal{B}}\left(
    \vec \Theta \tth
    - \eta\tth \nabla_{\vec \Theta}
    \left(\weight(\vec x\tth, \yadv\tth; \vec \Theta\tth)
    ~
    \ell(\yadv\tth, \vec \Theta\tth \x \tth)
    \right)
    \right)
    \]
    \end{enumerate}
    \item Return $\vec \Theta^{(T)}$.
    \end{enumerate}
\caption{Single-Pass Stochastic Gradient Descent Algorithm}
\label{alg:one-pass-sgd}
\end{algorithm}
\begin{theorem}[Generalization of Reweighted Single-Pass SGD]
    \label{thm:single-pass}
Assume that the example distributions $\Dcln, \D$, the loss $\ell(\cdot)$, and the weight function
$\weight(\cdot)$ satisfy the assumptions of Definition~\ref{def:regularity}.  Set $\kappa =
L_\weight M_\ell + R M_\weight L_\ell$.  After $T = \Omega(\kappa^2/\eps^2 )$ SGD iterations
(with a step size sequence that depends on the regularity parameters of
Definition~\ref{def:regularity}, see Algorithm~\ref{alg:one-pass-sgd}), with probability at least
$99\%$, it holds
\[ 
    \Lc(\vec \Theta^{(T)}) \leq  \min_{\|\vec \Theta\|_F \leq 1} \Lc(\vec \Theta)  + \eps \,.
\]
\end{theorem}
The main observation in the single-pass setting is that, since the weight function $\weight(\cdot)$
is debiasing, we can view the gradients of the reweighted objective as stochastic gradients of the
true objective over the clean samples.  Therefore, as long as we draw a fresh i.i.d.\ noisy sample
$(\x,y) \sim \D$ at each round, the corresponding sequence of gradients corresponds to stochastic
unbiased estimates of the gradients of the true loss $\Lc(\vec \Theta)$.  We next turn our attention
to multi-pass SGD (see Algorithm~\ref{alg:multi-sgd}), where at each round we pick one of the $N$
samples with replacement and update according to its gradient.  The key difference between single-
and multi-pass SGD is that the expected loss over the stochastic algorithm for single-pass SGD is
the population risk, while the expected loss for multi-pass SGD is the empirical risk.  In other words,
in the multi-pass setting we have a stochastic gradient oracle to the empirical reweighted objective
$
    \wh{\risk}^\weight(\vec \Theta) = 
    \frac{1}{N} \sum_{i=1}^N \weight(\x^{(i)}, \yadv^{(i)}; \vec \Theta) ~ 
    \ell(\yadv^{(i)}, \vec \Theta \x^{(i)})  \,,
$
which is not necessarily convex.  Our second theorem shows that under the regularity conditions of
Definition~\ref{def:regularity} multi-pass SGD also achieves good generalization error.
\begin{theorem}[Generalization of Reweighted Multi-Pass SGD]
    \label{thm:multi-pass-gen}
Assume that the example distributions $\Dcln, \D$, the loss $\ell(\cdot)$, and the weight function
$\weight(\cdot)$ satisfy the assumptions of Definition~\ref{def:regularity}.  Set $\kappa =  R
M_\ell L_\ell B_\ell M_\weight L_\weight B_\weight$ and define the empirical reweighted objective
with $N =  (dL)^2/\eps^2 ~ \poly(\kappa)$ i.i.d.\ samples $(\x^{(1)}, \yadv^{(1)}), \ldots, (\x^{(N)},
\yadv^{(N)})$ from the noisy distribution $\D$ as
\[ 
    \wh{\risk}^\weight(\vec \Theta) = 
    \frac{1}{N} \sum_{i=1}^N \weight(\x^{(i)}, \yadv^{(i)}; \vec \Theta) ~ 
    \ell(\yadv^{(i)}, \vec \Theta \x^{(i)})  \,.
\]
Then, after $T = \poly(\kappa)/\eps^4$ iterations, multi-pass SGD with constant step size sequence
$\eta\tth = C/\sqrt{T}$ \footnote{$C$ is a constant that depends on the regularity parameters of 
Definition~\ref{def:regularity}. } (see Algorithm~\ref{alg:multi-sgd}) on $\wh{\risk}^\weight(\cdot)$
outputs a list $\vec \Theta^{(1)}, \ldots, \vec \Theta^{(T)}$ that, with probability at least
$99\%$, contains a vector $\wh{\vec \Theta}$ that satisfies
\[ 
    \Lc( \widehat{\vec \Theta}) \leq 
    \min_{\|\vec \Theta\|_F \leq 1} \Lc(\vec \Theta)  + \eps \,.
\]
\end{theorem}
We remark that our analysis also applies to the multi-pass SGD variant where, 
at every epoch we pick a random permutation of the $N$ samples and update with their 
gradients sequentially.
\begin{algorithm}
    {\bf Input:} Number of Rounds $T$, Number of Samples $N$, 
    Step size sequence $\eta\tth$.\\
    {\bf Output:} List of weight vectors $\vec \Theta^{(1)}, \ldots, \vec \Theta^{(T)}$.
    \begin{enumerate}[label={}]
    \item Draw $N$ i.i.d.\ samples $(\x^{(1)}, \yadv^{(1)}),
    \ldots, (\x^{(N)}, \yadv^{(N)})  \sim \D$.
    \item Initialize $\vec \Theta^{(1)} \gets \vec 0$.
    \item For $t = 1,\ldots, T$: 
    \begin{enumerate}[label={}]
    \item Pick $I$ uniformly at random from $\{1,\ldots, N\}$
    and update using the gradient of the reweighted objective:
    \[
    \vec \Theta^{(t+1)} \gets \proj_{\mathcal{B}}  \left(\vec \Theta \tth
    - \eta\tth \nabla_{\vec \Theta}
    \left(\weight(\vec x^{(I)}, \yadv^{(I)}; \vec \Theta\tth)
    \ell(\yadv^{(I)}, \vec \Theta\tth \x^{(I)})
    \right)
    \right)
    \,.
    \]
    \end{enumerate}
    \item Return $\vec \Theta^{(1)}, \ldots, \vec \Theta^{(T)}$.
    \end{enumerate}
\caption{Multi-Pass Stochastic Gradient Descent Algorithm}
\label{alg:multi-sgd}
\end{algorithm}


\subsection{The proof of Proposition~\ref{pro:naive-fails}}

In this subsection we restate and prove Proposition~\ref{pro:naive-fails}.
\begin{proposition}[Naive Objective Fails (Restate of \ref{pro:naive-fails}) ]
    Fix any $c \in [0,1]$.  Let $\ell(\cdot)$ be the Binary Cross Entropy loss, i.e.,
    $\ell(t) = \log(1+e^{-t})$.  There exists a ``clean'' distribution $\Dcln$ and 
    a noisy distribution $\D$ on $\R^d \times \{ \pm 1\}$ so that the following hold.
    \begin{enumerate}
    \item  The $\x$-marginal of both $\Dcln$ and $\D$ is uniform on a sphere.
    \item The clean labels of $\Dcln$ are consistent with 
    a linear classifier $\sgn(\vec \theta^\ast \cdot \x)$.
    \item $\D$ has (total) label noise 
    $\Pr_{(\x,\yadv) \sim \D}[\yadv\neq \sgn(\theta^\ast\cdot \x)] = c \in [0,1] $.
    \item  The minimizer $\widehat{\vec \theta}$ of the (population) 
    naive objective $\Ln(\vec \theta) = \E_{(\x,y) \sim \D}
    [\ell(\yadv\vec \theta \cdot \x)]$,
    constrained on the unit has generalization error
    \[
    \Lc(\widehat{\vec \theta}) \geq 
    \min_{\|\vec \theta\|_2 \leq 1} \Lc(\vec \theta)  + c/2 \,,
    \]
    \end{enumerate}
    where $\Lc(\vec \theta)$ is the ``clean'' risk, 
    $\Lc(\vec \theta) = \E_{(\x, z) \sim \Dcln}[\ell(z \vec \theta \cdot \x ) ]$.
\end{proposition}
\begin{proof}
We set the $\Dx$-marginal to be the uniform distribution on a sphere of radius $R>0$
to be specified later in the proof.
We first observe that the unit vector $\theta^\ast$
minimizes the (clean) Binary Cross Entropy $\Lc(\theta)$.
We can now pick a different
parameter vector $\wt{\vec \theta}$ with angle $\phi\in [0,\pi]$ with $\vec \theta^\ast$, and
construct a noisy instance as follows: we first draw $\vec x \sim \Dx$ 
(recall that we want the $\x$-marginal of the noisy distribution to be the same
as the ``clean'') and then set $\yadv= \sgn(\wt{\vec \theta} \cdot \x)$.  By the symmetry of the
uniform distribution on the sphere we have that 
\[
\pr_{(\x, y) \sim \D}[\yadv\neq \sgn(\vec \theta^\ast \cdot \x)]
= 
\pr_{\x \sim \D_x} [ \sgn(\wt{\vec \theta} \cdot \x) \neq \sgn(\vec \theta^\ast \cdot \x)]
= \frac{\phi}{\pi}
\,.
\]
Therefore, by picking the angle $\phi$ to be equal to 
$\pi c$ we obtain that 
$\pr_{\x \sim \D_x}[\sgn(\wt{\vec \theta} \cdot \x) \neq \sgn(\vec \theta^\ast \cdot \x)]
= c $ as required by Proposition~\ref{pro:naive-fails}.
Moroever, we have that the minimizer of the ``naive'' BCE objective (constrained on the unit ball)
is $\wt{\vec \theta}$ and the minimizer of the clean objective is $\vec \theta^\ast$.  We have that 
\begin{align*}
\Lc(\wt{\vec \theta})
&=
\E_{\x \sim \Dx} [\log( 1 + e^{- \sgn(\vec \theta^\ast \cdot \x) \wt{\vec \theta} \cdot \x})]
\geq 
\E_{\x \sim \Dx}[\1\{ \sgn(\vec \theta^\ast \cdot \x) \wt{\vec \theta} \cdot \x < 0\}]
\\
&=
\pr_{\x \sim \Dx}[\sgn(\wt{\vec \theta} \cdot \x) \neq 
\sgn(\vec \theta^\ast \cdot \x)] 
= c \,.
\end{align*}
Moreover, we have that 
\begin{align*}
\Lc(\vec \theta^\ast)
&=
\E_{\x \sim \Dx} [\log( 1 + e^{- \sgn(\vec \theta^\ast \cdot \x) \vec \theta^\ast \cdot \x}) ]
\\
&= \E_{\x \sim \Dx} [\log( 1 + e^{- |\vec \theta^\ast \cdot \x|} ) ]
\\
\end{align*}
We next need to bound from below the ``margin'' of the optimal weight vector $\vec \theta^\ast$,
i.e., provide a lower bound on $|\vec \theta^\ast \cdot \x|$ that holds with high probability.  We
will use the following anti-concentration inequality on the probability of a origin-centered slice
under the uniform distribution on the sphere. For a proof see, e.g., Lemma 4 in \cite{DasguptaKalai05}.
\begin{lemma}
    [Anti-Concentration of Uniform vectors, \cite{DasguptaKalai05}]
\label{lem:anti-concentration}
Let $\vec v \in \R^d$ be any unit vector and let $\Dx$ be the uniform distribution on the sphere. It
holds that 
\[
\pr_{\x \sim \Dx} \left[|\vec v \cdot \x| \leq \frac{\gamma}{\sqrt{d}} \right] \leq \gamma
\,.
\]
\end{lemma}
Using Lemma~\ref{lem:anti-concentration}, we obtain that 
\begin{align*}
\E_{\x \sim \Dx}[\log( 1 + e^{- |\vec \theta^\ast \cdot \x|} ) ]
\leq \log(2) \gamma + \log(1 + e^{-\gamma R/\sqrt{d}}) (1-\gamma)
\leq 2\gamma + e^{-\gamma R/\sqrt{d}} 
\,,
\end{align*}
where, at the last step, we used the elementary inequality $\log(1+x) \leq x$.
Assuming that $R/\sqrt{d}$ is much larger than $1$, we can pick  $\gamma = (\sqrt{d}/R)
\log(R/\sqrt{d})$.  For this choice of $\gamma$ we obtain that $\E_{\x \sim \Dx}[\log( 1 + e^{-
|\vec \theta^\ast \cdot \x|})] = O(\sqrt{d}/R \log(d/R))$.  Therefore, for $R = O(\sqrt{d}/c)$ we
obtain that 
\[ 
    \E_{\x \sim \Dx}[\log( 1 + e^{- |\vec \theta^\ast \cdot \x|} ) ] \leq c/2\,.
\]
Therefore, combining the above bounds we obtain that $\Lc(\wh{\vec \theta}) - \Lc(\vec \theta^\ast)
\geq c - c/2 \geq c/2$.
\end{proof}

\subsection{The Proof of Theorem~\ref{thm:single-pass}}
In this section we restate and prove our result on the generalization error 
of single-pass stochastic gradient descent on the weighted objective.
\begin{theorem}[Generalization of Reweighted Single-Pass SGD (Restate of \ref{thm:single-pass})]
Assume that the example distributions $\Dcln, \D$ and the $\ell(\cdot)$ and weight function $\weight(\cdot)$
satisfy the assumptions of Definition~\ref{def:regularity}.  
Set $\kappa = L_\weight M_\ell + R M_\weight L_\ell$.  After $T = \Omega(\kappa^2/\eps^2 )$ SGD iterations (see
Algorithm~\ref{alg:one-pass-sgd}), with probability at least $99\%$, it holds
\[ 
    \Lc(\vec \Theta^{(T)}) \leq  \min_{\|\vec \Theta\|_F \leq 1} \Lc(\vec \Theta)  + \eps \,.
\]
\end{theorem}
\begin{proof}
We observe that, since $\weight$ is a debiasing weight function, given a sample $(\x\tth, \yadv\tth)
\sim \D$  it holds that $\nabla (\weight(\x\tth, \yadv\tth; \vec \Theta) 
\ell(\yadv,  \vec \Theta \x \tth) ) $ is an unbiased gradient estimate of $\nabla_{\vec \Theta} \Lc(\vec \Theta)$.  We will use
the following result on the convergence of the last-iterate of SGD for convex objectives.
For simplicity, we state the following theorem for the case where 
the parameter $\vec \theta$ is a vector in $\R^d$ 
(instead of a $L \times d$ matrix).
\begin{lemma}[Last Iterate Stochastic Gradient Descent \cite{JainPrateekNagaraj21}]
\label{lem:last-sgd}
Let $\mathcal{W}$ be a closed convex set of diameter $R$.  Moreover, let $F: \R^d \mapsto \R$ be a
convex, $L$-Lipschitz function.  Define the stochastic gradient descent iteration as
\begin{align*}
&\vec \theta^{(0)} \gets \vec 0
\\
&\vec \theta^{(t+1)} \gets 
            \proj_{\mathcal{W}} 
            \left(
            \vec \theta\tth - \eta\tth \vec g\tth(\vec \theta\tth)
            \right)
\end{align*}
where $\vec g\tth(\vec \theta\tth)$ is an unbiased gradient estimate 
of $\nabla_{\vec \theta}\ftrue(\vec \theta\tth)$. 
Assume that for all $t \in [T]$ it holds $\|\vec g\tth(\vec \theta)\|_2 \leq L$
for all $\vec \theta \in \mathcal W$.
There exists a step size sequence $\eta\tth$ that depends only
on $T, L, R$ such that, with probability at least $1-\delta$,
it holds 
\[ 
    F(\vec \theta^{(T)}) \leq F(\vec \theta^\ast) + O\left( R L \sqrt{ \frac{\log(1/\delta)}{T} }
    \right) \,.
\]
\end{lemma}
To simplify notation we let 
$\ell\tth(\Theta) \triangleq \ell(\yadv\tth, \Theta \x\tth)$
and $\weight\tth(\Theta) \triangleq \weight(\x\tth, \yadv\tth; \Theta)$.
For a sample $(\x\tth,\yadv\tth)$, the gradient $\vec g\tth$
of the weighted loss is:
\begin{align*}
&
\vec g\tth = 
\nabla_{\vec \Theta}(\weight\tth(\vec \Theta) \ell\tth( \vec \Theta  ) )
\\ &= 
\ell\tth(\vec \Theta) 
\nabla_{\vec \Theta}\weight\tth(\vec \Theta) 
+ \weight\tth(\vec \Theta)  
\nabla_{\Theta} \ell\tth (\vec \Theta) 
\\ &= 
\ell\tth(\vec \Theta) 
\nabla_{\vec \Theta}\weight\tth(\vec \Theta) 
+ \weight\tth(\vec \Theta)  
 \nabla_{\vec z} \ell(\x\tth \yadv\tth, \vec z) \big |_{\vec z = \vec \Theta \x\tth}  (\x \tth)^T
 \in \R^{L \times d}
\,.
\end{align*}
Using the triangle inequality for the Frobenious norm and the 
assumptions of Definition~\ref{def:regularity} on the functions
$\weight(\cdot)$ and $\ell(\cdot)$, we obtain that $ \|\vec g\tth \|_F \leq L_\weight M_\ell + R
M_\weight L_\ell$.  Using Lemma~\ref{lem:last-sgd} we obtain that 
with  $T = \Omega( (L_\weight M_\ell + R M_\weight L_\ell)^2/\eps^2)$, 
the last iteration of
Algorithm~\ref{alg:one-pass-sgd} satisfies the claimed guarantee.
\end{proof}

\subsection{The proof of Theorem~\ref{thm:multi-pass-gen}}
In this section we prove our result on multi-pass SGD.  For convenience, we first restate it.
\begin{theorem}[Generalization of Multi-Pass SGD (Restate of \ref{thm:multi-pass-gen})]
Set $\kappa =  R M_\ell L_\ell B_\ell M_\weight L_\weight B_\weight$ and define the empirical
reweighted objective with $N =  d^2/\eps^2 ~ \poly(\kappa)$ i.i.d.\ samples $(\x^{(1)}, \yadv^{(1)}),
\ldots, (\x^{(N)}, \yadv^{(N)})$ from the noisy distribution $\D$ as
\[ 
    \wh{\risk}^\weight(\vec \Theta) = 
    \frac{1}{N} \sum_{i=1}^N \weight(\x^{(i)}, \yadv^{(i)}; \vec \Theta) ~ 
    \ell(\vec \theta \cdot \x^{(i)} \yadv^{(i)})  \,.
\]
Then, after $T = \poly(\kappa)/\eps^4$ iterations, multi-pass SGD (see
Algorithm~\ref{alg:multi-sgd}) on $\wh{\risk}^\weight(\cdot)$ outputs a list $\vec
\theta^{(1)}, \ldots, \vec \theta^{(T)}$ that, with probability at least $99\%$, contains a vector
$\wh{\vec \theta}$ that satisfies
\[ 
    \Lc( \widehat{\vec \Theta}) \leq 
    \min_{\|\vec \Theta\|_F \leq 1} \Lc(\vec \Theta)  + \eps \,.
\]
\end{theorem}
\begin{proof}
To prove the theorem we shall first show that all stationary points of the empirical objective
(which for arbitrary weight functions $\weight(\cdot)$ may be non-convex) will have good
generalization guarantees.  Before we proceed we formally define approximate stationary points.  To simplify notation we shall assume that
the parameter is a vector $\vec \theta \in \R^d$. The definition extends directly 
to the case where $L$ is a function of a parameter  matrix $\Theta$ by using the 
corresponding matrix inner product.
\begin{definition}[$\eps$-approximate Stationary Points]
    Let $L:\R^d \mapsto \R$ be a differentiable function and $C$ be any convex subset of $\R^d$.  A
    vector $\vec \theta \in \R^d$ is an $\eps$-approximate stationary point of $L(\cdot)$ if for
    every $\vec \theta' \in C$ it holds that 
    \[
    \left|\nabla_{\vec \theta} L(\vec \theta) \cdot 
    \frac{\vec \theta' - \vec \theta}{\| \vec \theta' - \vec \theta \|_2} \right| 
    \leq  \eps \,.
    \]
\end{definition}

\begin{proposition}
\label{pro:stationary-points}
Set $\kappa =  R M_\ell L_\ell B_\ell M_\weight L_\weight B_\weight$ and define the empirical
reweighted objective with $N =  \wt{O}( (d L / \eps )^2) ~ \poly(\kappa) 
\log(1/\delta) $ i.i.d.\ samples $(\x^{(1)},
y^{(1)}), \ldots, (\x^{(N)}, y^{(N)})$ from the noisy distribution $\D$ as
\[ 
    \widehat{\risk}^\weight(\vec \theta) =  
    \frac{1}{N} \sum_{i=1}^N \weight(\x^{(i)}, y^{(i)}; \vec \Theta) ~ 
    \ell(y\ith,  \vec \Theta x^{(i)})  \,.
\]
Let $\widehat{\vec \Theta}$ be any $\eps$-stationary point of 
$\wh{\risk}^\weight(\vec \Theta)$ constrained on $\mathcal{B}_R$.
Then, with probability at least $1-\delta$, it holds that 
\[
\Lc( \widehat{\vec \Theta}) \leq 
\min_{\|\vec \Theta\|_F \leq 1} \Lc(\vec \Theta)  + \eps \,.
\]
\end{proposition}
\begin{proof}
We first show that, as long as the empirical gradients are close to the population gradients, any
stationary point of the weighted empirical objective will achieve good generalization error. 
In what follows we shall denote by $\vec \Theta^\ast$ the parameter that 
minimizes the clean objective:
\[
\vec \Theta^\ast \triangleq
\argmin_{\| \vec \Theta\|_F \leq 1} \risk(\vec \Theta) \,.
\]

Since the population objective is convex in $\vec \Theta$, 
we have that for any $\vec \Theta$ it holds that 
\begin{align*}
\Lc(\vec \Theta) - \Lc(\vec \Theta^\ast)
&\leq \nabla_{\vec \Theta} \Lc(\vec \Theta) \cdot (\vec \Theta - \vec \Theta^\ast)
\\
&=
(\nabla_{\vec \Theta} \Lc(\vec \Theta) - 
\nabla_{\vec \Theta} \wt{\calL}^\weight 
(\vec \Theta)) 
\cdot (\vec \Theta - \vec \Theta^\ast)
+ 
\nabla_{\vec \Theta} \wh{\risk}^\weight 
(\vec \theta)
\cdot (\vec \Theta - \vec \Theta^\ast)
\\
&\leq 2 
 \| 
\nabla_{\vec \Theta} \wh{\risk}^\weight(\vec \Theta) - 
\nabla_{\vec \Theta} \wh{\risk}^\weight(\vec \Theta)
\|_2 
+
\nabla_{\vec \Theta} \wh{\risk}^\weight 
(\vec \Theta)
\cdot (\vec \Theta - \vec \Theta^\ast)
\,.
\end{align*}
We have that the contstraint set $\|\vec \Theta \|_F \leq 1$ is convex and therefore for a
stationary point $\widehat{\vec \Theta}$ of $\calL^\weight(\vec \Theta)$ we have that 
$ |\nabla_{\vec \theta} \wh{\risk}^\weight (\vec \Theta) 
\cdot (\vec \Theta - \vec \Theta^\ast)| 
\leq  \eps \|\vec \Theta - \vec \Theta^\ast\|_F  \leq 2 \eps$.  
Therefore, $\wh{\vec \Theta}$ satisfies
\[
\Lc(\widehat{\vec \Theta}) 
-
\Lc(\vec \Theta^\ast)  \leq 
2 \| \nabla_{\vec \Theta} \Lc(\wh{\vec \Theta}) - 
\nabla_{\vec \Theta} \wh{\risk}^\weight(\wh{\vec \Theta}) \|_2 
+ 2 \eps \,.
\]
Since $\weight(\cdot)$ is a debiasing weighting function, 
we know that, as the number of samples $N \to \infty$, the empirical
gradients of the reweighted objective will converge 
to the gradients of the population clean objective $\Lc(\cdot)$, i.e.,
it holds 
that 
$ \nabla_{\vec \theta} \wh{\risk}^\weight(\vec \Theta) \to \nabla_{\vec \Theta} \Lc(\vec \Theta)  $.
Therefore, to finish the proof, we need to provide a uniform convergence bound for the gradient
field of the empirical objective.  We first consider estimating the gradient of some fixed parameter matrix $\vec \Theta$.  We will use McDiarmid's inequality.
\begin{lemma}[McDiarmid's Inequality]
\label{lem:mcdiarmid}
Let $\vec x_1, \ldots, \vec x_n$ be $n$ i.i.d.\ random variables taking values in $\mathcal{X}$.
Let $\phi:\mathcal{X}^n \to \R$ be such that $|\phi(\vec x) - \phi(\vec x')| \leq b_i$ whenever
$\vec x$ and $\vec x'$ differ only on the $i$-th coordinate.  It holds that 
\[
\pr\left[ |\phi(\vec x_1,\ldots, \vec x_n) - \E[\phi(\vec x_1,\ldots, \vec x_n) ] | \geq \eps \right]
\leq 2 \exp\left(-\frac{2 \eps^2}{\sum_{i=1}^n b_i^2 } \right)
\]
\end{lemma}
We consider the $n$ i.i.d.\ random variables $(\x\tth, \yadv\tth)$.  We have that the empirical gradient
of the weighted loss function is equal to 
\begin{align*}
\widehat{\vec g}(\vec \Theta) \triangleq
\frac{1}{N} \sum_{t=1}^N 
\left(
\ell( \yadv ; \vec \Theta \x \tth) 
~
\nabla_{\vec \Theta} \weight(\x\tth, \yadv\tth; \vec \Theta) + 
\weight(\x\tth, \yadv\tth; \vec \Theta) 
\nabla\ell(\yadv\tth, \vec \Theta \x\tth  ) 
~ (\vec x\tth)^T
\right)
\,.
\end{align*}
We have that $\weight(\x\tth, \yadv\tth; \vec \Theta)$ is $M_\weight$-bounded
and $L_\weight$-Lipschitz, $\ell$ is $M_\ell$-bounded and $L_\ell$-Lipschitz, and $\|\vec x\tth\|_2
\leq R$. 
Therefore, the maximum value of each coordinate of each term in the sum of the empirical gradient
$\wh{\vec g}(\vec \theta)$ is bounded by $L_q \triangleq  L_\weight M_\ell + R M_\weight L_\ell $.
Using this fact we obtain that each coordinate of the empirical gradient is a function of the $N$
i.i.d.\ random variables $(\x\tth, \yadv\tth)$ that satisfies the bounded differences assumption
with constants $b_1,\ldots, b_N$ that satisfy $b_t \leq L_q/N$.  From Lemma~\ref{lem:mcdiarmid}, we obtain that 
\begin{align*}
\pr\left[ \| \wh{\vec g}(\vec \Theta) - \nabla_{\vec \Theta} \Lc(\vec \Theta) \|_F  
\geq \eps 
\right]
&\leq 
\sum_{i=1}^d \sum_{j=1}^L
\pr\left[ | (\wh{\vec g}(\vec \Theta))_{ij} -  
(\nabla_{\vec \Theta} \Lc(\vec \Theta) )_{ij}
| \geq \eps/\sqrt{d L}
\right]
\\
&\leq 2 d L \exp\left(-\Omega\left( N \eps^2/ (d L ~ L_q^2) \right) \right) 
\footnote{To avoid confusion, we remark that here $L$ is the number of labels
and $L_q$ is the Lipschitz constant of the function $q(\Theta)$. } 
\,.
\end{align*}
We next need to provide a uniform convergence guarantee over the whole parameter space $\|\vec
\Theta \|_F \leq 1$.  We will use the following standard lemma bounding the cardinality of an
$\eps$-net of the unit ball in $d$-dimensions.  For a proof see, e.g., \cite{vershynin18}.
\begin{lemma} [Cover of the Unit Ball \cite{vershynin18}]
\label{lem:eps-net}
Let $\mathcal{B}$ be the $d$-dimensional unit ball around the origin.  There exists an $\eps$-net of $\mathcal{B}$
with cardinality at most $(1+2/\eps)^d$.
\end{lemma}

Since we plan to construct a net for the gradient of $\weight(x,\yadv ;\vec \Theta) 
\ell(\yadv; \Theta \x)$ we first need to show that the weighted loss 
$\weight(x,\yadv ;\vec \Theta) \ell(\yadv; \vec \Theta\x)$ is a smooth function of its parameter $\vec \Theta$ or,
in other words, that its gradients do not change very fast with respect to $\vec \Theta$.
The following lemma follows directly from the regularity assumptions of 
Defintion~\ref{def:regularity} and the chain and product rules for the derivatives.
\begin{lemma}\label{lem:smoothness}
    For all $(\x, y) \in \mathcal{B}_R \times \R^L$, it holds that the function $q(\vec \Theta)
    = w(\x,y; \vec \Theta) \ell(y ; \vec \Theta \x )$ is $B_q$-smooth for all $\Theta$
    with $\|\Theta\|_F \leq 1$, 
    with $B_q = M_\ell B_\weight + 2 L_\ell L_{\weight} R + M_{\weight} B_{\ell} R^2 $.
\end{lemma}
\begin{proof}
For simplicy we shall denote $\nabla_{\vec z} \ell(y; \vec z)$ simply
by $\nabla \ell(y; \vec z)$ and similarly $\nabla^2_{\vec z} \ell(y; \vec z)$
by $\nabla^2 \ell(y;\vec z)$.
    Using the chain rule, we have that the gradient of the weighted loss $q(\vec \Theta)$ is equal
    to
    \[ 
        \nabla_{\vec \Theta}q(\vec \Theta)
        = 
        \nabla_{\vec \Theta} \weight(\x, y; \vec \Theta)  
        ~
        \ell(y; \vec \Theta \vec x)  
        + 
        \weight(\x, y; \vec \Theta) 
        \nabla \ell(y ; \vec \Theta \x) 
        \x^T
        \,.
\]
Using again the chain and product rules we find the Hessian of $q(\Theta)$:
\begin{align*}
    &\nabla^2_{\vec \Theta} q(\Theta) 
    =
        \nabla^2_{\vec \Theta} \weight(\x, y; \vec \Theta)
        ~ \ell(y ; \vec \Theta \x)  
        + 
        (\nabla \ell(y; \Theta \x)  \vec x^T) \otimes \nabla_{\vec \Theta} \weight(\x, y; \vec \Theta) 
        \\
        &+ 
         \nabla_{\vec \Theta} \weight(\x, y; \vec \Theta) 
        \otimes (\nabla \ell(y; \vec \Theta \x) \vec x^T)
        +  \weight(x,y;\Theta) H
        \,,
\end{align*}
where $H$ is the $(L \times d) \times (L \times d)$ tensor
with element $H_{ijkl} =  \nabla^2 \ell(y; \vec \Theta \x)_{i k} \x_j \x_l $.
Recall that we view  $\nabla^2_{\vec \Theta} q(\Theta)$ as an $Ld \times Ld$ and to prove that it
is smooth we have to find its operator (spectral) norm.
Using the assumptions of Definition~\ref{def:regularity} 
we obtain that $\|\nabla^2_{\vec \Theta} \weight(\x, y; \vec \Theta)\|_2 \leq B_\weight M_\ell$.
For the term $(\nabla \ell(y; \Theta \x)  \vec x^T) \otimes \nabla_{\vec \Theta} \weight(\x, y; \vec \Theta) $
we consider any $\vec q \in \R^{Ld}$ with $\|\vec q\|_2 = 1$. We assume that $\vec q$ is indexed 
as $\vec q_{ij}$ for $i = 1,\ldots, L$ and $j = 1,\ldots, d$. We have
\begin{align*}
\vec q^T 
( (\nabla \ell(y; \Theta \x)  \vec x^T) \otimes \nabla_{\vec \Theta} \weight(\x, y; \vec \Theta) )
\vec q
&= \left(\sum_{ij} \vec q_{ij} (\nabla \ell(y; \Theta \x))_{i} x_{j} \right)
 \left(\sum_{kl} \vec q_{kl} (\nabla_{\vec \Theta} \weight(x, y; \Theta )_{kl} \right)
 \\
 &\leq R L_\ell L_\weight  \,.
 \end{align*}
 Similarly, we bound the spectral norm of the term 
$  \nabla_{\vec \Theta} \weight(\x, y; \vec \Theta)  \otimes (\nabla \ell(y; \vec \Theta \x) \vec x^T) $.
Finally for the term $H$ we have 
\begin{align*}
\vec q^T  H \vec q 
= \sum_{ijkl} \vec q_{ij} x_j (\nabla \ell(y; \Theta \x))_{ik} q_{kl} x_l
= \sum_{ik} \vec s_i (\nabla \ell(y; \Theta \x))_{ik} s_k,
\end{align*}
where $\vec s \in \R^L$ has $\vec s_i = \sum_{j} \vec q_{ij} x_j$.
Observe that since $\|\x\|_2 \leq R$ and $\|q\|_2 = 1$ we have that 
$\|\vec s\|_2 \leq R$.  Therefore, from the assumption of Definition~\ref{def:regularity},
we obtain that $\|H\|_2 \leq R^2 B_\ell$.

We conclude that the function $q(\vec \theta)$ is $B_q$-smooth 
on the unit ball $\mathcal{B}$ with 
$B_q = M_\ell B_\weight + 2 L_\ell L_{\weight} R  + M_{\weight} B_{\ell} R^2 $.
\end{proof}

Let $\mathcal{N}_\eps$ be an $\eps$-net of the unit ball $\mathcal{B}$.  Using
Lemma~\ref{lem:smoothness} we first observe that the vector maps $\theta \mapsto 
\widetilde{ \vec g}(\vec \Theta)$ and $\vec \Theta \mapsto \nabla_{\vec \Theta} \Lc(\vec \Theta)$ are both
$B_q$-Lipschitz, where $B_q$ is the constant defined in Lemma~\ref{lem:smoothness}.
Using the triangle inequality and the fact that $\wh{\vec g}(\cdot)$ and 
$\nabla_{\vec \Theta} \Lc(\cdot)$ are $B_q$-Lipschitz, we have that 
\[ \max_{\|\vec \Theta\|_F \leq 1} \| \wh{\vec g}(\vec \Theta) 
- \nabla_{\vec \Theta} \Lc(\vec \Theta) \|_2  \leq  
2 B_q \eps  + 
\max_{\vec \Theta \in \mathcal{N}_\eps} 
\| \wh{\vec g}(\vec \Theta) - \nabla_{\vec \Theta} 
\Lc(\vec \Theta) \|_2  \,.
\]
Combining the above, and performing a union bound over the $\eps$-net $\mathcal{N}_\eps$, we obtain that 
\[
\pr\left[ \max_{\|\vec \Theta\|_F \leq 1} \| \widetilde{\vec g}(\vec \Theta) - 
\nabla_{\vec \Theta} \Lc(\vec \Theta) \|_F \geq (2 B_q + 1) \eps
\right]
\leq (1 + 2/\eps)^{d L} \exp\left(-\Omega \left( N \eps^2/ (d L ~ L_q^2) \right) \right)
\,.
\]
We conclude that with $N = \wt{\Omega}((dL)^2 L_q^2 B_q^2/\eps^2 \log(1/\delta) )$ samples,
it holds that 
$
\| \wh{\vec g}(\vec \Theta) - \nabla_{\vec \Theta} \Lc(\vec \Theta) \|_2  
\leq \eps$, uniformly for all parameters $\vec \Theta $ with $\|\vec \Theta\|_F \leq 1$,
with probability at least $1-\delta$.
\end{proof}
We now have to show that the multi-pass SGD finds an approximate stationary point of the empirical
objective.  We will use the following result on non-convex projected SGD.
To simplify notation, we state the following optimization lemma assuming
that the parameter is a vector $\vec \theta \in \R^d$.
\begin{lemma}[Non-Convex Projected Stochastic Gradient Descent \cite{davis18}]
    \label{lem:non-convex-psgd}
Let $\mathcal{W}$ be a closed convex set of diameter $R$.
Moreover, let $F: \R^d \mapsto \R$ be an $L$-Lipschitz and $B$-smooth 
function.  Define the stochastic gradient descent iteration as
\begin{align*}
&\vec \theta^{(0)} \gets \vec 0
\\
&\vec \theta^{(t+1)} \gets 
            \proj_{\mathcal{W}} 
            \left(
            \vec \theta\tth - \eta\tth \vec g\tth(\vec \theta\tth)
            \right)
\end{align*}
where $\vec g\tth(\vec \theta\tth)$ is an unbiased gradient estimate of $\nabla_{\vec \theta}F(\vec
\theta\tth)$.  
Fix a number of iterations $T \geq 1$ and assume that for all $t \in [T]$ it holds $\|\vec
g\tth(\vec \theta)\|_2 \leq L$ for all $\vec \theta \in \mathcal W$.  Set the step-size $\eta\tth =
\Theta(\sqrt{R/ (B L^2 T)})$.  With probability at least $99\%$, there exists a $t \in
\{1,\ldots, T\}$ such that $\vec \theta \tth$ is an $O\left(\frac{\sqrt{B L R}}{T^{1/4}}
\right)$-stationary point of $F(\cdot)$ constrained on $\mathcal{W}$.
\end{lemma}
Theorem~\ref{thm:multi-pass-gen} now follows directly by applying Lemma~\ref{lem:non-convex-psgd} on the empirical objective 
to find an $\eps$-approximate stationary point and then using 
Proposition~\ref{pro:stationary-points}.

\end{proof}